\documentclass[11pt]{article}

\usepackage[utf8]{inputenc} % allow utf-8 input
\usepackage[T1]{fontenc}    % use 8-bit T1 fonts
\usepackage{hyperref}       % hyperlinks
\usepackage{url}            % simple URL typesetting
\usepackage{booktabs}       % professional-quality tables
\usepackage{amsfonts}       % blackboard math symbols
\usepackage{nicefrac}       % compact symbols for 1/2, etc.
\usepackage{tablefootnote}  % for table footnotes

\usepackage[x11names]{xcolor}
\usepackage{multirow}
\usepackage{adjustbox}
\usepackage{enumitem}
\usepackage{makecell}
\usepackage{colortbl}

\usepackage{fullpage}
\usepackage{wrapfig}

\usepackage[protrusion=true,expansion=true]{microtype}

\usepackage{algorithm}
\usepackage{algorithmic}
\usepackage{amsthm} 
\usepackage{amsmath}
\usepackage{amssymb}
\usepackage{float}
\usepackage{mathtools}
\usepackage{url}
\usepackage{graphicx,times}
\usepackage{fancyheadings}
\usepackage{subfigure}
\usepackage{bm}
\usepackage{authblk}

\usepackage{threeparttable}
\DeclareMathOperator*{\argmin}{arg\,min}

\newcommand{\norm}[1]{\left\lVert#1\right\rVert}

\usepackage{pifont}% 

\allowdisplaybreaks

\newtheorem{theorem}{Theorem}
\newtheorem{proposition}{Proposition} 
\newtheorem{lemma}{Lemma} 
\newtheorem{remark}{Remark} 
\newtheorem{assumption}{Assumption} 
\newtheorem{corollary}{Corollary} 
\newtheorem{definition}{Definition}

\usepackage{cleveref}

\title{Tuning-Free Bilevel Optimization: New Algorithms and Convergence Analysis}

\author[1]{Yifan Yang}
\author[1]{Hao Ban}
\author[2]{Minhui Huang}
\author[3]{Shiqian Ma}
\author[1]{Kaiyi Ji}
\affil[1]{Department of Computer Science and Engineering, University at Buffalo}
\affil[2]{AI Research for Monetization, Meta}
\affil[3]{Department of Computational Applied Math and Operations Research, Rice University}

\date{Oct 3, 2024}

\begin{document}

\maketitle
\footnotetext[1]{Email: \{yyang99, haoban, kaiyiji\}@buffalo.edu. }
\footnotetext[2]{Email: mhhuang@meta.com. }
\footnotetext[3]{Email: sqma@rice.edu. }

\begin{abstract}
\noindent
Bilevel optimization has recently attracted considerable attention due to its abundant applications in machine learning problems. However, existing methods rely on prior knowledge of problem parameters to determine stepsizes, resulting in significant effort in tuning stepsizes when these parameters are unknown. In this paper, we propose two novel tuning-free algorithms, D-TFBO and S-TFBO. D-TFBO employs a double-loop structure with stepsizes adaptively adjusted by the "inverse of cumulative gradient norms" strategy. S-TFBO features a simpler fully single-loop structure that updates three variables simultaneously with a theory-motivated joint design of adaptive stepsizes for all variables. We provide a comprehensive convergence analysis for both algorithms and show that D-TFBO and S-TFBO respectively require $\mathcal{O}(\frac{1}{\epsilon})$ and $\mathcal{O}(\frac{1}{\epsilon}\log^4(\frac{1}{\epsilon}))$ iterations to find an $\epsilon$-accurate stationary point, (nearly) matching their well-tuned counterparts using the information of problem parameters. Experiments on various problems show that our methods achieve performance comparable to existing well-tuned approaches, while being more robust to the selection of initial stepsizes. 
To the best of our knowledge, our methods are the first to completely eliminate the need for stepsize tuning, while achieving theoretical guarantees. 
\end{abstract}

\section{Introduction}
Bilevel optimization has gained considerable attention recently due to its widespread use in various machine learning applications, such as meta-learning~\cite{franceschi2018bilevel,bertinetto2018meta,rajeswaran2019meta}, hyperparameter optimization~\cite{shaban2019truncated,feurer2019hyperparameter}, reinforcement learning~\cite{konda2000actor,hong2020two}, robotics~\cite{wang2024imperative}, communication~\cite{ji2022network} and federated learning~\cite{tarzanagh2022fednest}. 
In this paper, we study a standard bilevel optimization problem that takes the following mathematical formulation:
% \vspace{-0.28cm}
\begin{align}\label{eq:bilevel}
    &\min_{x\in \mathbb{R}^{d_x}}\Phi(x) := f\big(x,y^*(x)\big) \nonumber \\
    &\;\;\mbox{s.t.} \;  y^*(x) = \arg\min_{y \in \mathbb{R}^{d_y}} g(x,y),
\end{align}
% \vspace{-0.5cm}
where $f$ and $g$ are jointly continuously differentiable outer (upper-level) and inner (lower-level)  functions. 
In this paper, we focus on the nonconvex-strongly-convex setting, where the lower-level function $g$ is strongly convex w.r.t. $y$ and the outer function $\Phi(x)$ is possibly nonconvex. 

Recent years have witnessed the rapid development of bilevel optimization algorithms, which can be categorized into approximate implicit differentiation (AID)~\cite{ji2021bilevel, dagreou2022framework} based, iterative differentiation (ITD)~\cite{ji2022will,grazzi2020iteration} based,  and value-function based~\cite{kwon2023fully,liu2021value} approaches. 
However, these methods often require substantial effort to tune a couple of hyperparameters like stepsizes, which typically depend on {\bf unknown} problem parameters (such as Lipschitzness parameters, strong convexity parameters, and optimal function values). This emphasizes the importance of {\em adaptive and tuning-free} methods in bilevel optimization.
{\em In this paper, an algorithm is considered tuning-free if it does not need to know the problem parameters in advance but can still achieve almost the same convergence rate guarantee as its well-tuned counterpart using this information.}  
Despite several recent efforts to reduce dependence on problem-specific parameters~\cite{fan2024bisls, antonakopoulosadaptive}, developing a fully tuning-free bilevel optimization algorithm remains an open challenge. For instance, \cite{fan2024bisls} utilizes Polyak's stepsizes to automate both inner and outer updates but still requires information such as gradient Lipschitzness parameters and optimal lower-level function values. Similarly, \cite{antonakopoulos2023adaptive} introduces an "on-the-fly" accumulation strategy for (hyper)gradient norms, which removes the reliance on inner and outer gradient Lipschitzness parameters but still depends on the strong convexity parameter for the inner AdaNGD-type updates. 

This paper aims to close this gap by introducing two novel fully tuning-free bilevel optimization algorithms named D-TFBO and S-TFBO (where D and S represent double- and single-loop approaches), along with a comprehensive convergence analysis demonstrating their competitive performance compared to existing well-tuned approaches (which tune their hyperparameters like stepsizes based on the problem parameters). Our key contributions are outlined below. 

\begin{list}{$\bullet$}{\topsep=0.3ex \leftmargin=0.25in \rightmargin=0.in \itemsep =-0.022in}
\item Our algorithms are inspired by the "inverse of cumulative gradient norms" strategy introduced by \cite{xie2020linear, ward2020adagrad}, adapting the stepsizes based on accumulated (hyper)gradient norms. D-TFBO utilizes two optimization sub-loops: one for solving the inner problem and another for addressing a linear system (LS), which approximates the Hessian-inverse-vector product of each hypergradient. 
Unlike previous approaches, D-TFBO introduces cold-start adaptive stepsizes that accumulate gradients exclusively within the sub-loops. This method establishes a tighter lower bound on stepsizes, improving gradient complexity. In contrast, S-TFBO adopts a single-loop structure, where all variables are updated simultaneously in each iteration. Rather than applying the "inverse of cumulative gradient norms" uniformly to all updates, our error analysis motivates a joint design of adaptive stepsizes for $y$, $v$, and $x$, which correspond to solving the inner problem, LS, and outer problem, respectively. For instance, the stepsize for $v$ is coupled with that for $y$, while the stepsize for $x$ depends on both $y$ and $v$.

\item  
Compared to the well-tuned AID methods in \cite{ji2022will}, our D-TFBO method achieves the same $\mathcal{O}(\frac{1}{T})$ convergence rate. Similarly, our S-TFBO method attains an $\widetilde{\mathcal{O}}(\frac{1}{T})$ convergence rate, matching that of well-tuned counterparts, up to polylogarithmic factors. 
The complexity analysis shows that D-TFBO and S-TFBO require $\mathcal{O}(\frac{1}{\epsilon^2})$ and $\widetilde{\mathcal{O}}(\frac{1}{\epsilon})$ gradient computations, respectively, to reach an $\epsilon$-accurate stationary point. This comparison differs from the observation in well-tuned bilevel optimization, where double-loop approaches generally achieve lower gradient complexity than single-loop methods \cite{ji2022will}. This is because the inner tuning-free solver requires $\mathcal{O}(\frac{1}{\epsilon})$ more iterations than well-tuned methods to achieve $\epsilon$-level accuracy. 

\item 
The theoretical analysis is inspired by the two-stage framework in~\cite{xie2020linear, ward2020adagrad}, where the stages describe the relationship between the stepsizes and certain constants that depend on the problem parameters. 
However, exploring this technical framework in bilevel problems is far more challenging because the stages for analyzing each stepsize interact with those for other stepsizes, resulting in intertwined multi-stage dynamics across different variables. For instance, the error analysis for the updates on $v$ must account for the accumulated gradient norms from the updates on $y$. This motivates us to couple the stepsize for $v$ with the adaptive stepsize for $y$ to prevent the propagation of accumulated errors. In addition, our analysis requires establishing precise upper and lower bounds for all stepsizes to ensure convergence results that match those achieved under well-tuned stepsizes. 

\item 
We validate the effectiveness of our methods through experiments on regularization selection, data hyper-cleaning, and coreset selection for continual learning. The results show that our methods perform comparably to existing well-tuned methods. More importantly, our methods demonstrate greater robustness to different initial stepsizes, due to the tuning-free design.

\end{list}

\section{Related Work}
\textbf{Bilevel Optimization.}
Bilevel optimization, initially introduced by \cite{bracken1973mathematical}, has been extensively studied for decades. Early works \cite{hansen1992new, shi2005extended} approached the bilevel problem from a constrained optimization perspective. More recently, gradient-based methods have gained significant attention for their efficiency and effectiveness. Among these, Approximate Implicit Differentiation (AID) methods \cite{domke2012generic, liao2018reviving, ji2021bilevel, dagreou2022framework} leverage the implicit derivation of the hypergradient by approximating it through the solution of a linear system. 
In contrast, Iterative Differentiation (ITD) methods \cite{maclaurin2015gradient, franceschi2017forward} estimate the hypergradient using automatic differentiation, employing either forward or reverse mode. Recently, a range of stochastic bilevel methods have been developed and analyzed, using techniques such as Neumann series \cite{chen2021single, ji2021bilevel}, recursive momentum \cite{yang2021provably, guo2021randomized}, and variance reduction \cite{yang2021provably}. 
Another class of methods formulates the lower-level problem as a value-function-based constraint \cite{kwon2023fully, wang2023effective}, enabling the solution of bilevel problems without the need for second-order gradients. A more detailed discussion of related work can be found in the Appendix.

\noindent
\textbf{Adaptive and Tuning-free Algorithms.}
Adaptive gradient descent has achieved remarkable success and is widely studied and applied in modern machine learning. Early adaptive algorithms trace back to line search methods, such as backtracking \cite{goldstein1962cauchy}, and Polyak’s stepsize \cite{polyak1969minimization}, both of which have inspired numerous recent variants \cite{armijo1966minimization, bello2016convergence, salzo2017variable, vaswani2019painless, hazan2019revisiting, loizou2021stochastic, orvieto2022dynamics}. To reduce the computational cost of line search and avoid the reliance on an unknown optimal function value, the Barzilai-Borwein stepsize \cite{barzilai1988two, raydan1993barzilai, dai2002r} was introduced, drawing inspiration from quasi-Newton schemes. 
Normalized gradient descent \cite{cortes2006finite, nesterov2013introductory, murray2019revisiting} preserves the direction of the gradient while ignoring its magnitude, removing the need for prior knowledge about the function. \cite{duchi2011adaptive} and \cite{mcmahan2010adaptive} pioneered AdaGrad, an adaptive gradient-based method, which proved efficient in solving online convex optimization problems. AdaGrad rapidly evolved for deep learning applications, giving rise to numerous methods, including popular variants like Adam \cite{diederik2014adam, reddi2019convergence, luo2019adaptive, xie2024adan}, RMSprop \cite{tieleman2012rmsprop}, and Adadelta \cite{zeiler2012adadelta}. 
Specifically, normalized versions of AdaGrad, such as AdaNGD$_k$ \cite{levy2017online}, AcceleGrad \cite{levy2018online}, and AdaGrad-Norm \cite{ward2020adagrad, xie2020linear}, introduced adaptive stepsizes that require no problem-specific parameters, making them tuning-free approaches. Recent work by \cite{maladkar2024convergence} further established lower bounds for minimizing the deterministic gradient $l_1$-norm. Additional methods, such as Lipschitzness parameter approximation \cite{malitsky2019adaptive} and restart techniques \cite{marumo2024parameter}, have also been explored. 
A more comprehensive discussion refers to \cite{khaled2024tuning}.

\noindent
\textbf{Adaptive and Tuning-free Bilevel Algorithms.}
Instead of focusing on single-level problems, \cite{huang2021biadam} extended Adam to bilevel optimization algorithms. \cite{fan2024bisls} introduced adaptive stepsizes for bilevel problems, based on Polyak’s stepsize and line search techniques. Most recently, \cite{antonakopoulosadaptive} proposed a novel framework that applies adaptive normalized gradient descent to the strongly convex inner problem and AdaGrad-Norm to the nonconvex outer problem, allowing the algorithm to update adaptively with fewer problem-specific parameters.

\section{Algorithm}
\subsection{Standard Bilevel Optimization}
A key challenge in bilevel optimization is calculating the hypergradient $\nabla \Phi(x)$, which, according to the implicit function theorem, is given by: 
\begin{align*}
    \nabla \Phi(x) = \nabla_x f\big(x,y^*(x)\big) - \nabla_x\nabla_y g\big(x,y^*(x)\big)\big[\nabla_y\nabla_y g\big(x,y^*(x)\big)\big]^{-1}\nabla_y f\big(x,y^*(x)\big),
\end{align*}
when $g$ is twice differentiable, $\nabla_y g$ is
continuously differentiable and the Hessian $\nabla_y\nabla_y g\big(x,y^*(x)\big)$ is invertible. 
In practice, $y^*(x)$ is not directly accessible, and one often use an iterative algorithm to obtain an estimate $\hat{y}$ instead. 
Since computing the Hessian inverse is prohibitively expensive, a more efficient way is to approximate the Hessian-inverse-vector product in the above hypergradient $\nabla \Phi(x)$ by solving the following linear system: 
\begin{align}\label{eq:defR}
    \min_v R(x,\hat{y},v) = \frac{1}{2}v^T \nabla_y\nabla_y g(x,\hat{y})v - v^T\nabla_y f(x,\hat{y}). 
\end{align}
Similarly, an iterative algorithm is usually deployed to obtain an approximate solution $\hat v$ of the problem in \cref{eq:defR}.
Given the approximates $\hat y$ and $\hat v$, the variable $x$ is then updated with a hypergradient estimate given by  
\begin{align}\label{eq:estimator}
    \bar{\nabla} f(x,\hat{y},\hat{v}) = \nabla_x f(x,\hat{y}) - \nabla_x\nabla_y g(x,\hat{y})\hat{v}. 
\end{align}
Standard bilevel optimization approaches select the stepsizes for updating $y$, $v$, and $x$ based on problem-specific parameters, such as Lipschitzness and strong convexity parameters~\cite{dagreou2022framework, ji2021bilevel, ji2022will}. However, these parameters are often difficult to obtain or approximate in practice, leading to significant tuning efforts. This challenge motivates the development of adaptive bilevel optimization algorithms that require less to no tuning.

\subsection{Existing Adaptive Bilevel Optimization Methods}
Among the existing adaptive bilevel methods, the most closely related to this work are \cite{fan2024bisls} and \cite{antonakopoulosadaptive}. \cite{fan2024bisls} utilizes Polyak's stepsizes and a line search to automate the stepsizes for both inner and outer updates. \cite{antonakopoulosadaptive} applies AdaNGD~\cite{levy2017online} to solve the inner problem and updates $x$ using the inverse of cumulative hypergradient norms, where the hypergradient norms are approximated via gradient mapping~\cite{nesterov2013introductory} with Fenchel coupling \cite{mertikopoulos2016learning}.

However, these methods are not entirely tuning-free. For instance, the initialization of Polyak's stepsizes in \cite{fan2024bisls} depends on Lipschitzness parameters, strong convexity parameters, and the optimal lower-level function values. While the line search approach in \cite{fan2024bisls} bypasses the need for problem-specific parameters, it lacks theoretical convergence guarantees. Similarly, \cite{antonakopoulos2023adaptive} requires the strong convexity parameter for the inner AdaNGD updates. 

\begin{algorithm}[t]
	\caption{\textbf{D}ouble-loop \textbf{T}uning-\textbf{F}ree \textbf{B}ilevel \textbf{O}ptimizer (D-TFBO)}   
	\small
	\label{alg:main_double}
	\begin{algorithmic}[1]
            \STATE {\bfseries Input:} initialization $x_0$, $y_0$, $v_0$, $\alpha_0>0$, $\beta_0>0$, $\gamma_0>0$, total iteration rounds $T$, and {$\epsilon_y = \epsilon_v = \frac{1}{T}$}
		    \FOR{$t=0,1,2,...,T-1$}
            \STATE{$p = 0$, $q = 0$, set $y_{t}^0 = y_{t-1}^{P_{t-1}}$, $v_{t}^0 = v_{t-1}^{Q_{t-1}}$ if $t > 0$ and $y_0$, $v_0$ otherwise}
            \WHILE{$\|\nabla_y g(x_t, y_t^p)\|^2 > {\epsilon_y}$}
                \STATE{$\beta_{p+1}^2 = \beta_p^2 + \|\nabla_y g(x_t, y_t^p)\|^2$, \ \  $y_t^{p+1} = y_t^p - \frac{1}{\beta_{p+1}}\nabla_y g(x_t, y_t^p)$, \ \ $p=p+1$}
            \ENDWHILE
            \STATE{$P_t = p$}
            \WHILE{$\|\nabla_v R(x_t, y_t^{P_t}, v_t^q)\|^2 > {\epsilon_v}$}
                \STATE{$\gamma_{q+1}^2 = \gamma_q^2 + \|\nabla_v R(x_t, y_t^{P_t}, v_t^q)\|^2$, \ \  $v_t^{q+1} = v_t^q - \frac{1}{\gamma_{q+1}}\nabla_v R(x_t, y_t^{P_t}, v_t^q)$, \ \ $q=q+1$}
            \ENDWHILE
            \STATE{$Q_t = q$}
            \STATE{$\alpha_{t+1}^2 = \alpha_t^2 + \|\Bar{\nabla} f(x_t, y_t^{P_t}, v_t^{Q_t})\|^2$, \ \ $x_{t+1} = x_t - \frac{1}{\alpha_{t+1}} \Bar{\nabla} f(x_t, y_t^{P_t}, v_t^{Q_t})$}
            \ENDFOR
	\end{algorithmic}
\end{algorithm}

\subsection{Double-Loop Tuning-Free Bilevel Optimization- D-TFBO}
As shown in \Cref{alg:main_double}, our D-TFBO method follows a double-loop structure, where two sub-loops of iterations are used to solve the lower-level and linear system problems.  In the first sub-loop, 
we employ the idea of "inverse of cumulative gradient norm" to design the adaptive updates as 
\begin{align}
    y_t^{p+1} \leftarrow y_t^p - \frac{1}{\beta_{p+1}}\nabla_y g(x_t, y_t^p), \quad \text{with} \ \  \beta_{p+1}^2 = \beta_p^2 + \|\nabla_y g(x_t, y_t^p)\|^2. \nonumber
\end{align}
It can be seen from \Cref{alg:main_double} that our D-TFBO algorithm employs a stopping criterion based on the gradient norm: $\|\nabla_y g(x_t, y_t^p)\|^2 \leq \epsilon_y$, where $\epsilon_y$ ({\bf defaulted to $1/T$ for convergence analysis}) is independent of problem parameters. The rationale behind this design is that if the stopping criterion is not met (i.e., $\|\nabla_y g(x_t, y_t^p)\|^2 > \epsilon_y$), the accumulation $\beta_p$ of gradient norms continues to increase. This increase causes the stepsize $\frac{1}{\beta_p}$ to decrease to a value at which a descent in the optimality gap is guaranteed. A similar stopping criterion applies to the updates of $v_t^q$ when solving the linear system.

Notably, both sub-loops utilize warm-start variable values but reset the stepsizes at each iteration (cold-start stepsizes). The warm-start variables ensure that the initial point is reasonably close to the optimal solution, while the cold-start scheme guarantees stepsizes to achieve stronger lower bounds. Finally, the update of $x_t$ is based on the accumulation of hypergradient estimates $\bar{\nabla} f(x_t, y_t^{P_t}, v_t^{Q_t})$.

\begin{remark}[Extension to a tunable version with problem-parameter-free tuning coefficients.]\label{rmk:tune_double}
    Although \Cref{alg:main_double} is designed as a tuning-free method, a tunable version with the flexibility to preset hyperparameters can still achieve the same convergence rate and gradient complexity. 
    The stepsizes for $\{x, y, v\}$ can be set as $\{\eta_x/\alpha_t, \eta_y/\beta_p, \eta_v/\gamma_q\}$ and the sub-loops stopping criteria can be set to $\{c_y/T, c_v/T\}$, where $\{\eta_x, \eta_y, \eta_v, c_y, c_v\}$ are configurable hyperparameters that are independent of the problem parameters such as strong-convexity and Lipschitzness parameters. 
\end{remark}

\begin{algorithm}[t]
	\caption{\textbf{S}ingle-loop \textbf{T}uning-\textbf{F}ree \textbf{B}ilevel \textbf{O}ptimizer (S-TFBO)}   
	\small
	\label{alg:main}
	\begin{algorithmic}[1]
            \STATE {\bfseries Input:} initialization $x_0$, $y_0$, $v_0$, $\alpha_0 \geq 1$, $\beta_0>0$, $\gamma_0>0$, number of iteration rounds $T$
		\FOR{$t=0,1,2,...,T-1$}
                \STATE{$\beta_{t+1}^2 = \beta_t^2 + \|\nabla_y g(x_t, y_t)\|^2$}
                \STATE{$\gamma_{t+1}^2 = \gamma_t^2 + \|\nabla_v R(x_t, y_t, v_t)\|^2$}
                \STATE{$\varphi_{t+1} = \max\{\beta_{t+1}, \gamma_{t+1}\}$}
                \STATE{$\alpha_{t+1}^2 = \alpha_t^2 + \|\Bar{\nabla} f(x_t, y_t, v_t)\|^2$}
    		  \STATE{$y_{t+1} = y_{t} - \frac{1}{\beta_{t+1}} \nabla_y g(x_t, y_t)$}
                \STATE{$v_{t+1} = v_{t} - \frac{1}{\varphi_{t+1}} \nabla_v R(x_t, y_t, v_t)$}
                \STATE{$x_{t+1} = x_t - \frac{1}{\alpha_{t+1}\varphi_{t+1}} \Bar{\nabla} f(x_t, y_t, v_t)$}
            \ENDFOR
	\end{algorithmic}
\end{algorithm}

\subsection{Single-Loop Tuning-Free Bilevel Optimization- S-TFBO}
The two sub-loops in D-TFBO may complicate the implementation, and increase the number of iterations to meet the stopping criterion.  
In this section, we propose a much simpler fully single-loop tuning-free bilevel optimization method named S-TFBO, as described in \Cref{alg:main}.

The design of stepsizes in \Cref{alg:main} follows a similar idea in \Cref{alg:main_double}. 
In each iteration $t$, 
we update $\alpha_t$, $\beta_t$, $\gamma_t$ as accumulations of gradient norms of $\bar{\nabla}f$, $\nabla_y g$, and $\nabla_v R$ from the previous $t-1$ iterations.
We then update variables $y_t$, $v_t$ and $x_t$ simultaneously with adaptive stepsizes $\big\{\frac{1}{\beta_t}, \frac{1}{\max\{\beta_t, \gamma_t\}}, \frac{1}{\alpha_t{\max\{\beta_t, \gamma_t\}}}\big\}$. 
However, the stepsizes for $v$ and $x$ are not straightforward and require careful designs guided by our theoretical analysis, as elaborated below. 

\noindent
{\bf Design of stepsize for $\boldsymbol{v_t}$. }
Instead of simply using $\frac{1}{\gamma_t}$, we introduce $\frac{1}{\varphi_t}:=\frac{1}{\max\{\beta_t, \gamma_t\}}$ as the stepsize.
This adjustment is necessary because $\nabla_v R(x_t, y_t, v_t)$ involves the approximation error $\|y_t-y^*(x_t)\|^2$.  
Since this error is proportional to $\|\nabla_y g(x_t,y_t)\|^2$, using $\frac{1}{\beta_t}$ helps control this error and prevents it from exploding after accumulation, as validated in our theoretical analysis later.

\noindent
{\bf Design of stepsize for $\boldsymbol{x_t}$. }
Similarly, we use $\frac{1}{\alpha_t\varphi_t}$ as the stepsize for updating $x_t$, where the coupled factor  
$\frac{1}{\varphi_t}$ is introduced to mitigate the approximation errors from the $y_t$ and $v_t$ updates, leading to a more stable convergence. 

\begin{remark}[Extension to a tunable version with problem-parameter-free tuning coefficients.]\label{rmk:tune}
Similarly to \Cref{rmk:tune_double}, \Cref{alg:main} can extend to a tunable version with the same convergence rate and gradient complexity. 
The stepsizes for $\{x, y, v\}$ can be set as $\{\eta_x/\alpha_t\varphi_t, \eta_y/\beta_t, \eta_v/\varphi_t\}$, where $\{\eta_x, \eta_y, \eta_v\}$ are configurable hyperparameters that are independent of the problem parameters. 
\end{remark}

\section{Theoretical Analysis}
\subsection{Technical Challenges}
Compared to existing single-level tuning-free approaches, fully tuning-free bilevel optimization poses unique challenges that have not been addressed well.
\begin{itemize}[leftmargin=8.5mm]
    \item 
    Compared to single-level problems, bilevel problems involve interdependent variable updates, resulting in more complex and interconnected stepsize designs. 
    \item 
    The stages for analyzing each stepsize interact with those of other stepsizes, leading to intertwined multi-stage dynamics across various variables.
    \item The optimization error of each variable can accumulate (hyper)gradient norms from previous iterations due to the adaptive stepsize designs, complicating the error analysis.
\end{itemize}
In \Cref{sec:assumption}, we introduce the standard definitions and assumptions. Next, in \Cref{sec:alg_double} and \ref{sec:alg_single}, we provide a detailed convergence analysis, explaining how we address the above challenges. 

\subsection{Assumptions and Definitions}\label{sec:assumption}
We make the following definitions and assumptions for outer- and inner-objective functions, as also adopted by \cite{ghadimi2018approximation,chen2021single,khanduri2021near}. 
\begin{definition}
    A mapping $f$ is $L$-Lipschitz continuous if $\|f(x_1) - f(x_2)\| \leq L\|x_1-x_2\|$ for $\forall x_1, x_2$. 
\end{definition}
Since the outer objective function $\Phi(x)$ is non-convex, we aim to find an $\epsilon$-accurate stationary point, as defined below.
\begin{definition}
    An output $\bar{x}$ of an algorithm is the $\epsilon$-accurate stationary point of the objective function $\Phi(x)$ if $\|\nabla \Phi(\bar{x})\|^2 \leq \epsilon$, where $\epsilon \in (0,1)$. 
\end{definition}

\begin{assumption}\label{as:sc}
    Functions $f(x,y)$ and $g(x,y)$ are twice continuously differentiable and $g(x,y)$ is $\mu$ strongly convex w.r.t. $y$, for $x\in \mathbb{R}^{d_x}$, $y\in \mathbb{R}^{d_y}$. 
\end{assumption}
The following assumption imposes the Lipschitz continuity on the outer and inner functions and their derivatives.
\begin{assumption}\label{as:lip}
    Function $f(x,y)$ is $L_{f,0}$-Lipschitz continuous; the gradients $\nabla f(x,y)$ and $\nabla g(x,y)$ are $L_{f,1}$ and $L_{g,1}$-Lipschitz continuous, respectively; the second-order gradients $\nabla_x\nabla_y g(x,y)$ and $\nabla_y\nabla_y g(x,y)$ are $L_{g,2}$-Lipschitz continuous.
\end{assumption}
Rather than directly using the Lipschitz continuity parameters as bounds on gradients-which can cause dimensional inconsistencies during logarithmic operations-we offer the following remark:
\begin{remark}\label{as:grad}
    Assumption \ref{as:lip} indicates that there exist constants $C_{f_x}$, $C_{f_y}$, $C_{g_{xy}}$ and $C_{g_{yy}}$ such that 
    {\small $\|\nabla_x f(x,y)\| \leq C_{f_x}$}, 
    {\small$\|\nabla_y f(x,y)\| \leq C_{f_y}$}, {\small$\|\nabla_x\nabla_y g(x,y)\| \leq C_{g_{xy}}$} and {\small$\|\nabla_y\nabla_y g(x,y)\| \leq C_{g_{yy}}$}. 
\end{remark}
\begin{assumption}\label{as:inf}
    There exists $m\in \mathbb{R}$ such that $\inf_x \Phi(x) \geq m$. 
\end{assumption}
Next, we present the main convergence theorems for \Cref{alg:main_double} and \Cref{alg:main}, along with key propositions that provide insights into these theorems. {\bf A proof sketch is provided in \Cref{proof_sktech}.}

\subsection{Convergence and Complexity Analysis for \Cref{alg:main_double}}\label{sec:alg_double}
Firstly, we explain the two-stage framework used in our analysis. 
\begin{proposition}\label{prop:bar} 
Suppose the iteration rounds to update $\{x,y,v\}$ are $\{T_1,T_2,T_3\}$ and $\{\alpha_t, \beta_t, \gamma_t\}$ are generated by \Cref{alg:main_double} or \ref{alg:main}. For any $C_\alpha \geq \alpha_0$, $C_\beta \geq \beta_0$, $C_\gamma \geq \gamma_0$, we have 
\begin{enumerate}[label=(\alph*)]
\item either $\alpha_t \leq C_\alpha$ for any $t \leq T_1$, or $\exists k_1 \leq T_1$ such that $\alpha_{k_1} \leq C_\alpha$, $\alpha_{k_1+1} > C_\alpha$; 
\item either $\beta_t \leq C_\beta$ for any $t \leq T_2$, or $\exists k_2 \leq T_2$ such that $\beta_{k_2} \leq C_\beta$, $\beta_{k_2+1} > C_\beta$; 
\item either $\gamma_t \leq C_\gamma$ for any $t \leq T_3$, or $\exists k_3 \leq T_3$ such that $\gamma_{k_3} \leq C_\gamma$, $\gamma_{k_3+1} > C_\gamma$.
\end{enumerate}
\end{proposition}
The analysis for each stepsize is divided into two cases. Let us take (a) as an illustration example. 
Case 1: the accumulation $\alpha_t$ of gradient norms is bounded by a constant $C_\alpha$ before the end of the iteration. In this case, the average gradient norm square can be bound as $\frac{C_\alpha^2}{T_1}$, which decreases with $T_1$. 
Case 2: the accumulation $\alpha_{T_1}$ exceeds $C_\alpha$, and hence $\alpha_t$ experiences two stages: in stage 1, $\alpha_t \leq C_\alpha$, and in stage 2, $\alpha_t > C_\alpha$. The error analysis for stage 1 is similar to that of case 1. In stage 2, the stepsizes are small enough to show the gradient norm decreases via a descent lemma.

\begin{proposition}\label{prop:subloop}
    Recall that for $t_{th}$ iteration, the sub-loops in \Cref{alg:main_double} aim to find $y_t^{P_t}$ and $v_t^{Q_t}$ such that $\|\nabla_y g(x_t, y_t^{P_t})\|^2 \leq \epsilon_y$ and $\|\nabla_v R(x_t, y_t^{P_t}, v_t^{Q_t})\|^2 \leq \epsilon_v$. 
Under Assumptions \ref{as:sc}, \ref{as:lip}, we have  
$$
\left\{
\begin{aligned}
    &P_t \leq \frac{\log(C_\beta^2/\beta_0^2)}{\log(1+\epsilon_y/C_\beta^2)} + \frac{\beta_{\text{max}}}{\mu}\log \big(\frac{L_{g,1}^2(\beta_{\text{max}}-C_\beta)}{\epsilon_y}\big),  \\
    &Q_t \leq \frac{\log(C_\gamma^2/\gamma_0^2)}{\log(1+\epsilon_v/C_\gamma^2)} + \frac{\gamma_{\text{max}}}{\mu}\log\big(\frac{C_{g_{yy}}^2(\gamma_{\text{max}}-C_\gamma)}{\epsilon_v}\big),
\end{aligned}
\right.
$$
where $\{C_\beta, C_\gamma\}$, $\beta_{\text{max}}$, $\gamma_{\text{max}}$ are denied in \cref{def:C_double}, \cref{eq:y_k2_doube7}, \cref{eq:v_k3_doube7} in the Appendix, respectively. 
\end{proposition}
\Cref{prop:subloop} provides upper bounds on $P_t$ and $Q_t$, which correspond to the total numbers of iterations of the two sub-loops. This result is the same as that of the standard AdaGrad-Norm in the strongly convex setting \cite{xie2020linear}. 
For the sub-loop for $y$, 
in Case 1 above, the loop terminates within ${\log(C_\beta^2/\beta_0^2)}/{\log(1+\epsilon_y/C_\beta^2)}$ steps; and in Case 2, it takes at most 
${\log(C_\beta^2/\beta_0^2)}/{\log(1+\epsilon_y/C_\beta^2)}$ steps for stage 1 and it takes at most $\frac{\beta_{\text{max}}}{\mu}\log ({L_{g,1}^2(\beta_{\text{max}}-C_\beta)}/{\epsilon_y})$ steps for stage 2. 
For $\epsilon_y$  small enough, it can be seen that $P_t$ takes an order of $1/\epsilon_y$, which is typically larger than those obtained with well-tuned stepsizes.  
Based on this proposition, we can derive the following convergence  results.
\begin{theorem}\label{thm:main_double}
Suppose Assumptions \ref{as:sc},\ref{as:lip},\ref{as:inf} are satisfied. 
By setting $\epsilon_y = 1/T$ and $\epsilon_v = 1/T$, 
the iterates generated by \Cref{alg:main_double} satisfy
\begin{align}
    \frac{1}{T}\sum_{t=0}^{T-1}\|\nabla \Phi(x_t)\|^2 \leq \frac{c_1(C_\alpha + 2c_1)}{T} = \mathcal{O}\Big(\frac{1}{T}\Big), \nonumber
\end{align}
where $C_\alpha$ and $c_1$ are constants defined in \cref{def:C_double} and \cref{eq:c1_double}, respectively. 
\end{theorem}
\begin{corollary}\label{cor:main_double}
    Under the same setting \Cref{thm:main_double}, to achieve an $\epsilon$-accurate stationary point, \Cref{alg:main_double} needs $T=\mathcal{O}(1/\epsilon)$, $\{P_t, Q_t\}=\mathcal{O}(1/\epsilon)$,  
    and the gradient complexity (i.e., the number of gradient evaluations) is ${\rm Gc}(\epsilon) = \mathcal{O}(1/\epsilon^2)$. 
\end{corollary}
\Cref{thm:main_double} shows that the convergence rate of \Cref{alg:main_double} matches that of the standard double-loop bilevel  algorithms \cite{ji2021bilevel, ji2022will}. 
According to \Cref{prop:subloop}, the sub-loops for updating $y$ and $v$ require $\mathcal{O}(1/\epsilon_y)$ iterations to ensure an $\epsilon_y$-level approximation accuracy, which is worse than the $\mathcal{O}(1)$ results achieved by well-tuned bilevel optimization methods. This is because  more iterations are needed to ensure high accuracy in both sub-loops, due to the lack of information about the Lipschitzness parameters and strong convexity parameters. Consequently, the gradient complexity of our D-TFBO method is worse than those of well-tuned double-loop methods by an order of $1/\epsilon$.

\subsection{Convergence and Complexity Analysis for \Cref{alg:main}}\label{sec:alg_single}
Differently from D-TFBO that uses sub-loops to achieve high-accurate $y$ and $v$ iterates, the main challenge for analyzing S-TFBO lies in dealing with the accumulated approximations errors for updating all variables over iterations. In the following propositions, we will show how we upper-bound such cumulative approximation errors and lower-bound the adaptive stepsizes. 

First, we present a descent result for the objective function $\Phi(\cdot)$.
\begin{proposition}\label{prop:objective}
Under Assumptions \ref{as:sc}, \ref{as:lip}, for Algorithm \ref{alg:main}, suppose the total iteration number is $T$. 
No matter $k_1$ in \Cref{prop:bar} exists or not, we always have
\begingroup
\allowdisplaybreaks
{
\small
\begin{align}
    \Phi(x_{t+1}) - \Phi(x_t) \leq&  - \frac{1}{2\alpha_{t+1}\varphi_{t+1}}\|\nabla \Phi(x_t)\|^2 - \frac{1}{2\alpha_{t+1}\varphi_{t+1}}\Big(1-\frac{L_{\Phi}}{\alpha_{t+1}\varphi_{t+1}}\Big)\|\bar{\nabla}f(x_t, y_t, v_t)\|^2 \nonumber \\
    & + \frac{\bar{L}^2}{2\mu^2}\bigg[1 + \frac{2}{\mu^2}\Big(\frac{L_{g,2}C_{f_y}}{\mu}+L_{f,1}\Big)^2\bigg]\frac{\big\|\nabla_y g(x_t,y_t)\big\|^2}{\alpha_{t+1}\varphi_{t+1}} + \frac{\bar{L}^2}{\mu^2}\frac{\big\|\nabla_v R(x_t,y_t,v_t)\|^2}{\alpha_{t+1}\varphi_{t+1}}. \nonumber
\end{align}}
If in addition, $k_1$ in \Cref{prop:bar} exists, then for $t\geq k_1$, we further have 
{
\small
\begin{align}
    \Phi(x_{t+1}) - \Phi(x_t)
    \leq& - \frac{1}{2\alpha_{t+1}\varphi_{t+1}}\|\nabla \Phi(x_t)\|^2 - \frac{1}{4\alpha_{t+1}\varphi_{t+1}}\|\bar{\nabla}f(x_t, y_t, v_t)\|^2 \nonumber \\
    & + \frac{\bar{L}^2}{2\mu^2}\bigg[1 + \frac{2}{\mu^2}\Big(\frac{L_{g,2}C_{f_y}}{\mu}+L_{f,1}\Big)^2\bigg]\frac{\big\|\nabla_y g(x_t,y_t)\big\|^2}{\alpha_{t+1}\varphi_{t+1}} + \frac{\bar{L}^2}{\mu^2}\frac{\big\|\nabla_v R(x_t,y_t,v_t)\|^2}{\alpha_{t+1}\varphi_{t+1}}, \nonumber
\end{align}}
\endgroup
where $\bar{L}:= \max\big\{2({C_{f_y}^2L_{g,2}^2}/{\mu^2} + L_{f,1}^2)^{\frac{1}{2}}, \sqrt{2}C_{g_{yy}}\big\}$.  
\end{proposition}
It can be seen from \Cref{prop:bar} that 
we derive two distinct forms of descent results for the objective function based on the relationship between $\alpha_{t+1}$ and $C_\alpha$ (whose form is specified in \cref{def:C} in the appendix). Their key difference is that the second inequality is tighter for the case $t\geq k_1$ by eliminating a term of $\frac{L_\Phi}{2\alpha_{t+1}^2\varphi_{t+1}^2}\|\bar{\nabla}f(x_t, y_t, v_t)\|^2$. 
Both upper bounds consist of two parts: (i) the approximation errors {\small $\mathcal{O}(\|\nabla_y g(x_t,y_t)\|^2 + \|\nabla_v R(x_t,y_t,v_t)\|^2)/(\alpha_{t+1}\varphi_{t+1})$} induced by the updates on $y$ and $v$; (ii) the descent term {\small$-\|\nabla \Phi(x_t)\|^2/(\alpha_{t+1}\varphi_{t+1})$}. It can be seen that there exists a trade-off: 
a smaller $\alpha_t\varphi_t$ leads to a more notable descent, but larger approximation errors. However, due to the lack of information about the problem parameters, the value of $\alpha_t\varphi_t$ remains unknown, making it infeasible to determine the optimal trade-off. Instead, we adjust this trade-off based on an overall bound on the descent and approximation errors, derived by telescoping all descent inequalities.

Next, we investigate the upper bounds on the summations of the positive error terms in \Cref{prop:objective}.

\begin{proposition}\label{prop:somebounds}
    Under Assumptions \ref{as:sc}, \ref{as:lip}, for any $0 \leq k_0 < t$, for the positive error  terms in \Cref{prop:objective}, we have the upper bounds in terms of logarithmic functions as
\begingroup
\allowdisplaybreaks
\begin{align}
    \sum_{k=k_0}^t \frac{\|\nabla_y g(x_k, y_k)\|^2}{\beta_{k+1}} \leq a_2 \log(t+1) + b_2, \ \ 
    \sum_{k=k_0}^t \frac{\|\nabla_v R(x_k, y_k, v_k)\|^2}{\varphi_{k+1}} \leq a_3 \log(t+1) + b_3, \nonumber
\end{align}
\endgroup
where $a_2$, $b_2$, $a_3$, $b_3$ are defined in \cref{def:c1_a2b2_a3b3} in the Appendix. 
\end{proposition}

\begin{proposition}\label{prop:alpha}
Under Assumptions \ref{as:sc}, \ref{as:lip}, \ref{as:inf}, suppose the total iteration rounds is $T$. For any case in \Cref{prop:bar}, we have the upper-bound of $\varphi_t$ and $\alpha_t$ in \Cref{alg:main} as
\begin{align}
    \varphi_t \leq a_1\log(t) + b_1, \ \ 
    \alpha_{t} \leq & C_\alpha + \big(a_4 \log(t) + b_4 + {4}(\Phi(x_0) - \inf_x \Phi(x))\big)\varphi_t, \nonumber
\end{align}
where $a_1$, $b_1$ are defined in \cref{def:a1b1} and $a_4$, $b_4$ are defined in \cref{def:a4b4} in the Appendix. 
\end{proposition}
\Cref{prop:somebounds} provides the upper bounds on the accumulated positive error terms in \Cref{prop:objective}, and \Cref{prop:alpha} shows that the cumulative gradient norms for all variables increase only logarithmically. 
By rearranging the terms 
and taking the average, we have the upper bound for the average squared hypergradient norm $\frac{1}{T}\sum_{t=0}^{T-1}\|\nabla \Phi(x_t)\|^2$, establishing the final convergence rate of \Cref{alg:main}, as shown in the following theorem and corollary.

\begin{theorem}\label{thm:main}
Suppose Assumptions \ref{as:sc},\ref{as:lip},\ref{as:inf} are satisfied. The iterates generated by \Cref{alg:main} satisfy
\begin{align}
    \frac{1}{T}\sum_{t=0}^T& \|\nabla \Phi(x_t)\|^2 
    \leq \frac{1}{2T}\Big[\Big(a_4 \log(T) + b_4 + 4\big(\Phi(x_0) - \inf_x \Phi(x)\big)\Big)^2\big(a_1\log(T)+b_1\big)^2 \nonumber \\
    & \quad \quad + C_\alpha\Big(a_4 \log(T) + b_4 + 4\big(\Phi(x_0) - \inf_x \Phi(x)\big)\Big)\big(a_1\log(T)+b_1\big)\Big] = \mathcal{O}\Big(\frac{\log^4(T)}{T}\Big), \nonumber
\end{align}
where $\{C_\alpha$, $a_1$, $b_1$, $a_4$, $b_4\} = \mathcal{O}(1)$ are defined in \cref{def:C}, \cref{def:a1b1}, \cref{def:a4b4} in the Appendix. 
\end{theorem}
\begin{corollary}\label{cor:main}
    Under the same setting \Cref{thm:main}, to achieve an $\epsilon$-accurate stationary point, \Cref{alg:main} needs $T = \mathcal{O}\big(\frac{1}{\epsilon}\log^4 (\frac{1}{\epsilon})\big)$ and the gradient complexity is ${\rm Gc}(\epsilon) = \mathcal{O}\big(\frac{1}{\epsilon}\log^4 (\frac{1}{\epsilon})\big)$. 
\end{corollary}
\Cref{thm:main} shows that the proposed \Cref{alg:main} achieves a convergence rate of  $\mathcal{O}({\log^4(T)}/{T})$ and a gradient complexity of $\mathcal{O}\big(\frac{1}{\epsilon}\log^4 (\frac{1}{\epsilon})\big)$, both of which nearly match the results in ~\cite{ji2022will} of the standard well-tuned bilevel optimization methods up to polylogarithmic factors. 

\begin{remark}
Note that the difference of $\frac{1}{\epsilon}$ in gradient complexity between double-loop and single-loop methods has not been observed in previous works on well-tuned bilevel optimization.
This difference stems from the design of the sub-loops. In previous double-loop works, carefully selected stepsizes are used to ensure that the iterates of each sub-loop converge linearly, up to an approximation error caused by the shift in $x$. However, due to the precise control of stepsizes, tuning-free approaches can only guarantee a sub-linear convergence for each sub-loop (as shown in \Cref{prop:subloop}). 
\end{remark}

\vspace{-0.2cm}

\section{Experiments}
\vspace{-0.2cm}
In this section, we evaluate the effectiveness of our proposed algorithm on practical applications including regularization selection, data hyper-cleaning~\cite{franceschi2017forward}, and coreset selection for continual learning~\cite{hao2024bilevel}. Our implementation is based on the benchmark provided in \cite{dagreou2022framework} and \cite{hao2024bilevel}, respectively. {\em Please refer to \Cref{supp:experiments} for more details about practical implementation, experiment configurations, and additional plots.}
\vspace{-0.2cm}
\subsection{Regularization Selection}
\vspace{-0.1cm}
The selection of regularization can be framed as a bilevel optimization problem, where the inner objective focuses on optimizing the model parameters $\theta$ on the training set $\mathcal{S}_T=\{(d_i^{train},y_i^{train})\}_{1\le i\le n}$, while the outer objective aims to determine the best regularization term $\lambda$ on the validation set $\mathcal{S}_V=\{(d_j^{val}, y_j^{val})\}_{1\le j\le m}$. Denote the model parameters by $\theta\in\mathbb{R}^p$ and regularization term by $\lambda\in\mathbb{R}^p$, then the outer and inner problems can be formulated as 
\begin{align*}
    & f(\theta, \lambda) = \frac{1}{m}\sum_{j=1}^m l\big((d_j^{val}, y_j^{val}), \theta\big);\quad  g(\theta, \lambda) = \frac{1}{n} \sum_{i=1}^n l\big((d_i^{train}, y_i^{train}), \theta\big) + \mathcal{R}(\theta, \lambda),
\end{align*}
where the loss $l((d_i,y_i),\theta)=\log(1+\exp(-y_id_i^\top\theta)$, and $\mathcal{R}(\theta, \lambda)=\frac{1}{2}\sum_{k=1}^p \exp(\lambda_k)\theta_k^2$ represents the regularization, where each element $\theta_k$ is regularized with strength $\exp(\lambda_k)$. We compare our proposed algorithm with benchmark bilevel algorithms including AmIGO~\cite{arbel2021amortized}, BSA~\cite{ghadimi2018approximation}, FSLA~\cite{li2022fully}, MRBO~\cite{yang2021provably}, SOBA~\cite{dagreou2022framework}, StocBiO~\cite{ji2021bilevel}, SUSTAIN~\cite{khanduri2021near}, TTSA~\cite{hong2023two}, VRBO~\cite{yang2021provably} on the Covtype dataset. More details are provided in \Cref{supp:experiments}.

\begin{figure}[H]
\centering     %%% not \center
\subfigure[Covtype]{\label{fig:covtype}\includegraphics[width=80mm]{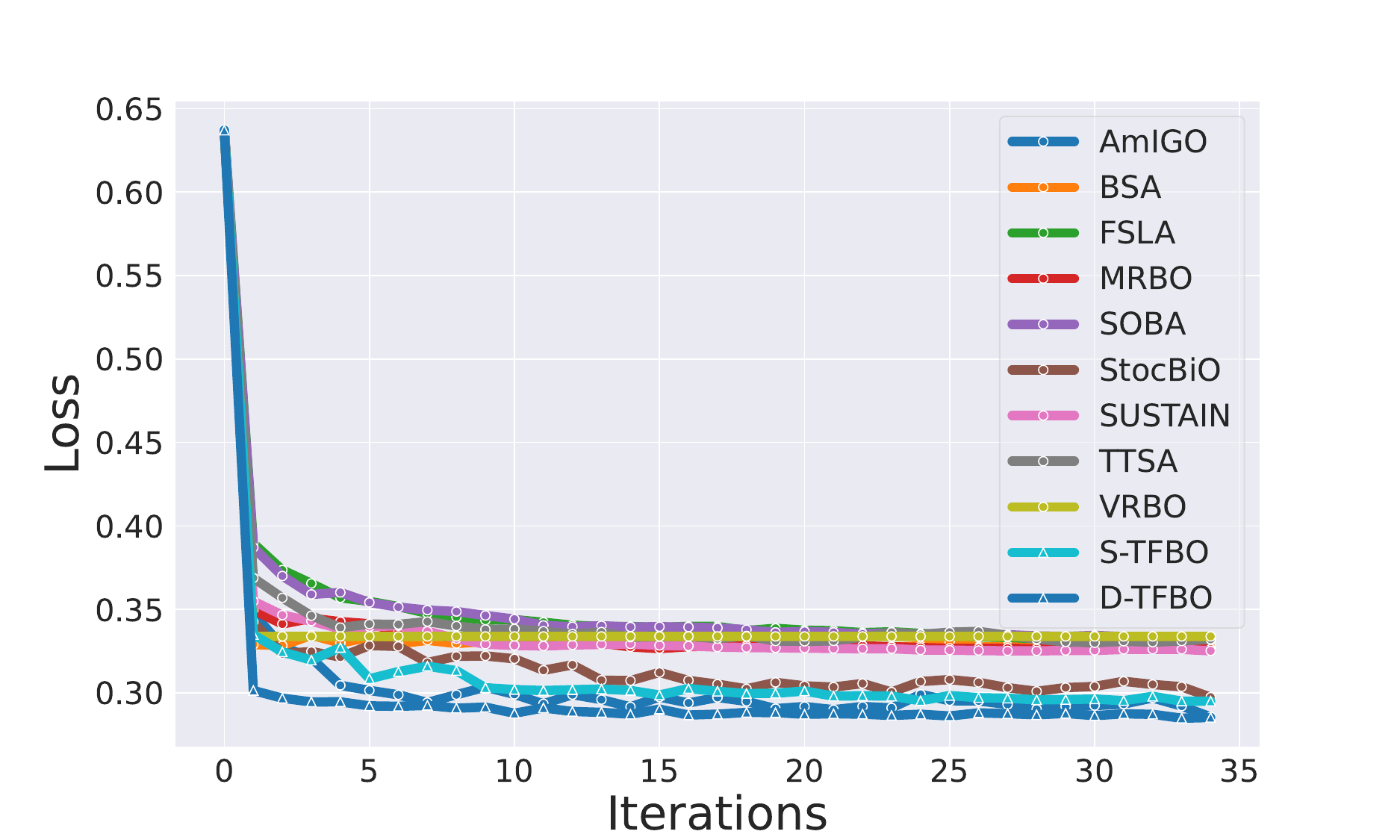}}
\subfigure[ MNIST]{\label{fig:mnist}\includegraphics[width=80mm]{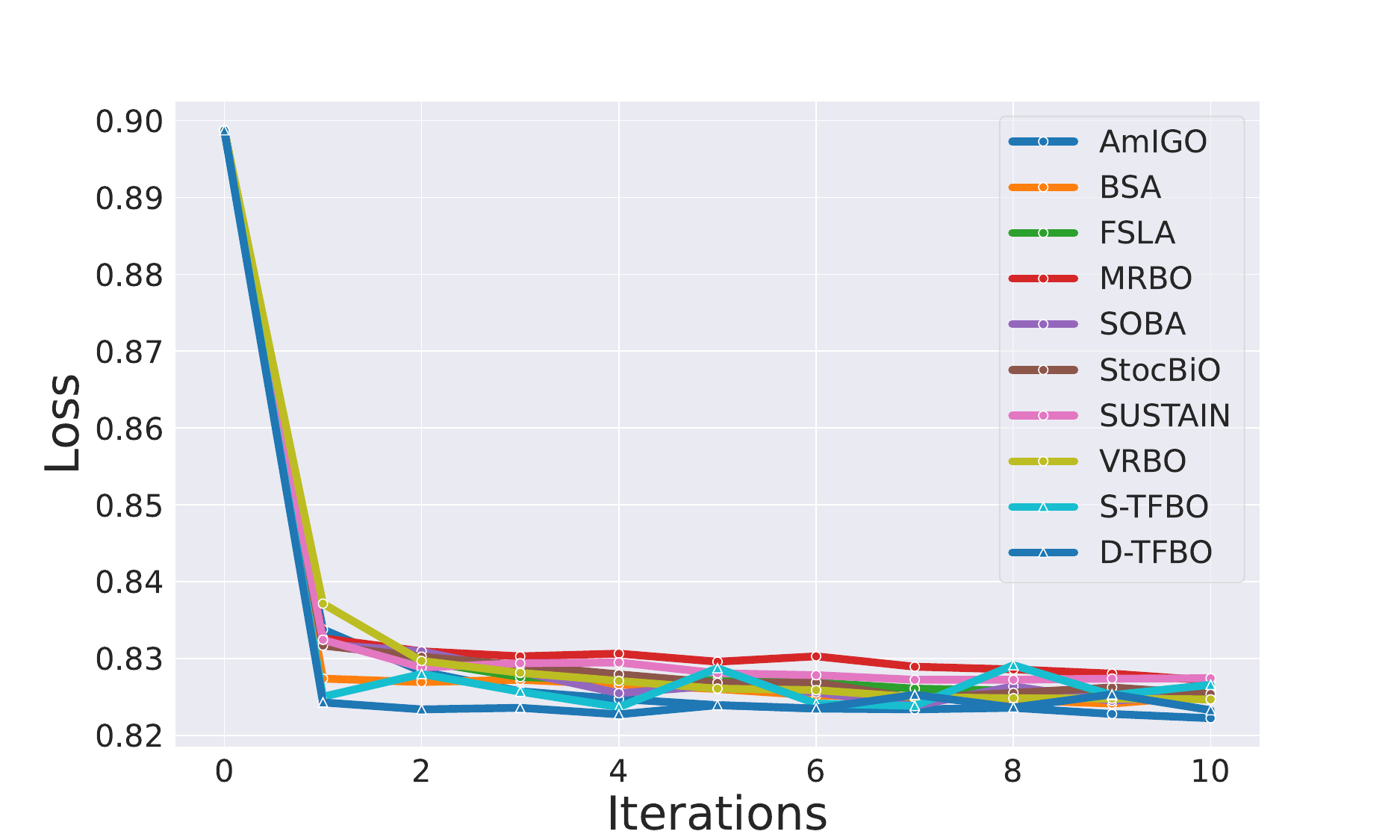}}
\caption{Comparison with other bilevel methods. (a) Regularization selection on Covtype dataset. (b) Data hyper-cleaning on MNIST dataset.}
\label{fig:fig1}
\end{figure}

As shown in \Cref{fig:covtype}, our D-TFBO achieves the fastest convergence rate, while S-TFBO converges slightly more slowly but remains comparable to other well-tuned methods. 

\vspace{-0.2cm}
\subsection{Data Hyper-Cleaning}
\vspace{-0.1cm}
The training set $\mathcal{S}_T=\{(d_i^{train}, y_i^{train})\}_{1\le i\le n}$ have been corrupted in this scenario, where the label of a data sample could be replaced by a random label with a certain probability $p$. It is important to note that we do not have prior knowledge about which data samples have been corrupted. The objective is to develop a model that can effectively fit the corrupted training set while performing well on the clean validation set $\mathcal{S}_V=\{(d_j^{val}, y_j^{val})\}_{1\le j\le m}$. We conduct experiments on the MNIST dataset, where we aim to learn a set of weights $\lambda$, one for each training sample, in addition to the model parameters $\theta$. Hence, the outer and inner problems are
\begin{align*}
    & f(\theta, \lambda) = \frac{1}{m}\sum_{j=1}^m l\big((d_j^{val},y_j^{val}), \theta\big);\quad g(\theta, \lambda) = \frac{1}{n} \sum_{i=1}^n \sigma(\lambda_i) l\big((d_i^{train}, y_i^{train}), \theta\big) + C\|\theta\|^2,
\end{align*}
where $\sigma(\cdot)$ is sigmoid function, $C$ is a regularization constant, and loss function $l((d_i, y_i), \theta)=1/(1+\exp(-y_id_i^\top\theta))$. Ideally, we would like the weights to be 0 for the corrupted sample and 1 for the clean sample. More details can be found in \Cref{supp:experiments}. We compare the performance with other bilevel optimization methods including AmIGO~\cite{arbel2021amortized}, BSA~\cite{ghadimi2018approximation}, FSLA~\cite{li2022fully}, MRBO~\cite{yang2021provably}, SOBA~\cite{dagreou2022framework}, StocBiO~\cite{ji2021bilevel}, SUSTAIN~\cite{khanduri2021near}, VRBO~\cite{yang2021provably}. The results presented in \Cref{fig:mnist} demonstrate that our algorithms achieve a convergence rate  comparable to other baselines.

\vspace{-0.2cm}
\subsection{Coreset Selection for Continual Learning}
\vspace{-0.1cm}
Coreset selection aims to improve training efficiency by selecting a subset of the most informative data samples, which can be used as an approximation of the entire dataset. Thus, the model that minimizes the loss on the coreset can also minimize the loss on the entire dataset. Following the design in \cite{hao2024bilevel}, we apply the proposed algorithms to coreset selection for continual learning. 
The inner problem learns model parameters $\theta$, and the outer problem determines the distribution $\lambda$ ($0\le\lambda_{(i)}\le 1$ and $\|\lambda\|_1=1$) over the entire dataset
\begin{align*}
    f(\theta, \lambda)=\sum_{i=1}^n l_i(\theta) + C \mathcal{R}(\lambda); \quad g(\theta, \lambda)=\sum_{i=1}^n \lambda_{(i)}l_i(\theta), 
\end{align*}
where $n$ is the sample size, $C$ is a constant, $\lambda_{(i)}$ is the $i$-th entry. $\mathcal{R}(\lambda)=-\sum_{i=1}^K\mathbb{E}(\lambda+\delta z)_{[i]}$ denotes the smoothed top-$K$ regularizer, where $\delta$ is a constant and $z\sim\mathcal{N}(0,1)$, $\lambda_{[i]}$ is the $i$-th largest component. The regularizer encourages the distribution to have $K$ non-zero entries, corresponding to the size of the selected coreset. Following \cite{zhou2022probabilistic}, we use the Split CIFAR100 dataset and conduct experiments in the balanced and imbalanced scenarios. We compare the proposed algorithms with various methods, including $k$-means features~\cite{nguyen2017variational}, $k$-means embedding~\cite{sener2017active}, Uniform Sampling, iCaRL~\cite{rebuffi2017icarl}, Grad Matching~\cite{campbell2019automated}, GCR~\cite{tiwari2022gcr}, Greedy Coreset~\cite{borsos2020coresets}, PBCS~\cite{zhou2022probabilistic}, and BCSR~\cite{hao2024bilevel}, with the last three being bilevel optimization-based methods. We evaluate the performance using the average accuracy and forgetting measure across all tasks after learning task $T$. The former is defined as $A_T=\frac{1}{T}\sum_{i=1}^T a_{T,i}$, where $a_{T,i}$ is the test accuracy of the $i$-th task after learning task $T$. The latter is defined as $FGT_T=\frac{1}{T}\sum_{i=1}^T[\max_{j\in 1,\cdots,T-1}(a_{j,i}-a_{T,i})]$. The results are shown in \Cref{tab:tab1}. Each experiment is repeated three times and the average is reported. It can be observed that our D-TFBO achieves the best $FGT_T$ under the balanced setting and the second-best performance under the imbalanced setting. 

\begin{table}[t]
\small
\caption{Results on Split CIFAR100. The best and second best results are in bold and  underlined.}
\label{tab:tab1}
\begin{center}
\begin{tabular}{lcccc}
\toprule
\multirow{2}*{Method} & \multicolumn{2}{c}{Balanced} & \multicolumn{2}{c}{Imbalanced} \\
\cmidrule(lr){2-3}\cmidrule(lr){4-5}
& $A_T$ & $FGT_T$ & $A_T$ & $FGT_T$ \\
\midrule
$k$-means features & 57.82$\pm$0.69 & 0.070$\pm$0.003 & 45.44$\pm$0.76 & 0.037$\pm$0.002 \\
$k$-means embedding & 59.77$\pm$0.24 & 0.061$\pm$0.001 & 43.91$\pm$0.15 &  0.044$\pm$0.001 \\
Uniform Sampling & 58.99$\pm$0.54 & 0.074$\pm$0.004 & 44.73$\pm$0.11 & 0.033$\pm$0.007 \\
iCaRL & \underline{60.74$\pm$0.09} & \underline{0.044$\pm$0.026} & 44.25$\pm$2.04 & 0.042$\pm$0.019 \\
Grad Matching & 59.17$\pm$0.38 & 0.067$\pm$0.003 & 45.44$\pm$0.64 & 0.038$\pm$0.001 \\
GCR & 58.73$\pm$0.43 & 0.073$\pm$0.013 & 44.48$\pm$0.05 &  0.035$\pm$0.005 \\
Greedy Coreset & 59.39$\pm$0.16 & 0.066$\pm$0.017 & 43.80$\pm$0.01 &  0.039$\pm$0.007 \\
PBCS &  55.64$\pm$2.26 & 0.062$\pm$0.001 &  39.87$\pm$1.12 &  0.076$\pm$0.011 \\
BCSR & \textbf{61.60$\pm$0.14} & 0.051$\pm$0.015 & \textbf{47.30$\pm$0.57} & \textbf{0.022$\pm$0.005} \\
\midrule
S-TFBO & 58.90$\pm$0.75 & 0.046$\pm$0.009 & 45.78$\pm$0.70 & 0.036$\pm$0.005 \\
D-TFBO & 59.54$\pm$0.45 & \textbf{0.041$\pm$0.005} & \underline{46.68$\pm$0.72} & \underline{0.029$\pm$0.002} \\
\bottomrule
\end{tabular}
\end{center}
\end{table}

\begin{table}[t]\vspace{-0.2cm}
\small
\caption{Experiment results of sensitivity analysis on Split CIFAR100. The initial values refer to the constant learning rates in BCSR or $\alpha_0$,$\beta_0$,$\gamma_0$ in S-TFBO and D-TFBO.}
\label{tab:tab2}
\begin{center}
\begin{tabular}{lccccc}
\toprule
Method & $initial=2$ & $initial=4$ & $initial=6$ & $initial=8$ & Relative Average Change \\
\midrule
BCSR & 59.42 & 56.25 & 58.75 & 57.55 & 5.8$\%$ \\
S-TFBO & 58.85 & 58.55 & 58.69 & 58.47 & {\bf 0.4$\%$} \\
D-TFBO & 59.71 & 59.62 & 59.11 & 59.08 &  {\bf 0.3$\%$} \\
\bottomrule
\end{tabular}
\end{center}
\vspace{-0.5cm}
\end{table}

\noindent
{\bf Sensitivity analysis w.r.t.~different initial learning rates.} The tuning-free design provides another benefit. The proposed algorithms demonstrate more robustness compared to the \cite{hao2024bilevel}. We conduct a simple sensitivity analysis under the balanced setting, regarding the learning rates in the inner and outer loops. Specifically, we set the initial learning rates in \cite{hao2024bilevel} and $\alpha_0$, $\beta_0$, $\gamma_0$ in S-TFBO and D-TFBO for the inner and outer loops to $\{2,4,6,8\}$, where the original values are set to 5. We run one experiment for each learning rate. Further, we compare the changes in average accuracy $A_T$. We also compute the average and report the relative change compared to the results presented in \Cref{tab:tab1}. 
The code is available at \url{https://github.com/OptMN-Lab/tfbo}.

\vspace{-0.2cm}
\section{Conclusion}
We introduce two fully tuning-free bilevel optimization algorithms, D-TFBO and S-TFBO. Both methods adaptively update stepsizes without requiring prior knowledge of problem parameters, while achieving convergence rates comparable to their well-tuned counterparts. The experimental results show that our tuning-free design performs comparably to existing well-tuned methods and is more robust to initial stepsizes. We anticipate that the proposed algorithms and the developed analysis can be extended to the stochastic setting, 
and the proposed algorithms may be applied to other applications such
as meta-learning, few-shot learning, and fair machine learning.

\bibliography{ref}
\bibliographystyle{abbrv}

\newpage
\appendix
{\huge \bf Supplementary material}

\section{Supplementary Related Work on Bilevel Optimization }
Initially introduced by \cite{bracken1973mathematical}, bilevel optimization has been extensively studied for decades. Early works~\cite{hansen1992new, shi2005extended, gould2016differentiating, sinha2017review} solved the bilevel problem from a constrained optimization perspective. More recently, gradient-based bilevel methods have gained significant attention for their efficiency and effectiveness in addressing machine learning problems. Among them, approaches based on Approximate Implicit Differentiation (AID)~\cite{domke2012generic, liao2018reviving, pedregosa2016hyperparameter, lorraine2020optimizing, grazzi2020iteration, ji2021bilevel, arbel2021amortized, hong2023two} exploit the implicit derivation of the hypergradient, approximating it by solving a linear system.

On the other hand, approaches based on Iterative Differentiation (ITD)~\cite{maclaurin2015gradient, franceschi2017forward, finn2017model, shaban2019truncated, grazzi2020iteration} estimate the hypergradient by employing automatic differentiation, utilizing either forward or reverse mode.

A series of stochastic bilevel approaches has been developed and analyzed recently, utilizing Neumann series~\cite{chen2021single, ji2021bilevel, arbel2021amortized}, recursive momentum~\cite{yang2021provably, huang2021biadam, guo2021randomized}, and variance reduction~\cite{yang2021provably, dagreou2022framework}, etc. For the lower-level problem with multiple solutions, several approaches were proposed based on upper- and lower-level gradient aggregation~\cite{sabach2017first, liu2020generic, li2020improved}, barrier types of regularization~\cite{liu2021value, liu2022bome}, penalty-based formulations~\cite{shen2023penalty}, primal-dual techniques~\cite{sow2022constrained}, and dynamic system-based methods~\cite{liu2021towards}. Another class of approaches formulated the lower-level problem as a value-function-based constraint~\cite{kwon2023fully, wang2023effective} to solve bilevel problems without second-order gradients.

\section{Specifications of Experiments}\label{supp:experiments}
\subsection{Practical Implementation} 
For regularization selection and data hyper-cleaning, we use the benchmark provided in \cite{dagreou2022framework}. For coreset selection, we use the codebase from \cite{hao2024bilevel}. We implement D-TFBO using ``for loops'' as an approximation, since the magnitude of $\|\nabla_v R(x,y,v)\|$ in \Cref{alg:main_double} varies across different experiments. Specifically, the number of loops for updating $y$ and $v$ in regularization selection and data hyper-cleaning are both set to 10, while the numbers of loops for updating $y$ and $v$ in coreset selection are 5 and 3, respectively.

\begin{figure}[H]
\centering     %%% not \center
\subfigure[Covtype]{\label{fig:covtype_supp}\includegraphics[width=80mm]{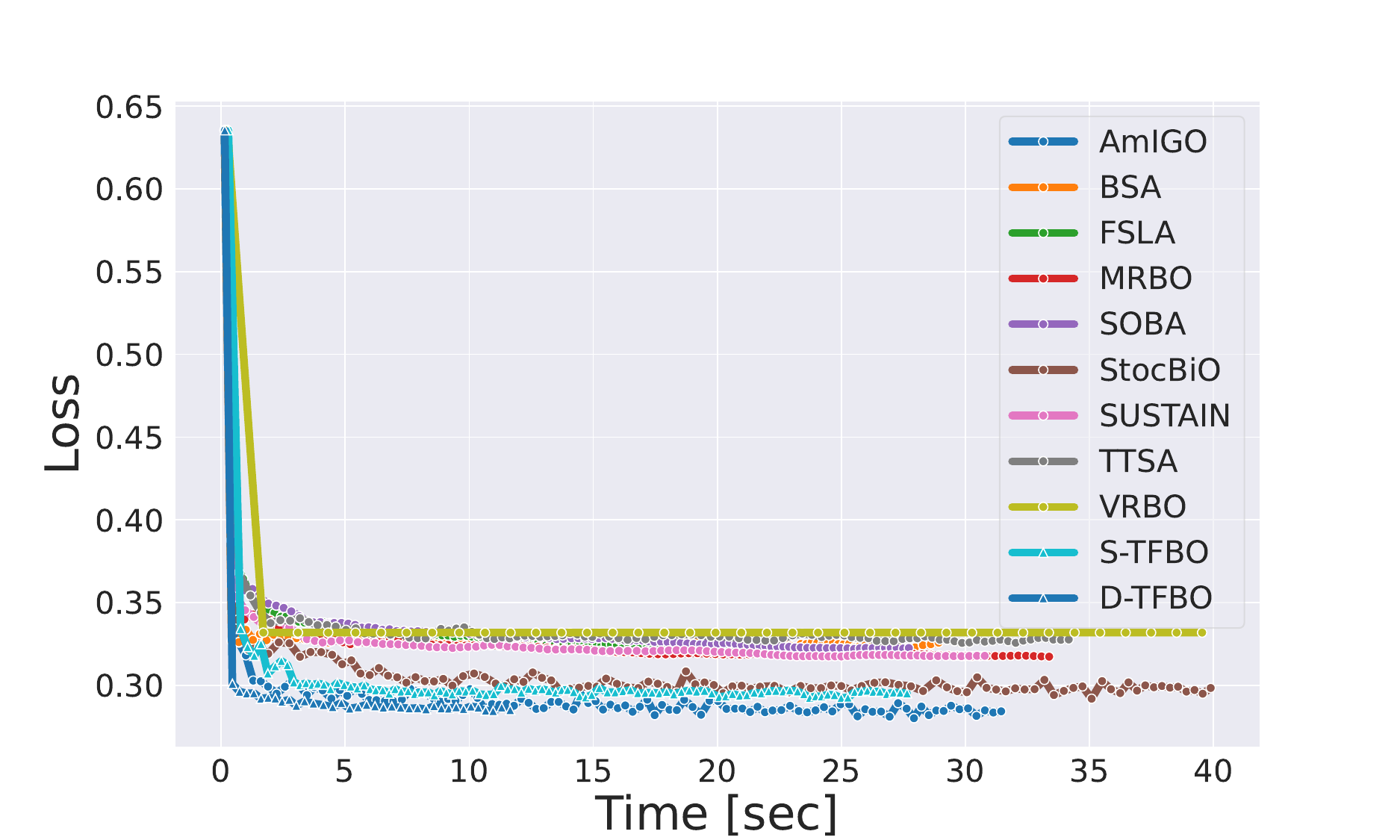}}
\subfigure[ MNIST]{\label{fig:mnist_supp}\includegraphics[width=80mm]{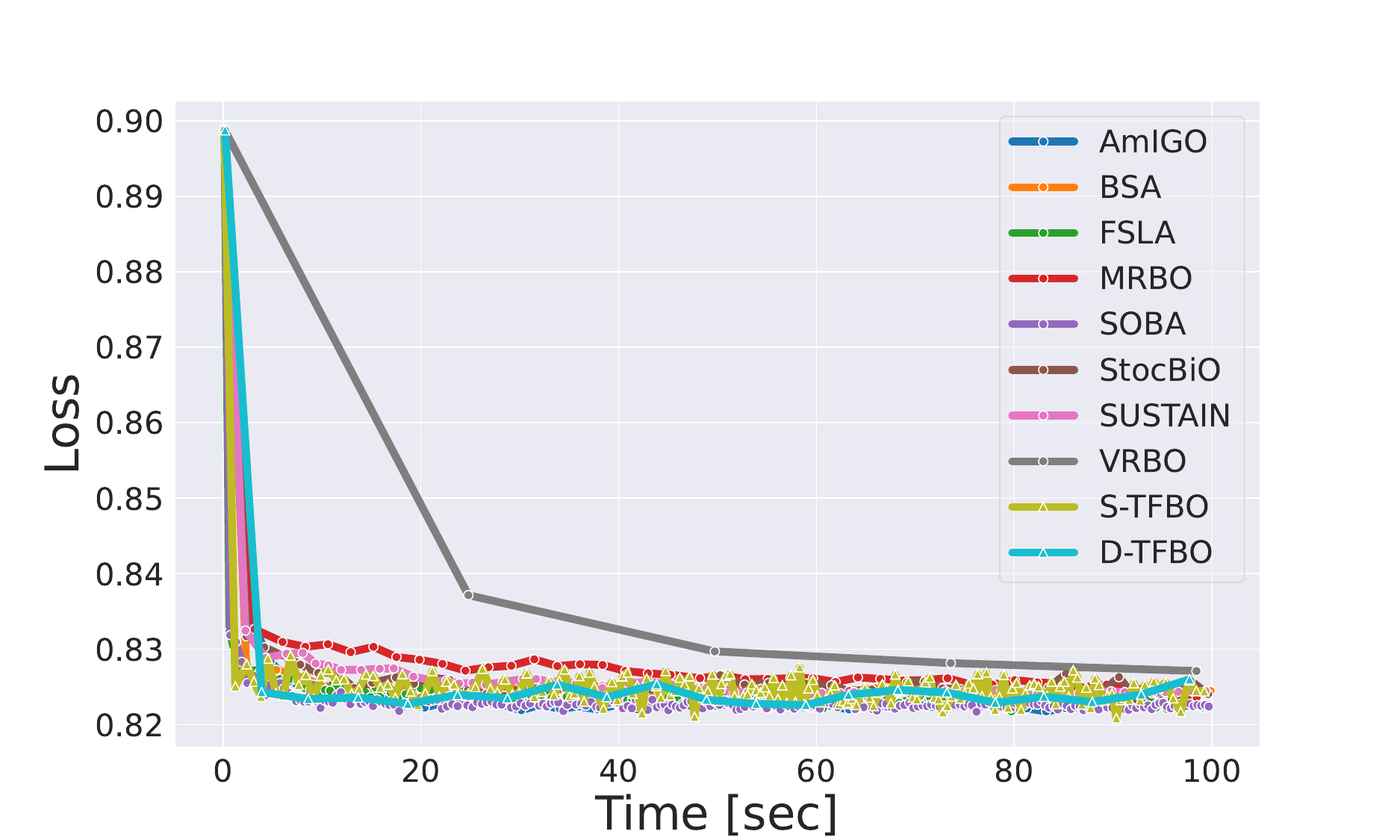}}
\caption{Comparison of running time with other bilevel optimization methods.}
\label{fig:fig2}
\end{figure}

\subsection{Configuration} 
We adopt the default configuration for regularization selection and data hyper-cleaning. The batch size is 64. The maximum iterations are 2048 and 512, respectively. The data corruption ratio in hyper-cleaning is 0.1. For coreset selection, we also use the default configuration except for the leaning rates, due to the tuning-free design. The $\alpha_0$, $\beta_0$, and $\gamma_0$ values are set to 5.

\subsection{Additional Results}\label{appen:add}
For regularization and data hyper-cleaning, we also present the loss curves regarding running time in \Cref{fig:fig2}. Our methods exhibit a faster running time than other baselines on the Covtype dataset.

\section{Proof Sketch}\label{proof_sktech}
The proofs of Propositions \ref{prop:bar}, \ref{prop:subloop}, \ref{prop:objective}, \ref{prop:somebounds} and \ref{prop:alpha} can be found in \Cref{lm:bar}, \ref{lm:subloop}, \ref{lm:objective}, \ref{lm:varphi},\ref{lm:alpha}, respectively. In this section, we present a high level proof sketch that outlines the convergence and gradient complexity analysis of \Cref{alg:main_double} and \Cref{alg:main}, emphasizing the key challenges and our technical innovations.

\noindent
\textbf{Proof sketch of \Cref{alg:main_double}:}

\noindent
\textbf{Step 1:} We first discuss the two-stage framework in our problem in \Cref{lm:bar} and we develop two forms of descent lemma of the objective function in \Cref{lm:objective_double} based on the two stages of $\alpha_t$ in \Cref{lm:bar}. 

\noindent
\textbf{Step 2:} We developed upper bounds of $\alpha_t$ under the two stages in \Cref{lm:bar}. 

\noindent
\textbf{Step 3:} We provide the maximum iteration numbers for the sub-loops approximating $y^*(x_t)$ and $v^*(x_t)$. 

\noindent
\textbf{Step 4:} Combining the results in Step 1 and Step 2, we telescope and take the average of the inequalities in the descent lemma of the objective function, then we obtain the convergence rate. 

\noindent
\textbf{Step 5:} Combining the maximum iteration numbers in Step 3 and convergence rate in Step 4, we obtain the gradient computation complexity to find $\epsilon$-stationary point. 
Then the proof is complete. 

\noindent
\textbf{Proof sketch of \Cref{alg:main}:}

\noindent
\textbf{Step 1:} We first discuss the two-stage framework in our problem in \Cref{lm:bar} and we develop two forms of descent lemma of the objective function in \Cref{lm:objective} based on the two stages of $\alpha_t$ in \Cref{lm:bar}. 

\noindent
\textbf{Step 2:} We develop a rough upper bound of two important components in the descent lemma in \Cref{lm:objective}:
$\sum_{k=k_2}^t \frac{\|\nabla_y g(x_k, y_k)\|^2}{\beta_{k+1}}$ and $\sum_{k=k_3}^t \frac{\|\nabla_v R(x_k, y_k, v_k)\|^2}{\varphi_{k+1}}$, where $k_2$ and $k_3$ represents the second stage in \Cref{lm:bar}. 

\noindent
\textbf{Step 3:} Following the results in Step 2 and the upper bound of $v_t$ in \Cref{lm:vk}, 
we obtain a two-way relationship between $\varphi_{t+1}$ and 
$\sum_{k=0}^t \frac{\|\bar{\nabla}f(x_k, y_k, v_k)\|^2}{\alpha_{k+1}^2}$, which further indicates the logarithmic upper bounds of both terms in \Cref{lm:varphi} and \Cref{lm:somebounds}, respectively. 

\noindent
\textbf{Step 4:} Incorporating the results from Step 3 into the rough bounds from Step 2, we can also obtain the logarithmic upper bounds of $\sum_{k=k_2}^t \frac{\|\nabla_y g(x_k, y_k)\|^2}{\beta_{k+1}}$ and $\sum_{k=k_3}^t \frac{\|\nabla_v R(x_k, y_k, v_k)\|^2}{\varphi_{k+1}}$ in \Cref{lm:somebounds}. 

\noindent
\textbf{Step 5:} We rearrange the terms in \Cref{lm:objective} and 
incorporate in the results in Step 4, we obtain two forms of the upper bound of $\alpha_{t}$ in \Cref{lm:alpha}. 

\noindent
\textbf{Step 6:} 
Combining the results in Steps 3, 4, 5, we telescope and take the average of the inequalities in the descent lemma of the objective function, then we obtain the convergence rate. 

\noindent
\textbf{Step 7:} Without sub-loops, via the convergence rate in Step 6, 
we can directly obtain the gradient computation complexity to find $\epsilon$-stationary point. 
Then the proof is complete. 

\newpage
\clearpage
\section{Proofs of Preliminary Lemmas}
\begin{lemma}[\cite{ward2020adagrad} Lemma 3.2]\label{lm:log}
    For any non-negative $a_1,..., a_T$, and $a_1 \geq 1$, we have
    \begin{align}
        \sum_{l=1}^T\frac{a_l}{\sum_{i=1}^l a_i} \leq \log\Bigg(\sum_{l=1}^T a_l\Bigg)+1.
    \end{align}
\end{lemma}

\begin{lemma}\label{lm:basic}
Under Assumptions \ref{as:sc}, \ref{as:lip}, we have basic properties as follows:
\begin{enumerate}[label=(\alph*)]   
\item $\Phi(x)$ is $L_{\Phi}$-smooth w.r.t $x$, where $L_{\Phi} := \Big(L_{f,1}+\frac{L_{g,2}C_{f_y}}{\mu}\Big)\Big(1+\frac{C_{g_{xy}}}{\mu}\Big)^2$;
\item $y^*(x)$ is $L_y$-Lipschitz continuous w.r.t. $x$, where $L_y:=\frac{C_{g_{xy}}}{\mu}$; 
\item the gradient estimator $\bar{\nabla} f(x,y,v)$ is $(L_{g,2}\|v\|+L_{f,1})$ -Lipschitz continuous w.r.t. $(x,y)$, and $L_{g,1}$-Lipschitz continuous w.r.t. $v$;
\item $\bar{\nabla} f(x,y,v)$ can be bounded as $\|\bar{\nabla} f(x,y,v)\| \leq C_{g_{xy}}\|v\|+C_{f_x}$. 
\end{enumerate}
\end{lemma}
\begin{proof}
The proof of (a) and (b) can refer to \cite{ghadimi2018approximation}. For (c), under Assumption \ref{as:lip}, we have 
\begin{align}
    \|\bar{\nabla}f(x_1,y_1,v) - \bar{\nabla}f(x_2,y_2,v)\| \leq& \|\nabla_x\nabla_y g(x_1,y_1) - \nabla_x\nabla_y g(x_2,y_2)\|\cdot\|v\| \nonumber \\
    & + \|\nabla_x f(x_1,y_1) - \nabla_x f(x_2,y_2)\|\nonumber \\
    \leq & (L_{g,2}\|v\|+L_{f,1})(\|x_1-x_2\| + \|y_1-y_2\|) \nonumber \\
    \|\bar{\nabla}f(x,y,v_1) - \bar{\nabla}f(x,y,v_2)\| \leq& \|\nabla_x\nabla_y g(x,y)\|\cdot\|v_1-v_2\| \leq L_{g,1}\|v_1-v_2\|. \nonumber
\end{align}
By Assumption \ref{as:lip} and Remark \ref{as:grad}, we can easily prove (d) as 
\begin{align}
    \|\bar{\nabla} f(x,y,v)\| \leq \|\nabla_x\nabla_y g(x,y)\|\cdot\|v\| + \|\nabla_xf(x,y)\| \leq C_{g_{xy}}\|v\|+C_{f_x}. \nonumber
\end{align}
Then the proof is complete.
\end{proof}

\begin{lemma}\label{lm:LS}
Under Assumptions \ref{as:sc}, \ref{as:lip}, we have basic properties of linear system function $R$ in \cref{eq:defR} as follows:
\begin{enumerate}[label=(\alph*)]
\item $R(x,y,v)$ is $\mu$-strongly convex and $C_{g_{yy}}$-smooth w.r.t. $v$;
\item $\nabla_v R(x,y,v)$ is $(L_{g,2}\|v\|+L_{f,1})$-Lipschitz continuous w.r.t. $(x,y)$;
\item $\nabla_v R(x,y,v)$ can be bounded as $\|\nabla_v R(x, y, v)\| \leq C_{g_{yy}}\|v\|+C_{f_y}$; 
\item $v^*(x)$ in \cref{eq:defR} can be bounded as $\|v^*(x)\| \leq \frac{C_{f_y}}{\mu}$, and {\small$\hat{v}^*(x, y):=\argmin_v R(x, y, v)$} can also be bounded as $\|\hat{v}^*(x,y)\| \leq \frac{C_{f_y}}{\mu}$;
\item $v^*(x)$ is $L_v$-Lipschitz continuous w.r.t. $x$ and $\hat{v}^*(x,y)$ is $\bar{L}_v$-Lipschitz continuous w.r.t. $y$, 
where $L_v:= \Big(\frac{L_{f,1}}{\mu}+\frac{C_{f_y}L_{g,2}}{\mu^2}\Big)(1+L_y)$ and $\bar{L}_v := \frac{L_{f,1}}{\mu}+\frac{C_{f_y}L_{g,2}}{\mu^2}$.
\end{enumerate}
\end{lemma}
\begin{proof}
First of all, since $\nabla_v\nabla_v R(x,y,v) = \nabla_y\nabla_y g(x,y)$, we know $\mu I \preceq \nabla_y\nabla_y g(x,y)$. Thus, according to Assumptions\ref{as:sc},\ref{as:lip}, we have 
\begin{align}
    \|\nabla_v\nabla_v R(x,y,v_1) - \nabla_v\nabla_v R(x,y,v_2)\| \leq \|\nabla_y\nabla_y g(x,y)\|\|v_1 - v_2\| \leq C_{g_{yy}}\|v_1 - v_2\|. \nonumber
\end{align}
Then (a) is proved. 
Next, by using Lipschitz continuity in Assumption \ref{as:lip}, we have 
\begin{align}
    \|\nabla_v R(x_1,y_1,v) - \nabla_v R(x_2,y_2,v)\| \leq& \|\nabla_y\nabla_y g(x_1,y_1) - \nabla_y\nabla_y g(x_2,y_2)\|\cdot\|v\| \nonumber \\
    & + \|\nabla_y f(x_1,y_1) - \nabla_y f(x_2,y_2)\|\nonumber \\
    \leq & (L_{g,2}\|v\|+L_{f,1})(\|x_1-x_2\| + \|y_1-y_2\|).  \nonumber 
\end{align}
Then (b) is proved. By Assumption \ref{as:lip}, we can easily prove (c) as 
\begin{align}
    \|\nabla_v R(x, y, v)\| \leq \|\nabla_y\nabla_y g(x,y)\|\cdot\|v\| + \|\nabla_yf(x,y)\| \leq C_{g_{yy}}\|v\|+C_{f_y}. \nonumber
\end{align}
Next, for $\hat{v}^*(x, y)$, we have 
\begin{align}
    \nabla_v R\big(x,y,\hat{v}^*(x, y)\big) = \nabla_y\nabla_y g(x,y)\hat{v}^*(x, y) - \nabla_y f(x,y) = 0, \nonumber
\end{align}
which indicates that 
\begin{align}
    \norm{\hat{v}^*(x, y)} =& \big\|\big[\nabla_y\nabla_yg(x,y)\big]^{-1}\nabla_y f(x,y)\big\|
    \leq \big\|\big[\nabla_y\nabla_yg(x,y)\big]^{-1}\big\|\cdot \|\nabla_y f(x,y))\| \leq \frac{C_{f_y}}{\mu}. \nonumber
\end{align}
Since $v^*(x)$ is a special case as $v^*(x) = \hat{v}^*(x, y^*(x))$, (d) is proved. 
The proof of the first part of (e) can refer to Lemma 4 in \cite{yang2024simfbo};
for the second part, we have
\begingroup
\allowdisplaybreaks
{\small
\begin{align}
    \|\hat{v}^*&(x, y_1) - \hat{v}^*(x, y_2)\| \nonumber \\
    = & \big\|[\nabla_y\nabla_y g(x,y_1)]^{-1}\nabla_y f(x,y_1) - [\nabla_y\nabla_y g(x,y_2)]^{-1}\nabla_y f(x,y_2)\big\| \nonumber \\
    \leq & \big\|[\nabla_y\nabla_y g(x,y_1)]^{-1}\big(\nabla_y f(x,y_1) - \nabla_y f(x,y_2)\big)\big\| \nonumber \\
    & + \big\| \big([\nabla_y\nabla_y g(x,y_1)]^{-1} - [\nabla_y\nabla_y g(x,y_2)]^{-1}\big)\nabla_y f(x,y_2) \big\| \nonumber \\
    \leq & \frac{L_{f,1}}{\mu}\|y_1 - y_2\| + C_{f_y}\big\|\big([\nabla_y\nabla_y g(x,y_1)]^{-1}\big(\nabla_y\nabla_y g(x,y_2) - \nabla_y\nabla_y g(x,y_1)\big)[\nabla_y\nabla_y g(x,y_2)]^{-1}\big)\big\| \nonumber \\
    \leq & \bigg(\frac{L_{f,1}}{\mu} + \frac{C_{f_y}L_{g,2}}{\mu^2}\bigg)\|y_1 - y_2\|. \nonumber 
\end{align}}
\endgroup
Thus, the second part of (e) is proved and the proof of \Cref{lm:LS} is complete. 
\end{proof}

{
\begin{lemma}\label{lm:bar}
Suppose the iteration rounds to update $\{x,y,v\}$ are $\{T_1,T_2,T_3\}$ and $\{\alpha_t, \beta_t, \gamma_t\}$ are generated by \Cref{alg:main_double} or \ref{alg:main}. For any $C_\alpha \geq \alpha_0$, $C_\beta \geq \beta_0$, $C_\gamma \geq \gamma_0$, we have 
\begin{enumerate}[label=(\alph*)]
\item either $\alpha_t \leq C_\alpha$ for any $t \leq T_1$, or $\exists k_1 \leq T_1$ such that $\alpha_{k_1} \leq C_\alpha$, $\alpha_{k_1+1} > C_\alpha$; 
\item either $\beta_t \leq C_\beta$ for any $t \leq T_2$, or $\exists k_2 \leq T_2$ such that $\beta_{k_2} \leq C_\beta$, $\beta_{k_2+1} > C_\beta$; 
\item either $\gamma_t \leq C_\gamma$ for any $t \leq T_3$, or $\exists k_3 \leq T_3$ such that $\gamma_{k_3} \leq C_\gamma$, $\gamma_{k_3+1} > C_\gamma$.
\end{enumerate}
\end{lemma}
}
\begin{proof}
The proof resembles the Lemma 4.1 in \cite{ward2020adagrad}. Here we only prove part (a), and the other two are similar. 
Note that if $\alpha_{T_1}>C_\alpha$, then there must exist $k_1\leq T_1$ such that $\alpha_{k_1} \leq C_\alpha$, $\alpha_{k_1+1} > C_\alpha$, because $C_\alpha\geq\alpha_0$ and the sequence $\{\alpha_k\}$ is monotonically increasing. Otherwise, we have $\alpha_t \leq \alpha_{T_1} \leq C_\alpha$ for any $t \leq T_1$. 
This completes the proof of part (a). 
\end{proof}

\newpage
\clearpage

\section{Proof of \Cref{thm:main_double}}
We define some notation for convenience before proving \Cref{thm:main_double}. 
\subsection{Notation}
Here, we define the following constants as thresholds for parameters $\beta_{p}$, $\gamma_{q}$, $\alpha_{t}$ in \Cref{alg:main_double} as 
\begin{align}\label{def:C_double}
    C_\alpha := \max\big\{2L_{\Phi}, \alpha_0 \big\}, \quad
    C_\beta := \max\big\{L_{g,1}, \beta_0\big\}, \quad 
    C_\gamma := \max\big\{C_{g_{yy}}, \gamma_0\big\}.
\end{align}

\subsection{Proofs of Preliminary Lemmas}
\begin{lemma}\label{lm:innererror}
    Under Assumptions \ref{as:sc}, \ref{as:lip}, for any $t \geq 0$ in \Cref{alg:main_double}, we have
\begin{align}
    \|y_t^{P_t} - y^*(x_t)\|^2 \leq \frac{\epsilon_y}{\mu^2}, \quad 
    \big\|v_t^{Q_t} - \hat{v}^*(x_t, y_t^{P_t})\big\|^2 \leq \frac{\epsilon_v}{\mu^2}, \nonumber
\end{align}
where $\epsilon_y$ and $\epsilon_v$ are sub-loop stopping criteria in \Cref{alg:main_double}. 
\end{lemma}
\begin{proof}
For the $k_{th}$ iteration, according to the stop criteria of the sub-loops, we have 
\begin{align}
    \|\nabla_y g(x_t, y_t^{P_t})\|^2\leq\epsilon_y, \quad \|\nabla_v R(x_t, y_t^{P_t}, v_t^{Q_t})\|^2\leq\epsilon_v. \nonumber
\end{align}
By using Assumptions \ref{as:sc},\ref{as:lip}, we have
{
\begin{align}
    &\|y_t^{P_t} - y^*(x_t)\|^2 \leq \frac{1}{\mu^2}\big\|\nabla_y g(x_t, y_t^{P_t}) - \nabla_y g\big(x_t, y^*(x_t)\big)\big\|^2 
    \leq \frac{\epsilon_y}{\mu^2}, \nonumber \\
    &\big\|v_t^{Q_t} - \hat{v}^*(x_t, y_t^{P_t})\big\|^2 \leq \frac{1}{\mu^2}\big\|\nabla_v R(x_t, y_t^{P_t}, v_t^{Q_t}) - \nabla_v R\big(x_t, y_t^{P_t}, \hat{v}^*(x_t, y_t^{P_t})\big)\big\|^2 \leq \frac{\epsilon_v}{\mu^2}, \nonumber
\end{align}}
since $\|\nabla_y g(x_t, y^*(x_t))\|^2 = 0$ and $\big\|\nabla_v R\big(x_t, y_t^{P_t}, \hat{v}^*(x_t, y_t^{P_t})\big)\big\|^2 = 0$. 
Thus, the proof is complete. 
\end{proof}

\begin{lemma}\label{lm:boundoff_double}
    Under Assumptions \ref{as:sc}, \ref{as:lip}, for any $t \geq 0$ in Algorithm \ref{alg:main_double}, we have $\|\Bar{\nabla}f(x_t, y_t^{P_t}, v_t^{Q_t})\|^2 \leq C_{f}^2$, where $C_f := \Big(\frac{2C_{g_{xy}}^2\epsilon_v}{\mu^2} + \frac{4C_{g_{xy}}^2C_{f_y}^2}{\mu^2} + 4C_{f_y}^2\Big)^{\frac{1}{2}}$.
\end{lemma}
\begin{proof}
For the $k_{th}$ iteration, we have
\begingroup
\allowdisplaybreaks
{
\small
{
\begin{align}
    \|\Bar{\nabla}f(&x_t, y_t^{P_t}, v_t^{Q_t})\|^2 \nonumber \\
    \leq& 2\big\|\Bar{\nabla}f(x_t, y_t^{P_t}, v_t^{Q_t}) - \Bar{\nabla}f\big(x_t, y_t^{P_t}, \hat{v}^*(x_t, y_t^{P_t})\big)\big\|^2 + 2\big\|\Bar{\nabla}f\big(x_t, y_t^{P_t}, \hat{v}^*(x_t, y_t^{P_t})\big)\big\|^2 \nonumber \\
    =& 2\big\|\nabla_x\nabla_y g(x_t, y_t^{P_t})\big(v_t^{Q_t}-\hat{v}^*(x_t, y_t^{P_t})\big)\big\|^2  + 2\|\nabla_x\nabla_y g(x_t, y_t^{P_t})\hat{v}^*(x_t, y_t^{P_t}) - \nabla_y f(x_t, y_t^{P_t})\|^2 \nonumber \\
    \leq & 2\big\|\nabla_x\nabla_y g(x_t, y_t^{P_t})\big\|^2\cdot\|v_t^{Q_t}-\hat{v}^*(x_t, y_t^{P_t})\|^2   + 2\|\nabla_x\nabla_y g(x_t, y_t^{P_t})\hat{v}^*(x_t, y_t^{P_t}) - \nabla_y f(x_t, y_t^{P_t})\|^2 \nonumber \\
    \overset{(a)}{\leq}& \frac{2C_{g_{xy}}^2\epsilon_v}{\mu^2} + \frac{4C_{g_{xy}}^2C_{f_y}^2}{\mu^2} + 4C_{f_y}^2, \nonumber
\end{align}
}
}
\endgroup
where (a) uses Assumption \ref{as:sc}, \Cref{as:grad}, \Cref{lm:LS} and \Cref{lm:innererror}. Then, the proof is complete. 
\end{proof}

\subsection{Descent in Objective Function}
\begin{lemma}\label{lm:objective_double}
    Under Assumptions \ref{as:sc}, \ref{as:lip}, for Algorithm \ref{alg:main_double}, suppose the total iteration number is $T$. 
    No matter $k_1$ in \Cref{lm:bar} exists or not, we always have
\begin{align}\label{eq:objective_title_double1}
    \Phi(x_{t+1}) \leq& \Phi(x_t) - \frac{1}{2\alpha_{t+1}}\|\nabla \Phi(x_t)\|^2 - \frac{1}{2\alpha_{t+1}}\Big(1-\frac{L_{\Phi}}{2\alpha_{t+1}}\Big)\|\bar{\nabla}f(x_t, y_t^{P_t}, v_t^{Q_t})\|^2 + \frac{\epsilon'}{2\alpha_{t+1}}. 
\end{align}
If in addition, $k_1$ in \Cref{lm:bar} exists, then for $t\geq k_1$, we further have 
\begin{align}\label{eq:objective_title_double2}
    \Phi(x_{t+1}) 
    \leq& \Phi(x_t) - \frac{1}{2\alpha_{t+1}}\|\nabla \Phi(x_t)\|^2 - \frac{1}{4\alpha_{t+1}}\|\bar{\nabla}f(x_t, y_t^{P_t}, v_t^{Q_t})\|^2 + \frac{\epsilon'}{2\alpha_{t+1}}, 
\end{align}
where 
$\epsilon' := \frac{\bar{L}^2}{\mu^2}(\epsilon_y + \epsilon_v)$ and $\bar{L}:= \max\big\{2\big(\frac{C_{f_y}^2L_{g,2}^2}{\mu^2} + L_{f,1}^2 + C_{g_{yy}}^2\bar{L}_v^2\big)^{\frac{1}{2}}, \sqrt{2}C_{g_{yy}}\big\}$.  
\end{lemma}
\begin{proof} 
From \Cref{lm:basic}, we have $\Phi(x)$ is $L_{\Phi}$-smooth. So we can apply the descent lemma to $\Phi$ as 
\begingroup
\allowdisplaybreaks
\begin{align}\label{eq:doubledescent1}
    \Phi(x_{t+1}) \leq& \Phi(x_t) + \langle\nabla\Phi(x_t), x_{t+1} - x_t \rangle + \frac{L_{\Phi}}{2}\|x_{t+1} - x_t\|^2 \nonumber \\
    = & \Phi(x_t) - \frac{1}{\alpha_{t+1} }\big\langle\nabla\Phi(x_t), \bar{\nabla}f\big(x_t, y_t^{P_t}, v_t^{Q_t}\big) \big\rangle + \frac{L_{\Phi}}{2\alpha_{t+1}^2 }\big\|\bar{\nabla}f\big(x_t, y_t^{P_t}, v_t^{Q_t}\big)\big\|^2 \nonumber \\
    = & \Phi(x_t) - \frac{1}{2\alpha_{t+1} }\|\nabla \Phi(x)\|^2 - \frac{1}{2\alpha_{t+1} }\big\|\bar{\nabla}f\big(x_t, y_t^{P_t}, v_t^{Q_t}\big)\big\|^2 \nonumber \\
    & + \frac{1}{2\alpha_{t+1} }\big\|\nabla \Phi(x_t) - \bar{\nabla}f\big(x_t, y_t^{P_t}, v_t^{Q_t}\big)\big\|^2 + \frac{L_{\Phi}}{2\alpha_{t+1}^2 }\big\|\bar{\nabla}f\big(x_t, y_t^{P_t}, v_t^{Q_t}\big)\big\|^2, 
\end{align}
\endgroup
where the approximation error 
\begingroup
\allowdisplaybreaks
{
\small
\begin{align}\label{eq:doubledescent2}
    \big\|\nabla&\Phi(x_t) - \bar{\nabla}f\big(x_t, y_t^{P_t}, v_t^{Q_t}\big)\big\|^2 \nonumber \\
    =& \big\|\bar{\nabla}f\big(x_t, y^*(x_t), v^*(x_t)\big) - \bar{\nabla}f\big(x_t, y_t^{P_t}, v_t^{Q_t}\big)\big\|^2 \nonumber \\
    \leq & 2\big\|\bar{\nabla}f\big(x_t, y^*(x_t), v^*(x_t)\big) - \bar{\nabla}f\big(x_t, y_t^{P_t}, v^*(x_t)\big)\big\|^2 \nonumber \\
    &+ 2\big\|\bar{\nabla}f\big(x_t, y_t^{P_t}, v^*(x_t)\big) - \bar{\nabla}f\big(x_t, y_t^{P_t}, v_t^{Q_t}\big)\big\|^2 \nonumber \\
    \leq & 4\big\|\nabla_y\nabla_y g(x_t, y^*(x_t))v^*(x_t) - \nabla_y\nabla_y g\big(x_t, y_t^{P_t}\big)v^*(x_t)\big\|^2 \nonumber \\
    & + 4\big\|\nabla_y f\big(x_t, y^*(x_t)\big) - \nabla_y f(x_t, y_t^{P_t})\big\|^2 + 2\big\|\nabla_y\nabla_y g(x_t, y_t^{P_t}) \big(v^*(x_t) - v_t^{Q_t}\big)\big\|^2 \nonumber \\
    \overset{(a)}{\leq} & 4\Big(\frac{C_{f_y}^2L_{g,2}^2}{\mu^2} + L_{f,1}^2\Big)\|y_t^{P_t} - y^*(x_t)\|^2 + 2C_{g_{yy}}^2\|v_t^{Q_t} - v^*(x_t)\|^2 \nonumber \\
    \leq & {4\Big(\frac{C_{f_y}^2L_{g,2}^2}{\mu^2} + L_{f,1}^2\Big)\|y_t^{P_t} - y^*(x_t)\|^2 + 4C_{g_{yy}}^2\|v_t^{Q_t} - \hat{v}^*(x_t, y_t^{P_t})\|^2 + 4C_{g_{yy}}^2\|\hat{v}^*(x_t, y_t^{P_t}) - v^*(x_t)\|^2} \nonumber \\
    \overset{(b)}{\leq} & {4\bigg(\frac{C_{f_y}^2L_{g,2}^2}{\mu^2} + L_{f,1}^2 + C_{g_{yy}}^2\bar{L}_v^2\bigg)\|y_t^{P_t} - y^*(x_t)\|^2 + 4C_{g_{yy}}^2\|v_t^{Q_t} - \hat{v}^*(x_t, y_t^{P_t})\|^2} \nonumber \\
    \leq & \bar{L}^2 \big(\|y_t^{P_t} - y^*(x_t)\|^2 + \|v_t^{Q_t} - \hat{v}^*(x_t, y^*(x_t))\|^2\big), 
\end{align}}
\endgroup
where (a) uses Assumption \ref{as:lip}, \Cref{as:grad} and \Cref{lm:LS};
(b) uses $v^*(x_t) = \hat{v}^*\big(x_t, y^*(x_t)\big)$ and \Cref{lm:LS}. 
By using \Cref{lm:innererror}, we have
\begin{align}\label{eq:doubledescent3}
    \big\|\nabla\Phi&(x_t) - \bar{\nabla}f\big(x_t, y_t^{P_t}, v_t^{Q_t}\big)\big\|^2 \leq \frac{\bar{L}^2}{\mu^2}(\epsilon_y + \epsilon_v) =: \epsilon'.
\end{align}
By plugging \cref{eq:doubledescent3} into \cref{eq:doubledescent1}, we obtain \eqref{eq:objective_title_double1}. 

\noindent
Now if in addition, $k_1$ in \Cref{lm:bar} exists, then for $t\geq k_1$, we have $\alpha_{t+1} > C_\alpha\geq 2L_{\Phi}$. From \eqref{eq:objective_title_double1} we can immediately obtain \eqref{eq:objective_title_double2}.
Thus, the proof is complete. 
\end{proof}

\subsection{The bound of $\alpha_t$}
\begin{lemma}\label{lm:alpha_double}
Under Assumptions \ref{as:sc}, \ref{as:lip}, \ref{as:inf}, suppose the number of total iteration rounds in \Cref{alg:main_double} is $T$. 
If there exists $k_1 \leq T$ described in \Cref{lm:bar}, we have 
$$
\left\{
\begin{aligned}
\alpha_{t} \leq & C_\alpha, \quad & t \leq k_1; \\
\alpha_{t} \leq & C_\alpha + 2c_0 + \frac{2t\epsilon'}{\alpha_{0}}, \quad & t \geq k_1, 
\end{aligned}
\right.
$$
where we define
\begin{align}\label{eq:c0_double}
    c_0 := 2\big(\Phi(x_{0}) - \inf_x \Phi(x)\big) + \frac{L_{\Phi}C_\alpha^2}{2\alpha_{0}^2}. 
\end{align}

\noindent
When such $k_1$ does not exist, we have $\alpha_t \leq C_\alpha$ for any $t \leq T$. 
\end{lemma}
\begin{proof}
According to \Cref{lm:bar}, the proof can be split into the following three cases. 

\noindent
{\bf Case 1:} if $\alpha_T \leq C_\alpha$, for any $t < T$, we have the upper bound of $\alpha_{t+1}$ as $\alpha_{t+1} \leq C_\alpha$. 

\noindent
{\bf Case 2:} if $\alpha_T > C_\alpha$, there exists $k_1 \leq T$ described in \Cref{lm:bar}. 
Then we have the upper bound of $\alpha_{t+1}$ as $\alpha_{t+1} \leq C_\alpha$ for any $t < k_1$. 

\noindent
{\bf Case 3:} in the remaining proof, we only consider and explore the case $k_1 \leq t \leq T$ when $\alpha_T > C_\alpha$.

\noindent
From \Cref{lm:objective_double}, for $k \geq k_1$, we have
\begin{align}
    \Phi(x_{k+1}) 
    \leq& \Phi(x_k) - \frac{1}{2\alpha_{k+1}}\|\nabla \Phi(x_k)\|^2 - \frac{1}{4\alpha_{k+1}}\|\bar{\nabla}f(x_k, y_k^{P_k}, v_k^{Q_k})\|^2 + \frac{\epsilon'}{2\alpha_{k+1}}, \nonumber
\end{align}
which indicates that 
\begin{align}
    \frac{\|\bar{\nabla}f(x_k, y_k^{P_k}, v_k^{Q_k})\|^2}{\alpha_{k+1}} \leq 4\big(\Phi(x_k) - \Phi(x_{k+1})\big) + \frac{2\epsilon'}{\alpha_{k+1}}. \nonumber
\end{align}
By taking summation over $k=k_1,\ldots,t$, we have
\begin{align}\label{eq:double_alpha1}
    \sum_{k=k_1}^{t}\frac{\|\bar{\nabla}f(x_k, y_k^{P_k}, v_k^{Q_k})\|^2}{\alpha_{k+1}} \leq& 4\sum_{k={k_1}}^{t}\big(\Phi(x_k) - \Phi(x_{k+1})\big) + \sum_{k={k_1}}^{t}\frac{2\epsilon'}{\alpha_{k+1}} \nonumber \\
    = & 4\big(\Phi(x_{k_1}) - \Phi(x_{t+1})\big) + \sum_{k={k_1}}^{t}\frac{2\epsilon'}{\alpha_{k+1}}.
\end{align}
For $\Phi(x_{k_1})$, by telescoping \eqref{eq:objective_title_double1}, we get
\begin{align}\label{eq:double_alpha2}
    \Phi(x_{k_1}) \leq& \Phi(x_0) + \sum_{k=0}^{k_1-1}\frac{L_{\Phi}}{4\alpha_{k+1}^2}\|\bar{\nabla}f(x_k, y_k^{P_k}, v_k^{Q_k})\|^2 + \sum_{k=0}^{k_1-1}\frac{\epsilon'}{2\alpha_{k+1}}.
\end{align}
Plugging \cref{eq:double_alpha2} into \cref{eq:double_alpha1}, we obtain
\begingroup
\allowdisplaybreaks
{\small
\begin{align}\label{eq:double_alpha3}
    \sum_{k=k_1}^{t}\frac{\|\bar{\nabla}f(x_k, y_k^{P_k}, v_k^{Q_k})\|^2}{\alpha_{k+1}} 
    \leq & 4\big(\Phi(x_{0}) - \inf_x \Phi(x)\big) + \sum_{k=0}^{k_1-1}\frac{L_{\Phi}}{\alpha_{k+1}^2}\|\bar{\nabla}f(x_k, y_k^{P_k}, v_k^{Q_k})\|^2 + \sum_{k=0}^{t}\frac{2\epsilon'}{\alpha_{k+1}} \nonumber \\
    \leq & 
    4\big(\Phi(x_{0}) - \inf_x \Phi(x)\big) + \frac{L_{\Phi}\sum_{k=0}^{k_1-1}\|\bar{\nabla}f(x_k, y_k^{P_k}, v_k^{Q_k})\|^2}{\alpha_{0}^2} + \sum_{k=0}^{t}\frac{2\epsilon'}{\alpha_{k+1}} \nonumber \\
    \leq & 4\big(\Phi(x_{0}) - \inf_x \Phi(x)\big) + \frac{L_{\Phi}\alpha_{k_1}^2}{\alpha_{0}^2} + \frac{2(t+1)\epsilon'}{\alpha_{0}} \nonumber \\
    \leq & 4\big(\Phi(x_{0}) - \inf_x \Phi(x)\big) + \frac{L_{\Phi}C_\alpha^2}{\alpha_{0}^2} + \frac{2(t+1)\epsilon'}{\alpha_{0}}. 
\end{align}}
\endgroup

\noindent
Inspired by \cite{ward2020adagrad} and using telescoping, we have 
\begingroup
\allowdisplaybreaks
\begin{align}
    \alpha_{t+1} =& \alpha_t + \frac{\|\bar{\nabla}f(x_t, y_t^{P_t}, v_t^{Q_t})\|^2}{\alpha_{t+1} + \alpha_t} \nonumber \\
    \leq& \alpha_t + \frac{\|\bar{\nabla}f(x_t, y_t^{P_t}, v_t^{Q_t})\|^2}{\alpha_{t+1} } \nonumber \\
    \leq& \alpha_{k_1} + \sum_{k=k_1}^{t}\frac{\|\bar{\nabla}f(x_k, y_k^{P_k}, v_k^{Q_k})\|^2}{\alpha_{k+1}} \nonumber \\
    \leq & C_\alpha + 4\big(\Phi(x_{0}) - \inf_x \Phi(x)\big) + \frac{L_{\Phi}C_\alpha^2}{\alpha_{0}^2} + \frac{2(t+1)\epsilon'}{\alpha_{0}}. \nonumber
\end{align}
\endgroup
Thus, the proof is complete. 
\end{proof}

\subsection{Convergence Analysis of Sub-loops}
\begin{lemma}\label{lm:subloop}
    Recall that for the $t_{th}$ iteration, the sub-loops in \Cref{alg:main_double} aim to find $y_t^{P_t}$ and $v_t^{Q_t}$ such that $\|\nabla_y g(x_t, y_t^{P_t})\|^2 \leq \epsilon_y$ and $\|\nabla_v R(x_t, y_t^{P_t}, v_t^{Q_t})\|^2 \leq \epsilon_v$. 
Here we prove that 
\begin{subequations}
\begin{align}
    &P_t \leq P' := \frac{\log(C_\beta^2/\beta_0^2)}{\log(1+\epsilon_y/C_\beta^2)} + \frac{\beta_{\text{max}}}{\mu}\log \big(\frac{L_{g,1}^2(\beta_{\text{max}}-C_\beta)}{\epsilon_y}\big), \label{lm:subloop-bound-P} \\
    &Q_t \leq Q' := \frac{\log(C_\gamma^2/\gamma_0^2)}{\log(1+\epsilon_v/C_\gamma^2)} + \frac{\gamma_{\text{max}}}{\mu}\log\big(\frac{C_{g_{yy}}^2(\gamma_{\text{max}}-C_\gamma)}{\epsilon_v}\big),\label{lm:subloop-bound-Q}
\end{align}
\end{subequations}
where $\beta_{\text{max}} := C_\beta + L_{g,1}\big(\frac{2\epsilon_y}{\mu^2} + \frac{2C_{g_{xy}}^2C_f^2}{\mu^2\alpha_0^2} + 2\log(C_\beta/\beta_0) + 1\big)$ and $\gamma_{\text{max}} := C_\gamma + C_{g_{yy}}\big(\frac{2\epsilon_y}{\mu^2} + \frac{8C_{f_y}^2}{\mu^2} + 2\log(C_\gamma/\gamma_0) + 1\big)$. 
\end{lemma}

\begin{proof}
The proof is split into the following two parts. 

\noindent
{\bf Part I: maximum number for convergence of $g(x_t, y_t^{P_t})$.} 

\noindent
Inspired by \cite{xie2020linear}, we split the analysis into the following two cases. 

\noindent
{\bf Case 1: $k_2$ does not exist before we find $P_t$.} This indicates $\beta_{P_t} < C_\beta$. 
Referring to Lemma 2 in \cite{xie2020linear}, we have $P_t < \frac{\log(C_\beta^2/\beta_0^2)}{\log(1+\epsilon_y/C_\beta^2)}$ and therefore the desired upper bound for $P_t$ holds. This can be proved as follows. If $P_t \geq \frac{\log(C_\beta^2/\beta_0^2)}{\log(1+\epsilon_y/C_\beta^2)}$, we have the following result. 
\begingroup
\allowdisplaybreaks
\begin{align}\label{eq:k_2bound1}
    \beta_{P_t}^2 =& \beta_{P_t-1}^2 + \|\nabla_y g(x_t,y_t^{P_t-1})\|^2 \nonumber \\
    =& \beta_{P_t-1}^2\Big(1+\frac{\|\nabla_y g(x_t,y_t^{P_t-1})\|^2}{\beta_{P_t-1}^2}\Big) \nonumber \\
    \geq& \beta_0^2 \prod_{p=0}^{P_t-1}\Big(1+\frac{\|\nabla_y g(x_t,y_t^{p})\|^2}{\beta_{p}^2}\Big) \nonumber \\
    \geq& \beta_0^2 \Big(1+\frac{\epsilon_y}{C_\beta^2}\Big)^{P_t} \geq C_\beta^2. 
\end{align}
\endgroup
This contradicts $\beta_{P_t} < C_\beta$. 

\noindent
{\bf Case 2: $k_2$ exists and $P_t \geq k_2$. } Here we have $\beta_{k_2} \leq C_\beta$ and $\beta_{k_2+1} > C_\beta$. 

\noindent
{\bf Firstly}, we prove $k_2 \leq \frac{\log(C_\beta^2/\beta_0^2)}{\log(1+\epsilon_y/C_\beta^2)}$. 
Similar to {\bf Case 1}, if $k_2 > \frac{\log(C_\beta^2/\beta_0^2)}{\log(1+\epsilon_y/C_\beta^2)}$, following \cref{eq:k_2bound1}{by replacing $P_t$ with $k_2$}, we have 
\begin{align}
    \beta_{k_2}^2 \geq \beta_0^2 \Big(1+\frac{\epsilon_y}{C_\beta^2}\Big)^{k_2} > C_\beta^2, \nonumber
\end{align}
which contradicts $\beta_{k_2} \leq C_\beta$.

\noindent
{\bf Secondly}, referring to Lemma 3 in \cite{xie2020linear}, we have the bound of $\|y_t^{k_2} - y^*(x_t)\|^2$ as 
\begingroup
\allowdisplaybreaks
\begin{align}\label{eq:y_k2_doube1}
    \|y_t^{k_2}& - y^*(x_t)\|^2 \nonumber \\
    =& \bigg\|y_t^{k_2-1} - \frac{\nabla_y g(x_t,y_t^{k_2-1})}{\beta_{k_2}} - y^*(x_t)\bigg\|^2 \nonumber \\
    =& \|y_t^{k_2-1} - y^*(x_t)\|^2 + \bigg\|\frac{\nabla_y g(x_t,y_t^{k_2})}{\beta_{k_2}}\bigg\|^2 - 2\bigg\langle y_t^{k_2-1} - y^*(x_t), \frac{\nabla_y g(x_t,y_t^{k_2-1})}{\beta_{k_2}}\bigg\rangle \nonumber \\
    \overset{(a)}{\leq}& \|y_t^{k_2-1} - y^*(x_t)\|^2 + \bigg\|\frac{\nabla_y g(x_t,y_t^{k_2-1})}{\beta_{k_2}}\bigg\|^2 - \frac{2}{\beta_{k_2}L_{g,1}}\big\| \nabla_y g(x_t,y_t^{k_2-1}) - \nabla_y g\big(x_t,y^*(x_t)\big)\big\|^2 \nonumber \\
    \leq & \|y_t^{k_2-1} - y^*(x_t)\|^2 + \frac{\|\nabla_y g(x_t,y_t^{k_2-1})\|^2}{\beta_{k_2}^2} \nonumber \\
    \leq & \|y_t^{0} - y^*(x_t)\|^2 + \sum_{p=0}^{k_2-1}\frac{\|\nabla_y g(x_t,y_t^{p})\|^2}{\beta_{p+1}^2} \nonumber \\
    \overset{(b)}{\leq}& \|y_{t-1}^{P_{t-1}} - y^*(x_t)\|^2 + \sum_{p=0}^{k_2-1}\frac{\|\nabla_y g(x_t,y_t^{p})\|^2/\beta_0^2}{\sum_{k=0}^{p}\|\nabla_y g(x_t,y_t^{k})\|^2/\beta_0^2} \nonumber \\
    \overset{(c)}{\leq}& 2\|y_{t-1}^{P_{t-1}} - y^*(x_{t-1})\|^2 + 2\|y^*(x_{t-1}) - y^*(x_t)\|^2 + \log\bigg(\sum_{p=0}^{k_2-1}\frac{\|\nabla_y g(x_t,y_t^{p})\|^2}{\beta_0^2}\bigg)+1 \nonumber \\
    \overset{(d)}{\leq}& \frac{2\epsilon_y}{\mu^2} + \frac{2C_{g_{xy}}^2\|\Bar{\nabla}f(x_{t-1}, y_{t-1}^{P_{t-1}}, v_{t-1}^{Q_{t-1}})\|^2}{\mu^2\alpha_t^2} + \log\bigg(\sum_{p=0}^{k_2-1}\frac{\|\nabla_y g(x_t,y_t^{p})\|^2}{\beta_0^2}\bigg)+1 \nonumber \\
    \overset{(e)}{\leq}& \frac{2\epsilon_y}{\mu^2} + \frac{2C_{g_{xy}}^2C_f^2}{\mu^2\alpha_0^2} + 2\log(C_\beta/\beta_0) + 1,
\end{align}
\endgroup
where (a) uses Assumptions \ref{as:sc},\ref{as:lip}; (b) refers to the warm start of $y_t^0$;
(c) uses \Cref{lm:log}; (d) uses Lemmas \ref{lm:basic} and \ref{lm:innererror}; (e) follows from \Cref{lm:boundoff_double} and $\beta_{k_2} \leq C_\beta$. 

\noindent
{\bf Last}, following \cite{xie2020linear}, {for all $P > k_2$}, we have the bound of $\|y_t^{P} - y^*(x_t)\|^2$ as 
\begingroup
\allowdisplaybreaks
{
% \small
\begin{align}\label{eq:y_k2_doube2}
    \|y_t^{P} - y^*(x_t)\|^2 =& \|y_t^{P-1} - y^*(x_t)\|^2 + \frac{\|\nabla_y g(x_t, y_t^{P-1})\|^2}{\beta_{P}^2} - \frac{2\big\langle y_t^{P-1} - y^*(x_t), \nabla_y g(x_t, y_t^{P-1})\big\rangle}{\beta_{P}} \nonumber \\
    \leq& \|y_t^{P-1} - y^*(x_t)\|^2 - \frac{1}{\beta_{P}}\Big(2-\frac{L_{g,1}}{\beta_{P}}\Big)\big\langle y_t^{P-1} - y^*(x_t), \nabla_y g(x_t, y_t^{P-1})\big\rangle \nonumber \\
    \overset{(a)}{\leq}& \|y_t^{P-1} - y^*(x_t)\|^2 - \frac{1}{\beta_{P}}\big\langle y_t^{P-1} - y^*(x_t), \nabla_y g(x_t, y_t^{P-1})\big\rangle \nonumber \\
    \overset{(b)}{\leq}& \Big(1-\frac{\mu}{\beta_{P}}\Big)\|y_t^{P-1} - y^*(x_t)\|^2 \nonumber \\
    \overset{(c)}{\leq}& e^{-{\mu(P-k_2)}/{\beta_{P}}}\|y_t^{k_2} - y^*(x_t)\|^2 \nonumber \\
    \overset{(d)}{\leq}& e^{-{\mu(P-k_2)}/{\beta_{P}}}\bigg(\frac{2\epsilon_y}{\mu^2} + \frac{2C_{g_{xy}}^2C_f^2}{\mu^2\alpha_0^2} + 2\log(C_\beta/\beta_0) + 1\bigg),
\end{align}
}
\endgroup
where (a) uses $\beta_{P} \geq C_\beta \geq L_{g,1}$; (b) uses Assumption \ref{as:sc}; (c) follows from $\beta_{P} \geq C_\beta \geq L_{g,1} \geq \mu$ and $1-m \leq e^{-m}$ for $0<m<1$; (d) refers to \cref{eq:y_k2_doube1}. Inspired by Lemma 4 in \cite{xie2020linear}, we have the upper-bound of $\beta_{P}$ as 
\begin{align}\label{eq:y_k2_doube3}
    \beta_{P} = \beta_{P-1} + \frac{\|\nabla_y g(x_t, y_t^{P-1})\|^2}{\beta_{P}+\beta_{P-1}} \leq \beta_{k_2} + \sum_{p=k_2}^{P-1}\frac{\|\nabla_y g(x_t, y_t^{p})\|^2}{\beta_{p+1}}.
\end{align}
To further bound the last term of the right-hand side of \cref{eq:y_k2_doube3}, using Assumption \ref{as:lip}, we have the following result: 
\begin{align}\label{eq:y_k2_doube4}
    \|y_t^{P}& - y^*(x_t)\|^2 \nonumber \\
    =& \|y_t^{P-1} - y^*(x_t)\|^2 + \frac{\|\nabla_y g(x_t, y_t^{P-1})\|^2}{\beta_{P}^2} - \frac{2\big\langle y_t^{P-1} - y^*(x_t), \nabla_y g(x_t, y_t^{P-1})\big\rangle}{\beta_{P}} \nonumber \\
    \overset{(a)}{\leq} & \|y_t^{P-1} - y^*(x_t)\|^2 + \frac{\|\nabla_y g(x_t, y_t^{P-1})\|^2}{\beta_{P}^2} - \frac{2\|\nabla_y g(x_t, y_t^{P-1}) - \nabla_y g(x_t, y^*(x_t))\|^2}{\beta_{P}L_{g,1}} \nonumber \\
    \overset{(b)}{\leq} & \|y_t^{P-1} - y^*(x_t)\|^2 - \frac{\|\nabla_y g(x_t, y_t^{P-1})\|^2}{\beta_{P}L_{g,1}} \nonumber \\
    {\leq} & \|y_t^{k_2} - y^*(x_t)\|^2 - \sum_{p=k_2}^{P-1} \frac{\|\nabla_y g(x_t, y_t^{p})\|^2}{\beta_{p+1}L_{g,1}}, 
\end{align}
where (a) uses Assumptions \ref{as:sc},\ref{as:lip}; (b) refers to $\beta_{P} \geq C_\beta \geq L_{g,1}$. By rearranging \cref{eq:y_k2_doube4} and using \cref{eq:y_k2_doube1}, we have
\begingroup
\allowdisplaybreaks
\begin{align}\label{eq:y_k2_doube5}
    \sum_{p=k_2}^{P-1}\frac{\|\nabla_y g(x_t, y_t^{p})\|^2}{\beta_{p+1}} \leq& L_{g,1}\big(\|y_t^{k_2} - y^*(x_t)\|^2 - \|y_t^{P} - y^*(x_t)\|^2\big) \nonumber \\
    \leq & L_{g,1}\|y_t^{k_2} - y^*(x_t)\|^2 \nonumber \\
    \leq & L_{g,1}\bigg(\frac{2\epsilon_y}{\mu^2} + \frac{2C_{g_{xy}}^2C_f^2}{\mu^2\alpha_0^2} + 2\log(C_\beta/\beta_0) + 1\bigg). 
\end{align}
\endgroup
Plugging \cref{eq:y_k2_doube5} into \cref{eq:y_k2_doube3}, we obtain the upper-bound of $\beta_{P}$ as 
\begin{align}\label{eq:y_k2_doube7}
    \beta_{P} \leq C_\beta + L_{g,1}\bigg(\frac{2\epsilon_y}{\mu^2} + \frac{2C_{g_{xy}}^2C_f^2}{\mu^2\alpha_0^2} + 2\log(C_\beta/\beta_0) + 1\bigg) =: \beta_{\text{max}}. 
\end{align}
Then, by plugging \cref{eq:y_k2_doube7} into \cref{eq:y_k2_doube2}, 
we have the upper bound of $\|y_t^{P} - y^*(x_t)\|^2$ as 
\begin{align}\label{eq:y_k2_doube6}
    \|y_t^{P} - y^*(x_t)\|^2 \leq e^{-{\mu(P-k_2)}/{\beta_{\text{max}}}}\bigg(\frac{2\epsilon_y}{\mu^2} + \frac{2C_{g_{xy}}^2C_f^2}{\mu^2\alpha_0^2} + 2\log(C_\beta/\beta_0) + 1\bigg).
\end{align}
Recall we have the upper bound $k_2 \leq \frac{\log(C_\beta^2/\beta_0^2)}{\log(1+\epsilon_y/C_\beta^2)}$. 
Note that $P'$ defined in \eqref{lm:subloop-bound-P} satisfies 
\begin{align*}
    P' \geq k_2 + \frac{\beta_{\text{max}}}{\mu}\log\big({L_{g,1}^2(\beta_{\text{max}}-C_\beta)}/{\epsilon_y}\big). 
\end{align*}
By replacing $P$ with $P'$ in \cref{eq:y_k2_doube6}, 
we have 
\begin{align*}
    \|\nabla_y g(x_t,y_t^{P'})\|^2 \leq L_{g,1}^2\|y_t^{P'} - y^*(x_t)\|^2 \leq e^{-{\mu(P'-k_2)}/{\beta_{\text{max}}}}L_{g,1}^2(\beta_{\text{max}} - C_\beta) \leq \epsilon_y. 
\end{align*}
Therefore, $P_t\leq P'$ and this completes the proof of \eqref{lm:subloop-bound-P}.

\noindent
{\bf Part II: maximum number for convergence of $R(x_t, y_t^{P_t}, v_t^{Q_t})$.} 

\noindent
Similarly to {\bf Part I}, we split the analysis into the following two cases. 

\noindent
{\bf Case 1: $k_3$ does not exist before we find $Q_t$.} This indicates $\gamma_{Q_t} < C_\gamma$. 
Then we have $Q_t < \frac{\log(C_\gamma^2/\gamma_0^2)}{\log(1+\epsilon_v/C_\gamma^2)}$. 
Otherwise, if $Q_t \geq \frac{\log(C_\gamma^2/\gamma_0^2)}{\log(1+\epsilon_v/C_\gamma^2)}$, we have the following result. 
\begingroup
\allowdisplaybreaks
\begin{align}
    \gamma_{Q_t}^2 =& \gamma_{Q_t-1}^2 + \|\nabla_v R(x_t,y_t^{P_t},v_t^{Q_t-1})\|^2 \nonumber \\
    =& \gamma_{Q_t-1}^2\Big(1+\frac{\|\nabla_v R(x_t,y_t^{P_t},v_t^{Q_t-1})\|^2}{\gamma_{Q_t-1}^2}\Big) \nonumber \\
    \geq& \gamma_0^2 \prod_{q=0}^{Q_t-1}\Big(1+\frac{\|\nabla_v R(x_t,y_t^{P_t},v_t^{Q_t-1})\|^2}{\gamma_{Q_t-1}^2}\Big) \nonumber \\
    \geq& \gamma_0^2 \Big(1+\frac{\epsilon_v}{C_\gamma^2}\Big)^{Q_t} \geq C_\gamma^2. \nonumber 
\end{align}
\endgroup
 This contradicts $\gamma_{Q_t} < C_\gamma$. 

\noindent
{\bf Case 2: $k_3$ exists and $Q_t \geq k_3$. } Here we have $\gamma_{k_3} \leq C_\gamma$ and $\gamma_{k_3+1} > C_\gamma$. 

\noindent
{\bf Firstly}, we have $k_3 \leq \frac{\log(C_\gamma^2/\gamma_0^2)}{\log(1+\epsilon_v/C_\gamma^2)}$. 
Similar to {\bf Case 1}, if $k_3 > \frac{\log(C_\gamma^2/\gamma_0^2)}{\log(1+\epsilon_v/C_\gamma^2)}$, following \cref{eq:k_2bound1}, by replacing $Q_t$ with $k_3$, we have 
\begin{align}
    \gamma_{k_3}^2 \geq \gamma_0^2 \Big(1+\frac{\epsilon_v}{C_\gamma^2}\Big)^{k_3} > C_\gamma^2, \nonumber
\end{align}
which contradicts $\gamma_{k_3} \leq C_\gamma$.

\noindent
{\bf Secondly}, referring to Lemma 3 in \cite{xie2020linear}, we have the bound of $\|v_t^{k_3} - v^*(x_t)\|^2$ as following:
\begingroup
\allowdisplaybreaks
{
{
\begin{align}\label{eq:v_k3_new_doube1}
    \big\|v_t^{k_3} - \hat{v}^*(x_t, y_t^{P_t})\big\|^2 
    =& \bigg\|v_t^{k_3-1} - \frac{\nabla_v R(x_t,y_t^{P_t},v_t^{k_3-1})}{\gamma_{k_3}} - \hat{v}^*(x_t, y_t^{P_t})\bigg\|^2 \nonumber \\
    =& \big\|v_t^{k_3-1} - \hat{v}^*(x_t, y_t^{P_t})\big\|^2 + \bigg\|\frac{\nabla_v R(x_t,y_t^{P_t},v_t^{k_3-1})}{\gamma_{k_3}}\bigg\|^2 \nonumber \\
    & - \frac{2}{\gamma_{k_3}}\big\langle v_t^{k_3-1} - \hat{v}^*(x_t, y_t^{P_t}), \nabla_v R(x_t,y_t^{P_t},v_t^{k_3-1})\big\rangle \nonumber \\
    \overset{(a)}{\leq} & \big\|v_t^{k_3-1} - \hat{v}^*(x_t, y_t^{P_t})\big\|^2 + \bigg\|\frac{\nabla_v R(x_t,y_t^{P_t},v_t^{k_3-1})}{\gamma_{k_3}}\bigg\|^2 \nonumber \\
    & - \frac{2}{\gamma_{k_3}C_{g_{yy}}}\big\| \nabla_v R(x_t,y_t^{P_t},v_t^{k_3-1}) - \nabla_v R\big(x_t,y_t^{P_t},\hat{v}^*(x_t, y_t^{P_t})\big)\big\|^2 \nonumber \\
    \leq & \big\|v_t^{k_3-1} - \hat{v}^*(x_t, y_t^{P_t})\big\|^2 + \bigg\|\frac{\nabla_v R(x_t,y_t^{P_t},v_t^{k_3-1})}{\gamma_{k_3}}\bigg\|^2 \nonumber \\
    \leq & \big\|v_t^{0} - \hat{v}^*(x_t, y_t^{P_t})\big\|^2 + \sum_{q=0}^{k_3-1}\bigg\|\frac{\nabla_v R(x_t,y_t^{P_t},v_t^{q})}{\gamma_{k_3}}\bigg\|^2 \nonumber \\
    \overset{(b)}{\leq}& \big\|v_t^{0} - \hat{v}^*(x_t, y_t^{P_t})\big\|^2 + \sum_{q=0}^{k_3-1}\frac{\|\nabla_v R(x_t,y_t^{P_t},v_t^{q})\|^2/\gamma_0^2}{\sum_{k=0}^{q}\|\nabla_v R(x_t,y_t^{P_t},v_t^{k})\|^2/\gamma_0^2} \nonumber \\
    \overset{(c)}{\leq}& 2\big\|v_{t-1}^{P_{t-1}} - \hat{v}^*(x_{t-1}, y_{t-1}^{P_{t-1}})\big\|^2 + 2\big\|\hat{v}^*(x_{t-1}, y_{t-1}^{P_{t-1}}) - \hat{v}^*(x_t, y_{t}^{P_{t}})\big\|^2 \nonumber \\
    & + \log\bigg(\sum_{q=0}^{k_3-1}\|\nabla_v R(x_t,y_t^{P_t},v_t^{k})\|^2/\gamma_0^2\bigg)+1 \nonumber \\
    \leq & 2\big\|v_{t-1}^{P_{t-1}} - \hat{v}^*(x_{t-1}, y_{t-1}^{P_{t-1}})\big\|^2 + 4\big\|\hat{v}^*(x_{t-1}, y_{t-1}^{P_{t-1}})\big\|^2 + 4\big\|\hat{v}^*(x_t, y_{t}^{P_{t}})\big\|^2 \nonumber \\
    & + \log\bigg(\sum_{q=0}^{k_3-1}\|\nabla_v R(x_t,y_t^{P_t},v_t^{k})\|^2/\gamma_0^2\bigg)+1 \nonumber \\
    \overset{(d)}{\leq}& \frac{2\epsilon_y}{\mu^2} + \frac{8C_{f_y}^2}{\mu^2} + 2\log(C_\gamma/\gamma_0) + 1, 
\end{align}}}
\endgroup
where (a) uses \Cref{lm:LS} and $\nabla_v R\big(x_t,y_t^{P_t},\hat{v}^*(x_t, y_t^{P_t})\big) = 0$; 
(b) refers to the warm start of $v_t^0$;
(c) uses \Cref{lm:log}; 
(d) follows from \Cref{lm:LS},\ref{lm:innererror} and $\gamma_{k_3} \leq C_\gamma$.

\noindent
{\bf Last}, similar to {\bf Part I}, for all $Q > k_3$, we explore the bound of $\|v_t^{Q} - v^*(x_t)\|^2$ as  
\begingroup
\allowdisplaybreaks
{
\begin{align}\label{eq:v_k3_douben}
    \big\|v_t^{Q}& - \hat{v}^*(x_t, y_t^{P_t})\big\|^2 \nonumber \\
    =& \big\|v_t^{Q-1} - \hat{v}^*(x_t, y_t^{P_t})\big\|^2 + \frac{\|\nabla_v R(x_t,y_t^{P_t},v_t^{Q-1})\|^2}{\gamma_{Q}^2} \nonumber \\
    & - \frac{2\big\langle v_t^{Q-1} - \hat{v}^*(x_t, y_t^{P_t}), \nabla_v R(x_t,y_t^{P_t},v_t^{Q-1})\big\rangle}{\gamma_{Q}} \nonumber \\
    \overset{(a)}{\leq} & \big\|v_t^{Q-1} - \hat{v}^*(x_t, y_t^{P_t})\big\|^2 -\frac{1}{\gamma_{Q}}\Big(2-\frac{C_{g_{yy}}}{\gamma_{Q}}\Big)\big\langle v_t^{Q-1} - \hat{v}^*(x_t, y_t^{P_t}), \nabla_v R(x_t,y_t^{P_t},v_t^{Q-1})\big\rangle \nonumber \\
    \overset{(b)}{\leq} & \big\|v_t^{Q-1} - \hat{v}^*(x_t, y_t^{P_t})\big\|^2 - \frac{1}{\gamma_{Q}}\big\langle v_t^{Q-1} - \hat{v}^*(x_t, y_t^{P_t}), \nabla_v R(x_t,y_t^{P_t},v_t^{Q-1})\big\rangle \nonumber \\
    \overset{(c)}{\leq} & \Big(1-\frac{\mu}{\gamma_{Q}}\Big)\big\|v_t^{Q-1} - \hat{v}^*(x_t, y_t^{P_t})\big\|^2 \nonumber \\
    \overset{(d)}{\leq}& e^{-{\mu(Q-k_3)}/{\gamma_{Q}}}\big\|v_t^{k_3} - \hat{v}^*(x_t, y_t^{P_t})\big\|^2  \nonumber \\
    \overset{(e)}{\leq}& e^{-{\mu(Q-k_3)}/{\gamma_{Q}}}\Big(\frac{2\epsilon_y}{\mu^2} + \frac{8C_{f_y}^2}{\mu^2} + 2\log(C_\gamma/\gamma_0) + 1\Big), 
\end{align}
where (a) uses \Cref{lm:LS}; 
(b) follows from $\gamma_{Q} > C_\gamma \geq C_{g_{yy}}$; 
(c) uses $\nabla_v R\big(x_t,y_t^{P_t},\hat{v}^*(x_t, y_t^{P_t})\big)=0$ and \Cref{lm:LS}; 
(d) follows from $\gamma_{Q} \geq C_\gamma \geq C_{g_{yy}} \geq \mu$ and $1-m \leq e^{-m}$ for $0<m<1$;
(e) uses \cref{eq:v_k3_new_doube1}. 
}
\endgroup
Similar to \cref{eq:y_k2_doube3}, we have the upper-bound of $\gamma_{Q}$ as 
\begin{align}\label{eq:v_k3_doube4}
    \gamma_{Q} = \gamma_{Q-1} + \frac{\|\nabla_v R(x_t, y_t^{P_t}, v_t^{Q-1})\|^2}{\gamma_{Q}+\gamma_{Q-1}} \leq \gamma_{k_3} + \sum_{q=k_3}^{Q-1}\frac{\|\nabla_v R(x_t, y_t^{P_t}, v_t^{q})\|^2}{\gamma_{q+1}}.
\end{align}
To further bound the last term on the right-hand side of \cref{eq:v_k3_doube4}, we can have the following result: 
\begingroup
\allowdisplaybreaks
{
\begin{align}\label{eq:v_k3_doube5}
    \big\|v_t^{Q} - \hat{v}^*(x_t, y_t^{P_t})\big\|^2 
    =& \big\|v_t^{Q-1} - \hat{v}^*(x_t, y_t^{P_t})\big\|^2 + \frac{\|\nabla_v R(x_t, y_t^{P_t}, v_t^{Q-1})\|^2}{\gamma_{Q}^2} \nonumber \\
    & - \frac{2\big\langle v_t^{Q} - \hat{v}^*(x_t, y_t^{P_t}), \nabla_v R(x_t, y_t^{P_t}, v_t^{Q-1})\big\rangle}{\gamma_{Q}} \nonumber \\
    \overset{(a)}{\leq} & \big\|v_t^{Q-1} - \hat{v}^*(x_t, y_t^{P_t})\big\|^2 + \frac{\|\|\nabla_v R(x_t, y_t^{P_t}, v_t^{Q-1})\|^2}{\gamma_{Q}^2} \nonumber \\
    & - \frac{2\big\|\nabla_v R(x_t, y_t^{P_t}, v_t^{Q-1}) - \nabla_v R\big(x_t, y_t^{P_t}, \hat{v}^*(x_t, y_t^{P_t})\big)\big\|^2}{\gamma_{Q}C_{g_{yy}}} \nonumber \\
    \overset{(b)}{\leq} & \big\|v_t^{Q-1} - \hat{v}^*(x_t, y_t^{P_t})\big\|^2 - \frac{\big\|\nabla_v R(x_t, y_t^{P_t}, v_t^{Q-1})\big\|^2}{\gamma_{Q}C_{g_{yy}}} \nonumber \\
    \leq & \big\|v_t^{k_3} - \hat{v}^*(x_t, y_t^{P_t})\big\|^2 - \sum_{q=k_3}^{Q-1}\frac{\big\|\nabla_v R(x_t, y_t^{P_t}, v_t^{q})\big\|^2}{\gamma_{q+1}C_{g_{yy}}},
\end{align}}
\endgroup
where (a) uses \Cref{lm:LS}; (b) refers to $\gamma_{Q} \geq C_\gamma \geq C_{g_{yy}}$. By rearranging \cref{eq:v_k3_doube5} and using \cref{eq:v_k3_new_doube1}, we have
\begingroup
\allowdisplaybreaks
\begin{align}\label{eq:v_k3_doube6}
    \sum_{q=k_3}^{Q-1}\frac{\big\|\nabla_v R(x_t, y_t^{P_t}, v_t^{q})\big\|^2}{\gamma_{q+1}} 
    \leq & C_{g_{yy}}\big(\big\|v_t^{k_3} - \hat{v}^*(x_t, y_t^{P_t})\big\|^2 - \big\|v_t^{Q} - \hat{v}^*(x_t, y_t^{P_t})\big\|^2\big) \nonumber \\
    \leq & C_{g_{yy}}\bigg(\frac{2\epsilon_y}{\mu^2} + \frac{8C_{f_y}^2}{\mu^2} + 2\log(C_\gamma/\gamma_0) + 1\bigg). 
\end{align}
Plugging \cref{eq:y_k2_doube6} into \cref{eq:y_k2_doube4}, we obtain the upper-bound of $\gamma_{Q}$ as 
\begin{align}\label{eq:v_k3_doube7}
    \gamma_{Q} \leq C_\gamma + C_{g_{yy}}\bigg(\frac{2\epsilon_y}{\mu^2} + \frac{8C_{f_y}^2}{\mu^2} + 2\log(C_\gamma/\gamma_0) + 1\bigg) =: \gamma_{\text{max}}. 
\end{align}
Then, we have the upper bound of $\|v_t^{Q} - \hat{v}^*(x_t, y_t^{P_t})\|^2$ as 
\begin{align}\label{eq:v_k3_doube8}
    \big\|v_t^{Q} - \hat{v}^*(x_t, y_t^{P_t})\big\|^2 \leq e^{-{\mu(Q_t-k_3)}/{\gamma_{\text{max}}}}\bigg(\frac{2\epsilon_y}{\mu^2} + \frac{8C_{f_y}^2}{\mu^2} + 2\log(C_\gamma/\gamma_0) + 1\bigg).
\end{align}
Recall we have the upper bound $k_3 \leq \frac{\log(C_\gamma^2/\gamma_0^2)}{\log(1+\epsilon_v/C_\gamma^2)}$. 
Note that $Q'$ defined in \eqref{lm:subloop-bound-Q} satisfies
\begin{align}
    Q' \geq k_3 + \frac{\gamma_{\text{max}}}{\mu}\log\big({C_{g_{yy}}^2(\gamma_{\text{max}}-C_\gamma)}/{\epsilon_v}\big). \nonumber
\end{align}
{
By replacing $Q$ with $Q'$ in \cref{eq:v_k3_doube8}, we have 
\begin{align}
    \|\nabla_v R(x_t,y_t^{P_t},v_t^{Q'})\|^2 \leq C_{g_{yy}}^2\big\|v_t^{Q'} - \hat{v}^*(x_t, y_t^{P_t})\big\|^2 \leq e^{-{\mu(Q'-k_3)}/{\gamma_{\text{max}}}}(\gamma_{\text{max}} - C_\gamma) \leq \epsilon_v. \nonumber
\end{align}
Therefore, $Q_t \leq Q'$ and this completes the proof of \eqref{lm:subloop-bound-Q}.}
\endgroup
Thus, the proof is complete. 
\end{proof}

\subsection{Proof of \Cref{thm:main_double}}
Here we suppose the total iteration round is $T$. 
According to \Cref{lm:bar}, the proof can be split into the following two cases. 

\noindent
{\bf Case 1: $k_1$ does not exist. } Based on \Cref{lm:bar}, we have $\alpha_T \leq C_\alpha$. 
Then by \Cref{lm:objective_double} we have 
\begin{align}
    \frac{\|\nabla \Phi(x_t)\|^2}{\alpha_{t+1}} \leq 2\big(\Phi(x_t) - \Phi(x_{t+1})\big) + \frac{L_{\Phi}}{2\alpha_{t+1}^2}\|\bar{\nabla}f(x_t, y_t^{P_t}, v_t^{Q_t})\|^2 + \frac{\epsilon'}{\alpha_{t+1}}, \nonumber
\end{align}
where $\epsilon'$ is defined in \Cref{lm:objective_double}. 
By taking the average, we have
\begingroup
\allowdisplaybreaks
\begin{align}\label{eq:converge_double1}
    \frac{1}{T}\sum_{t=0}^{T-1} \frac{\|\nabla \Phi(x_t)\|^2}{\alpha_{t+1}} \leq& \frac{2}{T}\big(\Phi(x_0) - \Phi(x_{T})\big) + \frac{L_{\Phi}}{2\alpha_{0}^2}\frac{1}{T}\sum_{t=0}^{T-1}\big\|\bar{\nabla}f(x_t, y_t^{P_t}, v_t^{Q_t})\big\|^2 + \frac{1}{T}\sum_{t=0}^{T-1}\frac{\epsilon'}{\alpha_{t+1}} \nonumber \\
    \leq & \frac{1}{T}\bigg(2\big(\Phi(x_0) - \inf_x \Phi(x)\big) + \frac{L_{\Phi}C_\alpha^2}{2\alpha_{0}^2}\bigg) + \frac{\epsilon'}{\alpha_0} = \frac{c_0}{T} + \frac{\epsilon'}{\alpha_0},
\end{align}
where $c_0$ is defined by \cref{eq:c0_double} in \Cref{lm:alpha_double}. 
\endgroup

\noindent
{\bf Case 2: $k_1$ exists. }
For $t < k_1$, according to \Cref{lm:objective_double}, we still have 
\begin{align}\label{eq:converge_double2}
    \frac{\|\nabla \Phi(x_t)\|^2}{\alpha_{t+1}} \leq 2\big(\Phi(x_t) - \Phi(x_{t+1})\big) + \frac{L_{\Phi}}{2\alpha_{t+1}^2}\|\bar{\nabla}f(x_t, y_t^{P_t}, v_t^{Q_t})\|^2 + \frac{\epsilon'}{\alpha_{t+1}}. 
\end{align}
For $t \geq k_1$, we have $\alpha_t \geq C_\alpha$. Using \Cref{lm:objective_double}, we have
\begin{align}\label{eq:converge_double3}
    \frac{\|\nabla \Phi(x_t)\|^2}{\alpha_{t+1}} \leq 2\big(\Phi(x_t) - \Phi(x_{t+1})\big) + \frac{\epsilon'}{\alpha_{t+1}}. 
\end{align}
By merging \cref{eq:converge_double2} and \cref{eq:converge_double3}, and taking an average from $t=0, ..., T-1$, we have
\begin{align}\label{eq:converge_double4}
    \frac{1}{T}\sum_{t=0}^{T-1} \frac{\|\nabla \Phi(x_t)\|^2}{\alpha_{t+1}} =& \frac{1}{T}\sum_{t=0}^{k_1-1} \frac{\|\nabla \Phi(x_t)\|^2}{\alpha_{t+1}} + \frac{1}{T}\sum_{t=k_1}^{T-1} \frac{\|\nabla \Phi(x_t)\|^2}{\alpha_{t+1}} \nonumber \\
    \leq & \frac{2}{T}\big(\Phi(x_0) - \Phi(x_{T})\big) + \frac{L_{\Phi}}{2\alpha_{0}^2}\frac{1}{T}\sum_{t=0}^{k_1-1}\big\|\bar{\nabla}f(x_t, y_t^{P_t}, v_t^{Q_t})\big\|^2 + \frac{1}{T}\sum_{t=0}^{T-1}\frac{\epsilon'}{\alpha_{t+1}} \nonumber \\
    \leq & \frac{1}{T}\bigg(2\big(\Phi(x_0) - \inf_x \Phi(x)\big) + \frac{L_{\Phi}C_\alpha^2}{2\alpha_{0}^2}\bigg) + \frac{\epsilon'}{\alpha_0} = \frac{c_0}{T} + \frac{\epsilon'}{\alpha_0},
\end{align}
where $c_0$ is defined in \Cref{lm:alpha_double}. This result is the same as \cref{eq:converge_double1}. Thus, for both {\bf Case 1} and {\bf Case 2}, we have
\begin{align}
    \frac{1}{T}\sum_{t=0}^{T-1} \frac{\|\nabla \Phi(x_t)\|^2}{\alpha_{T}} \leq \frac{1}{T}\sum_{t=0}^{T-1} \frac{\|\nabla \Phi(x_t)\|^2}{\alpha_{t+1}} \leq \frac{1}{T}\bigg(2\big(\Phi(x_0) - \inf_x \Phi(x)\big) + \frac{L_{\Phi}C_\alpha^2}{2\alpha_{0}^2}\bigg) + \frac{\epsilon'}{\alpha_0}, \nonumber
\end{align}
which indicates that
\begin{align}\label{eq:converge_double5}
    \frac{1}{T}\sum_{t=0}^{T-1}\|\nabla \Phi(x_t)\|^2 \leq & \bigg[\frac{1}{T}\Big(2\big(\Phi(x_0) - \inf_x \Phi(x)\big) + \frac{L_{\Phi}C_\alpha^2}{2\alpha_{0}^2}\Big) + \frac{\epsilon'}{\alpha_0}\bigg]\alpha_{T} \nonumber \\
    \overset{(a)}{\leq} & \frac{1}{T}\bigg[\Big(2\big(\Phi(x_0) - \inf_x \Phi(x)\big) + \frac{L_{\Phi}C_\alpha^2}{2\alpha_{0}^2}\Big) + \frac{T\epsilon'}{\alpha_0}\bigg] \nonumber \\
    & \times \bigg[C_\alpha + 4\big(\Phi(x_{0}) - \inf_x \Phi(x)\big) + \frac{L_{\Phi}C_\alpha^2}{\alpha_{0}^2} + \frac{2T\epsilon'}{\alpha_{0}}\bigg],
\end{align}
where (a) uses \Cref{lm:alpha_double}. 
To achieve the $\mathcal{O}(1/T)$ convergence rate, we need $\epsilon' = \mathcal{O}({1}/{T})$ in \cref{eq:converge_double5}. This can be guaranteed by taking $\epsilon_y = 1/T$ and $\epsilon_v = 1/T$, which implies (see \Cref{lm:objective_double})
\begin{align}\label{eq:epsilonprime1}
    \epsilon' = \frac{1}{T}\bigg[\Big(\frac{2}{\mu^2}\big(\frac{L_{g,2}C_{f_y}}{\mu} + L_{f,1}\big) + 1\Big)L_{g,1}^2\bar{L}^2 + \frac{2\bar{L}^2}{\mu^2}\bigg]. 
\end{align}

\noindent
For symbol convenience, here we define 
\begin{align}\label{eq:c1_double}
    c_1 := c_0 + \frac{1}{\alpha_0}\bigg[\Big(\frac{2}{\mu^2}\big(\frac{L_{g,2}C_{f_y}}{\mu} + L_{f,1}\big) + 1\Big)L_{g,1}^2\bar{L}^2 + \frac{2\bar{L}^2}{\mu^2}\bigg],
\end{align}
where $c_0$ is defined in \cref{eq:c0_double}. 
Thus, we can obtain
\begin{align}
    \frac{1}{T}\sum_{t=0}^{T-1}\|\nabla \Phi(x_t)\|^2 \leq \frac{c_1(C_\alpha + 2c_1)}{T} = \mathcal{O}\Big(\frac{1}{T}\Big). \nonumber
\end{align}
Thus, \Cref{thm:main_double} is proved.

\subsection{Complexity Analysis of \Cref{alg:main_double} (Proof of \Cref{cor:main_double})}
Recall in \Cref{thm:main_double}, we take $\epsilon_y = 1/T$, $\epsilon_v = 1/T$, and we obtain 
\begin{align}
    \frac{1}{T}\sum_{t=0}^{T-1}\|\nabla \Phi(x_t)\|^2 \leq \frac{c_1(C_\alpha + 2c_1)}{T}. \nonumber
\end{align}
To achieve $\epsilon$-accurate stationary point, we need 
\begin{align}\label{eq:T_double}
    \frac{1}{T}\sum_{t=0}^{T-1}\|\nabla \Phi(x_t)\|^2 \leq \frac{c_1(C_\alpha + 2c_1)}{T} \leq \epsilon \quad \mbox{ i.e., } \quad T = \mathcal{O}(1/\epsilon). 
\end{align}
Recall in \Cref{lm:subloop}, we have
\begin{align}
    P_t \leq & \frac{\log(C_\beta^2/\beta_0^2)}{\log(1+\epsilon_y/C_\beta^2)} + \frac{\beta_{\text{max}}}{\mu}\log \big(\frac{L_{g,1}^2(\beta_{\text{max}}-C_\beta)}{\epsilon_y}\big) \nonumber \\
    \leq & \frac{\log(C_\beta^2/\beta_0^2)}{\log(1+1/C_\beta^2T)} + \frac{\beta_{\text{max}}}{\mu}\log \big(\frac{TL_{g,1}^2(\beta_{\text{max}}-C_\beta)}{1}\big) = \mathcal{O}\bigg(\frac{1}{\log(1+\epsilon)} + \log\Big(\frac{1}{\epsilon}\Big)\bigg). \nonumber 
\end{align}
When $\epsilon$ is sufficiently small, we have 
\begin{align}\label{eq:Pt_order}
    P_t = \mathcal{O}\bigg(\frac{1}{\log(1+\epsilon)} + \log\Big(\frac{1}{\epsilon}\Big)\bigg) = \mathcal{O}\bigg(\frac{1}{\epsilon} + \log\Big(\frac{1}{\epsilon}\Big)\bigg) = \mathcal{O}(1/\epsilon). 
\end{align}
Similarly, we have 
\begin{align}\label{eq:Pt_order}
    Q_t = \mathcal{O}\bigg(\frac{1}{\log(1+\epsilon)} + \log\Big(\frac{1}{\epsilon}\Big)\bigg) = \mathcal{O}\bigg(\frac{1}{\epsilon} + \log\Big(\frac{1}{\epsilon}\Big)\bigg) = \mathcal{O}(1/\epsilon). 
\end{align}
We denote ${\rm Gc}(\epsilon)$ as the gradient complexity, then we have 
\begin{align}
    {\rm Gc}(\epsilon) = T\cdot\max_t\{P_t + Q_t\} = \mathcal{O}(1/\epsilon^2). \nonumber
\end{align}
Therefore \Cref{cor:main_double} is proved.

\newpage
\clearpage
\section{Proof of \Cref{thm:main}}
We define some notation for convenience before proving \Cref{thm:main}. 
\subsection{Notation}
Below, we define several preset constants for notational convenience at their first use. 
We first define some Lipschitzness parameters for $\Phi(x)$ as 
\begin{align}
    L_{\Phi} :=& \Big(L_{f,1}+\frac{L_{g,2}C_{f_y}}{\mu}\Big)\Big(1+\frac{C_{g_{xy}}}{\mu}\Big)^2 \nonumber \\
    \bar{L} :=& \max\Big\{2\Big(\frac{C_{f_y}^2L_{g,2}^2}{\mu^2} + L_{f,1}^2\Big)^{\frac{1}{2}}, \sqrt{2}C_{g_{yy}}\Big\}. \nonumber 
\end{align}
Next, we define the following constants as thresholds for parameters $\beta_{k}$, $\gamma_{k}$, $\alpha_{k}$ as 
\begin{align}\label{def:C}
    C_\alpha :=& \max\Big\{{\frac{ 2L_{\Phi}}{\varphi_0}}, {\alpha_0} \Big\}, \nonumber \\
    C_\beta :=& \max\Big\{\mu+L_{g,1}, \frac{2\mu L_{g,1}}{\mu+L_{g,1}}, {\beta_0}, 64a_0^2, 1\Big\}, \nonumber \\
    C_\gamma :=& \max\Big\{2 (\mu+C_{g_{yy}}), 
    \frac{\mu C_{g_{yy}}}{\mu+C_{g_{yy}}}, {\gamma_0}, 64a_0^2, 1, C_{g_{yy}}\Big\},\nonumber \\
    C_\varphi :=& C_\beta + C_\gamma, 
\end{align}
where the constant $\alpha_0$ is defined as
{\small
\begin{align}
    a_0 :=& \Big(\big(\frac{4(\mu+C_{g_{yy}})^2}{\mu C_{g_{yy}}}+8\big)\big(\frac{L_{g,2}C_{f_y}}{\mu}+L_{f,1}\big)^2\frac{1}{\mu^2}+1\Big)\frac{(\mu+L_{g,1})^2L_y^2}{\mu L_{g,1}C_\beta} \nonumber \\
    & + \frac{4(\mu+C_{g_{yy}})(\mu+L_{g,1})L_y^2}{\mu^3 L_{g,1}\varphi_0} \Big(\frac{L_{g,2}C_{f_y}}{\mu}+L_{f,1}\Big)^2 + \frac{4(\mu+C_{g_{yy}})^2L_v^2}{\mu C_{g_{yy}}\gamma_0}. \nonumber
\end{align}}

\subsection{A rough bound of $v_k$}
\begin{lemma}\label{lm:vk}
Under Assumptions \ref{as:sc}, \ref{as:lip}, for any $t \geq 0$ in Algorithm \ref{alg:main}, we have $\|v_{t}\| \leq \frac{\sqrt{2}}{\mu}{\varphi_{t+1}} + \frac{\sqrt{2}C_{f_y}}{\mu}\sqrt{t}$.
\end{lemma}
\begin{proof}
By strong convexity of $g$ in Assumption \ref{as:sc}, we have
\begingroup
\allowdisplaybreaks
\begin{align}
    \sum_{k=1}^{t}\mu^2 \|v_k\|^2 \leq& \sum_{k=1}^{t} \|\nabla_y\nabla_y g(x_k, y_k)v_k\|^2 \nonumber \\
    \leq& \sum_{k=1}^{t} 2\|\nabla_y\nabla_y g(x_k, y_k)v_k - \nabla_y f(x_k, y_k)\|^2 + \sum_{k=1}^{t} 2\|\nabla_y f(x_k, y_k)\|^2 \nonumber \\
    =& \sum_{k=1}^{t}2 \|\nabla_v R(x_k, y_k, v_k)\|^2 + \sum_{k=1}^{t} 2\|\nabla_y f(x_k, y_k)\|^2 \nonumber \\
    \leq& 2\gamma_{t+1}^2 + 2tC_{f_y}^2, \nonumber
\end{align}
\endgroup
which indicates that for any $t\geq0$, $\|v_t\|$ can be bounded as 
\begin{align}\label{eq:vk00}
    \|v_t\| \leq \frac{\big(2\gamma_{t+1}^2 + 2tC_{f_y}^2\big)^{\frac{1}{2}}}{\mu} 
    \leq \frac{\big(2\varphi_{t+1}^2 + 2tC_{f_y}^2\big)^{\frac{1}{2}}}{\mu} \leq  \frac{\sqrt{2}\big({\varphi_{t+1}} + \sqrt{t} C_{f_y}\big)}{\mu}.
\end{align}
Then the proof is complete.  
\end{proof}

\subsection{Descent in Objective Function}
\begin{lemma}\label{lm:objective}
{
Under Assumptions \ref{as:sc}, \ref{as:lip}, for Algorithm \ref{alg:main}, suppose the total iteration number is $T$. 
No matter $k_1$ in \Cref{lm:bar} exists or not, we always have
}
\begin{align}\label{eq:objective_title1}
    \Phi(x_{t+1}) \leq& \Phi(x_t) - \frac{1}{2\alpha_{t+1}\varphi_{t+1}}\|\nabla \Phi(x_t)\|^2 - \frac{1}{2\alpha_{t+1}\varphi_{t+1}}\Big(1-\frac{L_{\Phi}}{\alpha_{t+1}\varphi_{t+1}}\Big)\|\bar{\nabla}f(x_t, y_t, v_t)\|^2 \nonumber \\
    & + \frac{\bar{L}^2}{2\mu^2}\bigg[1 + \frac{2}{\mu^2}\Big(\frac{L_{g,2}C_{f_y}}{\mu}+L_{f,1}\Big)^2\bigg]\frac{\big\|\nabla_y g(x_t,y_t)\big\|^2}{\alpha_{t+1}\varphi_{t+1}} + \frac{\bar{L}^2}{\mu^2}\frac{\big\|\nabla_v R(x_t,y_t,v_t)\big\|^2}{\alpha_{t+1}\varphi_{t+1}}. 
\end{align}
{
If in addition, $k_1$ in \Cref{lm:bar} exists, then for $t\geq k_1$, we further have 
}
\begin{align}\label{eq:objective_title2}
    \Phi(x_{t+1}) 
    \leq& \Phi(x_t) - \frac{1}{2\alpha_{t+1}\varphi_{t+1}}\|\nabla \Phi(x_t)\|^2 - \frac{1}{4\alpha_{t+1}\varphi_{t+1}}\|\bar{\nabla}f(x_t, y_t, v_t)\|^2 \nonumber \\
    & + \frac{\bar{L}^2}{2\mu^2}\bigg[1 + \frac{2}{\mu^2}\Big(\frac{L_{g,2}C_{f_y}}{\mu}+L_{f,1}\Big)^2\bigg]\frac{\big\|\nabla_y g(x_t,y_t)\big\|^2}{\alpha_{t+1}\varphi_{t+1}} + \frac{\bar{L}^2}{\mu^2}\frac{\big\|\nabla_v R(x_t,y_t,v_t)\big\|^2}{\alpha_{t+1}\varphi_{t+1}}, 
\end{align}
where $\bar{L}:= \max\big\{2\big(\frac{C_{f_y}^2L_{g,2}^2}{\mu^2} + L_{f,1}^2\big)^{\frac{1}{2}}, \sqrt{2}C_{g_{yy}}\big\}$.  
\end{lemma}

\begin{proof} 
From \Cref{lm:basic}, we have $\Phi(x)$ is $L_{\Phi}$-smooth. So we can apply the descent lemma to $\Phi$ as 
\begingroup
\allowdisplaybreaks
\begin{align}\label{eq:descent1}
    \Phi(x_{t+1}) \leq& \Phi(x_t) + \langle\nabla\Phi(x_t), x_{t+1} - x_t \rangle + \frac{L_{\Phi}}{2}\|x_{t+1} - x_t\|^2 \nonumber \\
    = & \Phi(x_t) - \frac{1}{\alpha_{t+1}\varphi_{t+1}}\langle\nabla\Phi(x_t), \bar{\nabla}f(x_t, y_t, v_t) \rangle + \frac{L_{\Phi}}{2\alpha_{t+1}^2\varphi_{t+1}^2}\|\bar{\nabla}f(x_t, y_t, v_t)\|^2 \nonumber \\
    = & \Phi(x_t) - \frac{1}{2\alpha_{t+1}\varphi_{t+1}}\|\nabla \Phi(x)\|^2 - \frac{1}{2\alpha_{t+1}\varphi_{t+1}}\|\bar{\nabla}f(x_t, y_t, v_t)\|^2 \nonumber \\
    & + \frac{1}{2\alpha_{t+1}\varphi_{t+1}}\|\nabla \Phi(x_t) - \bar{\nabla}f(x_t, y_t, v_t)\|^2 + \frac{L_{\Phi}}{2\alpha_{t+1}^2\varphi_{t+1}^2}\|\bar{\nabla}f(x_t, y_t, v_t)\|^2, 
\end{align}
\endgroup
and the approximation error 
\begingroup
\allowdisplaybreaks
\begin{align}\label{eq:descent2}
    \|\nabla\Phi&(x_t) - \bar{\nabla}f(x_t, y_t, v_t)\|^2 \nonumber \\
    =& \big\|\bar{\nabla}f\big(x_t, y^*(x_t), v^*(x_t)\big) - \bar{\nabla}f(x_t, y_t, v_t)\big\|^2 \nonumber \\
    \leq & 2\big\|\bar{\nabla}f\big(x_t, y^*(x_t), v^*(x_t)\big) - \bar{\nabla}f\big(x_t, y_t, v^*(x_t)\big)\big\|^2 + 2\big\|\bar{\nabla}f\big(x_t, y_t, v^*(x_t)\big) - \bar{\nabla}f(x_t, y_t, v_t)\big\|^2 \nonumber \\
    \leq & 4\big\|\nabla_y\nabla_y g\big(x_t, y^*(x_t)\big)v^*(x_t) - \nabla_y\nabla_y g(x_t, y_t)v^*(x_t)\big\|^2 \nonumber \\
    & + 4\big\|\nabla_y f\big(x_t, y^*(x_t)\big) - \nabla_y f(x_t, y_t)\big\|^2 + 2\big\|\nabla_y\nabla_y g(x_t, y_t) \big(v^*(x_t) - v_t\big)\big\|^2 \nonumber \\
    \leq & 4\Big(\frac{C_{f_y}^2L_{g,2}^2}{\mu^2} + L_{f,1}^2\Big)\|y_t - y^*(x_t)\|^2 + 2C_{g_{yy}}^2\|v_t - v^*(x_t)\|^2 \nonumber \\
    \leq & \bar{L}^2 \big(\|y_t - y^*(x_t)\|^2 + \|v_t - v^*(x_t)\|^2\big),
\end{align}
\endgroup
where the third inequality used results from Lemma \ref{lm:LS}.
By plugging \cref{eq:descent2} into \cref{eq:descent1}, we have 
\begin{align}\label{eq:descent3}
    \Phi(x_{t+1}) \leq& \Phi(x_t) - \frac{1}{2\alpha_{t+1}\varphi_{t+1}}\|\nabla \Phi(x_t)\|^2 - \frac{1}{2\alpha_{t+1}\varphi_{t+1}}\Big(1-\frac{L_{\Phi}}{\alpha_{t+1}\varphi_{t+1}}\Big)\|\bar{\nabla}f(x_t, y_t, v_t)\|^2 \nonumber \\
    & + \frac{\bar{L}^2}{2\alpha_{t+1}\varphi_{t+1}}\big(\|y_t - y^*(x_t)\|^2 + \|v_t - v^*(x_t)\|^2\big).
\end{align}
Note that $g(x,y)$ is $\mu$-strongly convex in $y$ and $R(x,y,v)$ is $\mu$-strongly convex in $v$. 
So here we can bound the approximation gaps $\|y_t - y^*(x_t)\|^2 + \|v_t - v^*(x_t)\|^2$ by $\|\nabla_y g(x_t,y_t)\|^2$ and $\|\nabla_v R(x_t,y_t,v_t)\|^2$ as 
\begingroup
\allowdisplaybreaks
\begin{align}\label{eq:gapconvert}
    \|y_t -& y^*(x_t)\|^2 + \|v_t - v^*(x_t)\|^2 \nonumber \\
    \overset{(a)}{\leq}& \frac{1}{\mu^2}\big\|\nabla_y g(x_t,y_t) - \nabla_y g\big(x_t,y^*(x_t)\big)\big\|^2 + \frac{1}{\mu^2}\big\|\nabla_v R(x_t,y_t,v_t) - \nabla_v R\big(x_t,y_t,v^*(x_t)\big)\big\|^2 \nonumber \\
    \overset{(b)}{\leq} & \frac{1}{\mu^2}\big\|\nabla_y g(x_t,y_t)\big\|^2 + \frac{2}{\mu^2}\big\|\nabla_v R(x_t,y_t,v_t)\|^2 \nonumber \\
    &+ \frac{2}{\mu^2}\big\|\nabla_v R\big(x_t,y_t,v^*(x_t)\big) - \nabla_v R\big(x_t,y^*(x_t),v^*(x_t)\big)\big\|^2 \nonumber \\
    \overset{(c)}{\leq} & \frac{1}{\mu^2}\big\|\nabla_y g(x_t,y_t)\big\|^2 + \frac{2}{\mu^2}\big\|\nabla_v R(x_t,y_t,v_t)\|^2 + \frac{2}{\mu^2}\bigg(\frac{L_{g,2}C_{f_y}}{\mu}+L_{f,1}\bigg)^2\|y_t - y^*(x_t)\|^2 \nonumber \\
    \overset{(d)}{\leq} & \bigg[\frac{1}{\mu^2} + \frac{2}{\mu^4}\Big(\frac{L_{g,2}C_{f_y}}{\mu}+L_{f,1}\Big)^2\bigg]\big\|\nabla_y g(x_t,y_t)\big\|^2 + \frac{2}{\mu^2}\big\|\nabla_v R(x_t,y_t,v_t)\big\|^2,
\end{align}
\endgroup
where (a) and (d) use the strong convexity; (b) and (d) result from $\nabla_y g\big(x,y^*(x)\big) = 0$ and $\nabla_v R\big(x,y^*(x),v^*(x)\big) = 0$; (c) uses \Cref{lm:LS}. 
By plugging \cref{eq:gapconvert} into \cref{eq:descent3}, we obtain \cref{eq:objective_title1}. 

\noindent
{
Now if in addition, $k_1$ in \Cref{lm:bar} exists, then for $t\geq k_1$, we have $\alpha_{t+1} > C_\alpha\geq 2L_{\Phi}/\varphi_0$. From \eqref{eq:objective_title1} we can immediately obtain \eqref{eq:objective_title2}. Thus, the proof is complete. 
}
\end{proof}

\noindent
Note that to further explore the bounds of the right-hand side of \cref{eq:objective_title1} and \cref{eq:objective_title2} in the above lemma, we next show the (summed) bounds of $\frac{\|\nabla_y g(x_t,y_t)\|^2}{\beta_{t+1}}$ and $\frac{\|\nabla_v R(x_t,y_t,v_t)\|^2}{\varphi_{t+1}}$.

\begin{lemma}\label{lm:sumG}
Under Assumptions \ref{as:sc}, \ref{as:lip}, for Algorithm \ref{alg:main}, suppose the total iteration rounds is $T$. If $k_2$ in \Cref{lm:bar} exists within $T$ iterations, for all integer $t \in [k_2,T]$, we have
\begingroup
\allowdisplaybreaks
{\small
\begin{align}
    \sum_{k=k_2}^t \frac{\|\nabla_y g(x_k, y_k)\|^2}{\beta_{k+1}} 
    \leq \frac{(\mu+L_{g,1})C_\beta^2}{\mu^2} + \frac{(\mu+L_{g,1})^2L_y^2}{\mu L_{g,1}\varphi_{0}} + \frac{(\mu+L_{g,1})^2L_y^2}{\mu L_{g,1}}\sum_{k=k_2}^t\frac{\|\bar{\nabla}f(x_k, y_k, v_k)\|^2}{\alpha_{k+1}^2\varphi_{k+1}}. \nonumber
\end{align}}
\endgroup
\end{lemma}
\begin{proof}
For $k_2 \leq t < T$, we have $\beta_{k_2}\leq C_\beta$  and $\beta_{t+1} > C_\beta$.
For any positive scalar $\bar{\lambda}_{t+1}$, using Young's inequality, we have
\begin{align}\label{eq:boundy1}
    \|y_{t+1} - y^*(x_{t+1})\|^2 \leq (1+\bar{\lambda}_{t+1})\|y_{t+1} - y^*(x_t)\|^2 + \Big(1+\frac{1}{\bar{\lambda}_{t+1}}\Big)\|y^*(x_t) - y^*(x_{t+1})\|^2.
\end{align}
For the first term on the right hand side of \cref{eq:boundy1}, we have 
\begingroup
\allowdisplaybreaks
\begin{align}\label{eq:boundy2}
    \|y_{t+1}& - y^*(x_t)\|^2 \nonumber \\
    =& \Big\|y_{t} - \frac{1}{\beta_{t+1}}\nabla_y g(x_t, y_t) - y^*(x_t)\Big\|^2 \nonumber \\
    =& \|y_{t} - y^*(x_t)\|^2 + \frac{1}{\beta_{t+1}^2}\|\nabla_y g(x_t, y_t)\|^2 - \frac{2}{\beta_{t+1}}\big\langle y_{t} - y^*(x_t), \nabla_y g(x_t, y_t) \big\rangle \nonumber \\
    \overset{(a)}{\leq}& \bigg(1-\frac{2 \mu L_{g,1}}{\beta_{t+1}(\mu+L_{g,1})}\bigg)\|y_{t} - y^*(x_t)\|^2 + \frac{1}{\beta_{t+1}}\bigg(\frac{1}{\beta_{t+1}} - \frac{2}{\mu+L_{g,1}}\bigg)\|\nabla_y g(x_t, y_t)\|^2 \nonumber \\
    \overset{(b)}{\leq}& \bigg(1-\frac{2 \mu L_{g,1}}{\beta_{t+1}(\mu+L_{g,1})}\bigg)\|y_{t} - y^*(x_t)\|^2 - \frac{1}{\beta_{t+1}(\mu+L_{g,1})}\|\nabla_y g(x_t, y_t)\|^2,
\end{align}
\endgroup
where (a) uses Lemma 3.11 in \cite{bubeck2015convex}; (b) follows from $\beta_{t+1} \geq C_{\beta} \geq \mu+L_{g,1}$. 
By plugging \cref{eq:boundy2} into \cref{eq:boundy1}, we have 
\begin{align}\label{eq:boundy3}
    \|&y_{t+1} - y^*(x_{t+1})\|^2 \nonumber \\
    \leq& (1+\bar{\lambda}_{t+1})\bigg(1-\frac{2 \mu L_{g,1}}{\beta_{t+1}(\mu+L_{g,1})}\bigg)\|y_{t} - y^*(x_t)\|^2 - (1+\bar{\lambda}_{t+1})\frac{1}{\beta_{t+1}(\mu+L_{g,1})}\|\nabla_y g(x_t, y_t)\|^2\nonumber \\
    &+ \Big(1+\frac{1}{\bar{\lambda}_{t+1}}\Big)\|y^*(x_t) - y^*(x_{t+1})\|^2.  
\end{align}
By rearranging the terms in \cref{eq:boundy3}, we have 
\begingroup
\allowdisplaybreaks
{
\begin{align}
    (1+&\bar{\lambda}_{t+1})\frac{1}{\beta_{t+1}(\mu+L_{g,1})}\|\nabla_y g(x_t, y_t)\|^2 \nonumber \\
    \leq& (1+\bar{\lambda}_{t+1})\bigg(1-\frac{2 \mu L_{g,1}}{\beta_{t+1}(\mu+L_{g,1})}\bigg)\|y_{t} - y^*(x_t)\|^2 - \|y_{t+1} - y^*(x_{t+1})\|^2 \nonumber \\
    &+ \Big(1+\frac{1}{\bar{\lambda}_{t+1}}\Big)\|y^*(x_t) - y^*(x_{t+1})\|^2 \nonumber.
\end{align}}
\endgroup
We take $\bar{\lambda}_{t+1}:= \frac{2 \mu L_{g,1}}{\beta_{t+1}(\mu+L_{g,1})}$. Since $\beta_{t+1} > C_\beta \geq {\frac{2\mu L_{g,1}}{\mu+L_{g,1}}}$ in \cref{def:C}, we have $\bar{\lambda}_{t+1} \leq {1}$. Then we have 
\begingroup
\allowdisplaybreaks
\begin{align}
    \frac{\|\nabla_y g(x_t, y_t)\|^2}{\beta_{t+1}} \leq &(1+\bar{\lambda}_{t+1})\frac{\|\nabla_y g(x_t, y_t)\|^2}{\beta_{t+1}} \nonumber \\
    \leq& {(\mu+L_{g,1})}\big(\|y_{t} - y^*(x_t)\|^2 - \|y_{t+1} - y^*(x_{t+1})\|^2\big) \nonumber \\
    & + \frac{2(\mu+L_{g,1})}{\bar{\lambda}_{t+1}} \|y^*(x_t) - y^*(x_{t+1})\|^2 \nonumber \\
    =& {(\mu+L_{g,1})}\big(\|y_{t} - y^*(x_t)\|^2 - \|y_{t+1} - y^*(x_{t+1})\|^2\big) \nonumber \\
    & + \frac{(\mu+L_{g,1})^2\beta_{t+1}}{\mu L_{g,1}}\|y^*(x_t) - y^*(x_{t+1})\|^2 \nonumber \\
    \overset{(a)}{\leq}& {(\mu+L_{g,1})}\big(\|y_{t} - y^*(x_t)\|^2 - \|y_{t+1} - y^*(x_{t+1})\|^2\big) \nonumber \\
    & + \frac{(\mu+L_{g,1})^2L_y^2\beta_{t+1}}{\mu L_{g,1}}\|x_t - x_{t+1}\|^2, \nonumber
\end{align}
\endgroup
where (a) uses \Cref{lm:basic}. 
Summing the above inequality over $k=k_2,\ldots,t$, we have 
\begingroup
\allowdisplaybreaks
\begin{align}\label{eq:boundy5}
    \sum_{k=k_2}^t& \frac{\|\nabla_y g(x_k, y_k)\|^2}{\beta_{k+1}} \nonumber \\
    \leq& \sum_{k=k_2-1}^t \frac{\|\nabla_y g(x_k, y_k)\|^2}{\beta_{k+1}} \nonumber \\
    \leq& {(\mu+L_{g,1})}\|y_{{k_2-1}} - y^*(x_{{k_2-1}})\|^2 + \frac{(\mu+L_{g,1})^2L_y^2}{\mu L_{g,1}}\sum_{k={k_2-1}}^t\beta_{k+1}\|x_k - x_{k+1}\|^2 \nonumber \\
    \overset{(a)}{\leq}& \frac{\mu+L_{g,1}}{\mu^2}\big\|\nabla_yg(x_{{k_2-1}}, y_{{k_2-1}}) - \nabla_yg\big(x_{{k_2-1}}, y^*(x_{{k_2-1}})\big)\big\|^2 \nonumber \\
    & + \frac{(\mu+L_{g,1})^2L_y^2}{\mu L_{g,1}}\sum_{k={k_2-1}}^t\frac{\beta_{k+1}}{\alpha_{k+1}^2\varphi_{k+1}^2}\|\bar{\nabla}f(x_k, y_k, v_k)\|^2 \nonumber \\
    \leq& \frac{\mu+L_{g,1}}{\mu^2}\|\nabla_yg(x_{{k_2-1}}, y_{{k_2-1}})\|^2 + \frac{(\mu+L_{g,1})^2L_y^2}{\mu L_{g,1}}\sum_{k={k_2-1}}^t\frac{\|\bar{\nabla}f(x_k, y_k, v_k)\|^2}{\alpha_{k+1}^2\varphi_{k+1}} \nonumber \\
    \overset{(b)}{\leq}& \frac{(\mu+L_{g,1}){C_\beta^2}}{\mu^2} + \frac{(\mu+L_{g,1})^2L_y^2}{\mu L_{g,1}}\sum_{k={k_2-1}}^t\frac{\|\bar{\nabla}f(x_k, y_k, v_k)\|^2}{\alpha_{k+1}^2\varphi_{k+1}} \nonumber \\
    \overset{(c)}{\leq}& \frac{(\mu+L_{g,1}){C_\beta^2}}{\mu^2} + \frac{(\mu+L_{g,1})^2L_y^2}{\mu L_{g,1}\varphi_{0}} + \frac{(\mu+L_{g,1})^2L_y^2}{\mu L_{g,1}}\sum_{k=k_2}^t\frac{\|\bar{\nabla}f(x_k, y_k, v_k)\|^2}{\alpha_{k+1}^2\varphi_{k+1}}, 
\end{align}
\endgroup
where (a) uses Assumption \ref{as:sc}; (b) results from $\|\nabla_yg(x_{{k_2-1}}, y_{{k_2-1}})\|^2 \leq \beta_{k_2}^2 \leq C_\beta^2$; (c) denotes $\varphi_0 = \max\{\beta_0, \gamma_0\}$. 
Then, the proof is complete. 
\end{proof}

\begin{lemma}\label{lm:sumR}
Under Assumptions \ref{as:sc}, \ref{as:lip}, for Algorithm \ref{alg:main}, suppose the total iteration rounds is $T$. If $k_3$ in \Cref{lm:bar} exists within $T$ iterations, for all integer $t\in[k_3,T)$, we have
\begingroup
\allowdisplaybreaks
{
\begin{align}
    \sum_{k=k_3}^t &\frac{\|\nabla_v R(x_k, y_k, v_k)\|^2}{\varphi_{k+1}} \nonumber \\
    \leq & \frac{4(\mu+C_{g_{yy}})C_\beta^2}{\mu^4} \bigg(\frac{L_{g,2}C_{f_y}}{\mu}+L_{f,1}\bigg)^2 + \frac{4(\mu+C_{g_{yy}})C_\gamma^2}{\mu^2} \nonumber \\
    &+ \frac{4(\mu+C_{g_{yy}})(\mu+L_{g,1})L_y^2}{\mu^3 L_{g,1}\varphi_0} \bigg(\frac{L_{g,2}C_{f_y}}{\mu}+L_{f,1}\bigg)^2\sum_{k=k_2-1}^{k_3-2}\frac{\|\Bar{\nabla} f(x_k, y_k, v_k)\|^2}{\alpha_{k+1}^2} \nonumber \\
    &+ \frac{4(\mu+C_{g_{yy}})^2L_v^2}{\mu C_{g_{yy}}C_\gamma}\sum_{k=k_3-1}^t \frac{\|\bar{\nabla}f(x_k, y_k, v_k)\|^2}{\alpha_{k+1}^2} \nonumber \\
    &+ \bigg(\frac{4(\mu+C_{g_{yy}})^2}{\mu C_{g_{yy}}}+8\bigg)\bigg(\frac{L_{g,2}C_{f_y}}{\mu}+L_{f,1}\bigg)^2\frac{1}{\mu^2}\sum_{k=k_3-1}^t\frac{\|\nabla_y g(x_k, y_k)\|^2}{\beta_{k+1}}. \nonumber
\end{align}
}
\endgroup
\end{lemma}
\begin{proof}
For $k_3 \leq t < T$, we have $\gamma_{t+1} > C_\gamma$. 
For any positive scalar $\hat{\lambda}_{t+1}$, using Young's inequality, we have
\begin{align}\label{eq:boundv1}
    \|v_{t+1} - v^*(x_{t+1})\|^2 \leq (1+\hat{\lambda}_{t+1})\|v_{t+1} - v^*(x_t)\|^2 + \Big(1+\frac{1}{\hat{\lambda}_{t+1}}\Big)\|v^*(x_t) - v^*(x_{t+1})\|^2.
\end{align}
For the first term on the right hand side of \cref{eq:boundv1}, we have 
\begin{align}\label{eq:boundv2}
    \|v_{t+1}& - v^*(x_t)\|^2 \nonumber \\
    =& \Big\|v_{t} - \frac{1}{\varphi_{t+1}}\nabla_v R(x_t, y_t, v_t) - v^*(x_t)\Big\|^2 \nonumber \\
    =& \|v_{t} - v^*(x_t)\|^2 + \frac{1}{\varphi_{t+1}^2}\|\nabla_v R(x_t, y_t, v_t)\|^2 - \frac{2}{\varphi_{t+1}}\big\langle v_{t} - v^*(x_t), \nabla_v R(x_t, y_t, v_t) \big\rangle. 
\end{align}
For the last term of the right-hand side of \cref{eq:boundv2}, we have 
\begingroup
\allowdisplaybreaks
\begin{align}\label{eq:boundv3}
    -\langle v_{t}& - v^*(x_t), \nabla_v R(x_t, y_t, v_t) \big\rangle \nonumber \\
    =& -\langle v_{t} - v^*(x_t), \nabla_v R(x_t, y_t, v_t) - \nabla_v R\big(x_t, y_t, v^*(x_t)\big)\big\rangle \nonumber \\
    & -\langle v_{t} - v^*(x_t), \nabla_v R\big(x_t, y_t, v^*(x_t)\big) - \nabla_v R\big(x_t, y^*(x_t), v^*(x_t)\big)\big\rangle \nonumber \\
    \overset{(a)}{\leq} & -\frac{1}{\mu+C_{g_{yy}}}\big\|\nabla_v R(x_t, y_t, v_t) - \nabla_v R\big(x_t, y_t, v^*(x_t)\big)\big\|^2 - \frac{\mu C_{g_{yy}}}{\mu+C_{g_{yy}}}\|v_{t} - v^*(x_t)\|^2 \nonumber \\
    & + \frac{\mu+C_{g_{yy}}}{2\mu C_{g_{yy}}} \big\|\nabla_v R\big(x_t, y_t, v^*(x_t)\big) - \nabla_v R\big(x_t, y^*(x_t), v^*(x_t)\big)\big\|^2 \nonumber \\
    & + \frac{\mu C_{g_{yy}}}{2(\mu+C_{g_{yy}})}\|v_{t} - v^*(x_t)\|^2 \nonumber \\
    \overset{(b)}{\leq}& -\frac{1}{2(\mu+C_{g_{yy}})}\|\nabla_v R(x_t, y_t, v_t)\|^2 + \frac{1}{\mu+C_{g_{yy}}}\|\nabla_v R\big(x_t, y_t, v^*(x_t)\big)\|^2 \nonumber \\
    & + \frac{\mu+C_{g_{yy}}}{2\mu C_{g_{yy}}} \big\|\nabla_v R\big(x_t, y_t, v^*(x_t)\big) - \nabla_v R\big(x_t, y^*(x_t), v^*(x_t)\big)\big\|^2 \nonumber \\
    & - \frac{\mu C_{g_{yy}}}{2(\mu+C_{g_{yy}})}\|v_{t} - v^*(x_t)\|^2 \nonumber \\
    \overset{(c)}{=}& -\frac{1}{2(\mu+C_{g_{yy}})}\|\nabla_v R(x_t, y_t, v_t)\|^2 - \frac{\mu C_{g_{yy}}}{2(\mu+C_{g_{yy}})}\|v_{t} - v^*(x_t)\|^2 \nonumber \\
    & + \bigg(\frac{1}{\mu+C_{g_{yy}}} + \frac{\mu+C_{g_{yy}}}{2\mu C_{g_{yy}}}\bigg)\big\|\nabla_v R\big(x_t, y_t, v^*(x_t)\big) - \nabla_v R\big(x_t, y^*(x_t), v^*(x_t)\big)\big\|^2 \nonumber \\
    \overset{(d)}{\leq}& -\frac{1}{2(\mu+C_{g_{yy}})}\|\nabla_v R(x_t, y_t, v_t)\|^2 - \frac{\mu C_{g_{yy}}}{2(\mu+C_{g_{yy}})}\|v_{t} - v^*(x_t)\|^2 \nonumber \\
    & + \bigg(\frac{1}{\mu+C_{g_{yy}}} + \frac{\mu+C_{g_{yy}}}{2\mu C_{g_{yy}}}\bigg)\big(L_{g,2}\|v^*(x_t)\|+L_{f,1}\big)^2\|y_t - y^*(x_t)\|^2 \nonumber \\
    \overset{(e)}{\leq}& -\frac{1}{2(\mu+C_{g_{yy}})}\|\nabla_v R(x_t, y_t, v_t)\|^2 - \frac{\mu C_{g_{yy}}}{2(\mu+C_{g_{yy}})}\|v_{t} - v^*(x_t)\|^2 \nonumber \\
    & + \bigg(\frac{1}{\mu+C_{g_{yy}}} + \frac{\mu+C_{g_{yy}}}{2\mu C_{g_{yy}}}\bigg)\bigg(\frac{L_{g,2}C_{f_y}}{\mu}+L_{f,1}\bigg)^2\|y_t - y^*(x_t)\|^2, 
\end{align}
\endgroup
where 
(a) follows from Lemma 3.11 in \cite{bubeck2015convex}; 
(b) uses $-\|a-b\|^2\leq-\frac{1}{2}\|a\|^2+\|b\|^2$ since $\|a-b+b\|^2\leq 2\|a-b\|^2+2\|b\|^2$; 
(c) uses $\nabla_v R\big(x_t, y^*(x_t), v^*(x_t)\big) = 0$; 
(d) and (e) follow from \Cref{lm:LS}. 
Plugging \cref{eq:boundv3} into \cref{eq:boundv2}, we have 
\begingroup
\allowdisplaybreaks
\begin{align}\label{eq:boundv4}
    \|v_{t+1}& - v^*(x_t)\|^2 \nonumber \\
    \leq& \bigg(1-\frac{\mu C_{g_{yy}}}{(\mu+C_{g_{yy}})\varphi_{t+1}}\bigg)\|v_{t} - v^*(x_t)\|^2 + \frac{1}{\varphi_{t+1}}\bigg(\frac{1}{\varphi_{t+1}} - \frac{1}{\mu+C_{g_{yy}}}\bigg)\big\|\nabla_v R\big(x_t, y_t, v_t\big)\big\|^2 \nonumber \\
    & + \bigg(\frac{2}{\mu+C_{g_{yy}}} + \frac{\mu+C_{g_{yy}}}{\mu C_{g_{yy}}}\bigg)\bigg(\frac{L_{g,2}C_{f_y}}{\mu}+L_{f,1}\bigg)^2\frac{1}{\varphi_{t+1}}\|y_t - y^*(x_t)\|^2 \nonumber \\
    \overset{(a)}{\leq} & \bigg(1-\frac{\mu C_{g_{yy}}}{(\mu+C_{g_{yy}})\varphi_{t+1}}\bigg)\|v_{t} - v^*(x_t)\|^2 - \frac{1}{2(\mu+C_{g_{yy}})\varphi_{t+1}} \big\|\nabla_v R\big(x_t, y_t, v_t\big)\big\|^2 \nonumber \\
    & + \bigg(\frac{2}{\mu+C_{g_{yy}}} + \frac{\mu+C_{g_{yy}}}{\mu C_{g_{yy}}}\bigg)\bigg(\frac{L_{g,2}C_{f_y}}{\mu}+L_{f,1}\bigg)^2\frac{1}{\varphi_{t+1}}\|y_t - y^*(x_t)\|^2,
\end{align}
\endgroup
where (a) follows from $\varphi_{t+1} \geq \gamma_{t+1} \geq C_{\gamma} \geq 2 (\mu+C_{g_{yy}})$. Combining \cref{eq:boundv4} with \cref{eq:boundv1}, we have 
\begingroup
\allowdisplaybreaks
\begin{align}\label{eq:boundv6}
    \|v_{t+1} &- v^*(x_{t+1})\|^2 \nonumber \\
    \leq& (1+\hat{\lambda}_{t+1}) \bigg(1-\frac{\mu C_{g_{yy}}}{(\mu+C_{g_{yy}})\varphi_{t+1}}\bigg)\|v_{t} - v^*(x_t)\|^2 \nonumber \\
    &- (1+\hat{\lambda}_{t+1})\frac{1}{2(\mu+C_{g_{yy}})\varphi_{t+1}} \big\|\nabla_v R\big(x_t, y_t, v_t\big)\big\|^2 \nonumber \\
    &+ (1+\hat{\lambda}_{t+1}) \bigg(\frac{2}{\mu+C_{g_{yy}}} + \frac{\mu+C_{g_{yy}}}{\mu C_{g_{yy}}}\bigg)\bigg(\frac{L_{g,2}C_{f_y}}{\mu}+L_{f,1}\bigg)^2\frac{1}{\varphi_{t+1}}\|y_t - y^*(x_t)\|^2 \nonumber \\
    &+ \Big(1+\frac{1}{\hat{\lambda}_{t+1}}\Big)\|v^*(x_t) - v^*(x_{t+1})\|^2.  
\end{align}
\endgroup
By rearranging the terms in \cref{eq:boundv6}, we have
\begingroup
\allowdisplaybreaks
\begin{align}\label{eq:boundv4_5}
    (1+\hat{\lambda}_{t+1})&\frac{1}{2(\mu+C_{g_{yy}})\varphi_{t+1}} \big\|\nabla_v R\big(x_t, y_t, v_t\big)\big\|^2 \nonumber \\
    \leq& (1+\hat{\lambda}_{t+1}) \bigg(1-\frac{\mu C_{g_{yy}}}{(\mu+C_{g_{yy}})\varphi_{t+1}}\bigg)\|v_{t} - v^*(x_t)\|^2 - \|v_{t+1} - v^*(x_{t+1})\|^2 \nonumber \\
    &+ (1+\hat{\lambda}_{t+1}) \bigg(\frac{2}{\mu+C_{g_{yy}}} + \frac{\mu+C_{g_{yy}}}{\mu C_{g_{yy}}}\bigg)\bigg(\frac{L_{g,2}C_{f_y}}{\mu}+L_{f,1}\bigg)^2\frac{1}{\varphi_{t+1}}\|y_t - y^*(x_t)\|^2 \nonumber \\
    &+ {\Big(1+\frac{1}{\hat{\lambda}_{t+1}}\Big)\|v^*(x_t) - v^*(x_{t+1})\|^2}. 
\end{align}
\endgroup
We now take $\hat{\lambda}_{t+1}:= \frac{\mu C_{g_{yy}}}{(\mu+C_{g_{yy}})\varphi_{t+1}}$. 
Since $\varphi_{t+1} \geq \gamma_{t+1} \geq C_{\gamma} \geq \frac{\mu C_{g_{yy}}}{\mu+C_{g_{yy}}}$ in \cref{def:C}, we have $\hat{\lambda}_{t+1} \leq {1}$. 
Then we get
\begingroup
\allowdisplaybreaks
{
\small 
\begin{align}\label{eq:boundv5}
    \frac{\|\nabla_v R(x_t, y_t, v_t)\|^2}{\varphi_{t+1}} <& 
    (1+\hat{\lambda}_{t+1})\frac{\|\nabla_v R(x_t, y_t, v_t)\|^2}{\varphi_{t+1}} \nonumber \\
    \overset{(a)}{\leq} & {2(\mu+C_{g_{yy}})}\big(\|v_t - v^*(x_t)\|^2 - \|v_{t+1} - v^*(x_{t+1})\|^2\big) \nonumber \\
    & + 4(\mu+C_{g_{yy}})\bigg(\frac{2}{\mu+C_{g_{yy}}} + \frac{\mu+C_{g_{yy}}}{\mu C_{g_{yy}}}\bigg)\bigg(\frac{L_{g,2}C_{f_y}}{\mu}+L_{f,1}\bigg)^2\frac{\|y_t - y^*(x_t)\|^2}{\varphi_{t+1}} \nonumber \\
    & + 2(\mu+C_{g_{yy}})\bigg(1+\frac{(\mu+C_{g_{yy}})\varphi_{t+1}}{\mu C_{g_{yy}}}\bigg)L_v^2\|x_t - x_{t+1}\|^2 \nonumber \\
    \overset{(b)}{\leq}& {2(\mu+C_{g_{yy}})}\big(\|v_t - v^*(x_t)\|^2 - \|v_{t+1} - v^*(x_{t+1})\|^2\big) \nonumber \\
    & + \bigg(\frac{4(\mu+C_{g_{yy}})^2}{\mu C_{g_{yy}}}+8\bigg)\bigg(\frac{L_{g,2}C_{f_y}}{\mu}+L_{f,1}\bigg)^2\frac{\|y_t - y^*(x_t)\|^2}{\varphi_{t+1}} \nonumber \\
    & + \frac{4(\mu+C_{g_{yy}})^2L_v^2\varphi_{t+1}}{\mu C_{g_{yy}}}\|x_t - x_{t+1}\|^2,
\end{align}
}
\endgroup
where 
(a) multiplies both sides of \cref{eq:boundv4_5} by $2(\mu+C_{g_{xy}})$ and uses $\hat{\lambda}_{t+1} \leq {1}$; 
(b) uses $\varphi_{t+1} \geq \gamma_{t+1} \geq C_{\gamma} \geq \frac{\mu C_{g_{yy}}}{\mu+C_{g_{yy}}}$. 
Take summation of \cref{eq:boundv5} and we have
\begingroup
\allowdisplaybreaks
{\small
\begin{align}\label{eq:boundR1}
    \sum_{k=k_3}^t& \frac{\|\nabla_v R(x_k, y_k, v_k)\|^2}{\varphi_{k+1}} \nonumber \\
    \leq& \sum_{k={k_3-1}}^t \frac{\|\nabla_v R(x_k, y_k, v_k)\|^2}{\varphi_{k+1}} \nonumber \\
    \leq& {2(\mu+C_{g_{yy}})}\|v_{{k_3-1}} - v^*(x_{{k_3-1}})\|^2 + \frac{4(\mu+C_{g_{yy}})^2L_v^2}{\mu C_{g_{yy}}}\sum_{k={k_3-1}}^t\varphi_{k+1}\|x_k - x_{k+1}\|^2 \nonumber \\
    &+ \bigg(\frac{4(\mu+C_{g_{yy}})^2}{\mu C_{g_{yy}}}+8\bigg)\bigg(\frac{L_{g,2}C_{f_y}}{\mu}+L_{f,1}\bigg)^2\sum_{k={k_3-1}}^t\frac{\|y_k - y^*(x_k)\|^2}{\varphi_{k+1}} \nonumber \\
    \leq& {2(\mu+C_{g_{yy}})}\|v_{{k_3-1}} - v^*(x_{{k_3-1}})\|^2 + \frac{4(\mu+C_{g_{yy}})^2L_v^2}{\mu C_{g_{yy}}}\sum_{k={k_3-1}}^t \frac{\|\bar{\nabla}f(x_k, y_k, v_k)\|^2}{\alpha_{k+1}^2\varphi_{k+1}} \nonumber \\
    &+ \bigg(\frac{4(\mu+C_{g_{yy}})^2}{\mu C_{g_{yy}}}+8\bigg)\bigg(\frac{L_{g,2}C_{f_y}}{\mu}+L_{f,1}\bigg)^2\sum_{k={k_3-1}}^t\frac{\|y_k - y^*(x_k)\|^2}{\varphi_{k+1}} \nonumber \\
    \leq& {2(\mu+C_{g_{yy}})}\|v_{{k_3-1}} - v^*(x_{{k_3-1}})\|^2 + \frac{4(\mu+C_{g_{yy}})^2L_v^2}{\mu C_{g_{yy}}}\sum_{k={k_3-1}}^t \frac{\|\bar{\nabla}f(x_k, y_k, v_k)\|^2}{\alpha_{k+1}^2\varphi_{k+1}} \nonumber \\
    &+ \bigg(\frac{4(\mu+C_{g_{yy}})^2}{\mu C_{g_{yy}}}+8\bigg)\bigg(\frac{L_{g,2}C_{f_y}}{\mu}+L_{f,1}\bigg)^2\sum_{k={k_3-1}}^t\frac{\|y_k - y^*(x_k)\|^2}{\beta_{k+1}} \nonumber \\
    \overset{(a)}{\leq}& {2(\mu+C_{g_{yy}})}\|v_{{k_3-1}} - v^*(x_{{k_3-1}})\|^2 + \frac{4(\mu+C_{g_{yy}})^2L_v^2}{\mu C_{g_{yy}}}\sum_{k={k_3-1}}^t \frac{\|\bar{\nabla}f(x_k, y_k, v_k)\|^2}{\alpha_{k+1}^2\varphi_{k+1}} \nonumber \\
    &+ \bigg(\frac{4(\mu+C_{g_{yy}})^2}{\mu C_{g_{yy}}}+8\bigg)\bigg(\frac{L_{g,2}C_{f_y}}{\mu}+L_{f,1}\bigg)^2\frac{1}{\mu^2}\sum_{k={k_3-1}}^t\frac{\big\|\nabla_y g(x_k, y_k) - \nabla_y g\big(x_k, y^*(x_k)\big)\big\|^2}{\beta_{k+1}} \nonumber \\
    \overset{(b)}{\leq}& {2(\mu+C_{g_{yy}})}\|v_{{k_3-1}} - v^*(x_{{k_3-1}})\|^2 + \frac{4(\mu+C_{g_{yy}})^2L_v^2}{\mu C_{g_{yy}}}\sum_{k={k_3-1}}^t \frac{\|\bar{\nabla}f(x_k, y_k, v_k)\|^2}{\alpha_{k+1}^2\varphi_{k+1}} \nonumber \\
    &+ \bigg(\frac{4(\mu+C_{g_{yy}})^2}{\mu C_{g_{yy}}}+8\bigg)\bigg(\frac{L_{g,2}C_{f_y}}{\mu}+L_{f,1}\bigg)^2\frac{1}{\mu^2}\sum_{k={k_3-1}}^t\frac{\|\nabla_y g(x_k, y_k)\|^2}{\beta_{k+1}} \nonumber \\
    \overset{(c)}{\leq}& \frac{2(\mu+C_{g_{yy}})}{\mu^2}\big\|\nabla_v R\big(x_{{k_3-1}}, y_{{k_3-1}}, v_{{k_3-1}}\big) - \nabla_v R\big(x_{{k_3-1}}, y_{{k_3-1}}, v^*(x_{{k_3-1}})\big)\big\|^2 \nonumber \\
    &+ \frac{4(\mu+C_{g_{yy}})^2L_v^2}{\mu C_{g_{yy}}}\sum_{k={k_3-1}}^t \frac{\|\bar{\nabla}f(x_k, y_k, v_k)\|^2}{\alpha_{k+1}^2\varphi_{k+1}} \nonumber \\
    &+ \bigg(\frac{4(\mu+C_{g_{yy}})^2}{\mu C_{g_{yy}}}+8\bigg)\bigg(\frac{L_{g,2}C_{f_y}}{\mu}+L_{f,1}\bigg)^2\frac{1}{\mu^2}\sum_{k={k_3-1}}^t\frac{\|\nabla_y g(x_k, y_k)\|^2}{\beta_{k+1}} \nonumber \\
    \overset{(d)}{\leq}& \frac{4(\mu+C_{g_{yy}})}{\mu^2}\big\|\nabla_v R\big(x_{{k_3-1}}, y^*(x_{{k_3-1}}), v^*(x_{{k_3-1}})\big) - \nabla_v R\big(x_{{k_3-1}}, y_{{k_3-1}}, v^*(x_{{k_3-1}})\big)\big\|^2 \nonumber \\
    &+ \frac{4(\mu+C_{g_{yy}})}{\mu^2}\big\|\nabla_v R(x_{{k_3-1}}, y_{{k_3-1}}, v_{{k_3-1}}) - \nabla_v R\big(x_{{k_3-1}}, y^*(x_{{k_3-1}}), v^*(x_{{k_3-1}})\big)\big\|^2 \nonumber \\
    &+ \frac{4(\mu+C_{g_{yy}})^2L_v^2}{\mu C_{g_{yy}}}\sum_{k={k_3-1}}^t \frac{\|\bar{\nabla}f(x_k, y_k, v_k)\|^2}{\alpha_{k+1}^2\varphi_{k+1}} \nonumber \\
    &+ \bigg(\frac{4(\mu+C_{g_{yy}})^2}{\mu C_{g_{yy}}}+8\bigg)\bigg(\frac{L_{g,2}C_{f_y}}{\mu}+L_{f,1}\bigg)^2\frac{1}{\mu^2}\sum_{k={k_3-1}}^t\frac{\|\nabla_y g(x_k, y_k)\|^2}{\beta_{k+1}} \nonumber \\
    \overset{(e)}{\leq}& \frac{4(\mu+C_{g_{yy}})}{\mu^2} \bigg(\frac{L_{g,2}C_{f_y}}{\mu}+L_{f,1}\bigg)^2{\|y_{{k_3-1}} - y^*(x_{{k_3-1}})\|^2} \nonumber \\
    &+ \frac{4(\mu+C_{g_{yy}})}{\mu^2}\|\nabla_v R(x_{{k_3-1}}, y_{{k_3-1}}, v_{{k_3-1}})\|^2 + \frac{4(\mu+C_{g_{yy}})^2L_v^2}{\mu C_{g_{yy}}}\sum_{k={k_3-1}}^t \frac{\|\bar{\nabla}f(x_k, y_k, v_k)\|^2}{\alpha_{k+1}^2\varphi_{k+1}} \nonumber \\
    &+ \bigg(\frac{4(\mu+C_{g_{yy}})^2}{\mu C_{g_{yy}}}+8\bigg)\bigg(\frac{L_{g,2}C_{f_y}}{\mu}+L_{f,1}\bigg)^2\frac{1}{\mu^2}\sum_{k={k_3-1}}^t\frac{\|\nabla_y g(x_k, y_k)\|^2}{\beta_{k+1}},
\end{align}
}
\endgroup
where (a) uses Assumption \ref{as:sc}; 
(b) results from $\nabla_y g\big(x, y^*(x)\big) = 0$;
(c) uses the strong convexity in \Cref{lm:LS};
(d) uses $\nabla_v R\big(x, y^*(x), v^*(x)\big) = 0$; 
(e) follows from \Cref{lm:LS}.

\noindent
Our next step is bounding $\|y_{{k_3-1}} - y^*(x_{{k_3-1}})\|^2$ on the right hand side of \cref{eq:boundR1} in two cases. 
{\bf The first case is {$\beta_{k_3} \leq C_\beta$}.} In this case, 
by using strong convexity of $g$ and the definition of $\beta_{k_3}$, 
we can easily have 
\begingroup
\allowdisplaybreaks
\begin{align}\label{eq:yk3_1}
    \|y_{{k_3-1}} - y^*(x_{{k_3-1}})\|^2 
    \leq& \frac{1}{\mu^2}\big\|\nabla_y g\big(x_{{k_3-1}}, y_{{k_3-1}}) - \nabla_y g(x_{{k_3-1}}, y^*(x_{{k_3-1}})\big)\big\|^2 \nonumber \\
    =& \frac{1}{\mu^2}\|\nabla_y g(x_{{k_3-1}}, y_{{k_3-1}}))\|^2 
    \leq  \frac{\beta_{k_3}^2}{\mu^2} \leq \frac{C_\beta^2}{\mu^2}.
\end{align}
\endgroup

\noindent
{\bf The second case is {$\beta_{k_3} > C_\beta$}.}  
This indicates that $k_2$ exists and $k_3 > k_2$ based on \Cref{lm:bar}. 
By plugging $\bar{\lambda}_{k_3-1}:= \frac{2 \mu L_{g,1}}{\beta_{k_3-1}(\mu+L_{g,1})}$ into \cref{eq:boundy3}, and noting $\bar{\lambda}_{k_3-1}\leq 1$, we have 
\begingroup
\allowdisplaybreaks
{
\small
\begin{align}\label{eq:yk3_2}
    \|y_{{k_3-1}} - y^*(x_{{k_3-1}})\|^2 \leq& \|y_{k_3-2} - y^*(x_{k_3-2})\|^2 + \frac{(\mu+L_{g,1})\beta_{k_3-1}}{\mu L_{g,1}}\|y^*(x_{k_3-2}) - y^*(x_{k_3-1})\|^2 \nonumber \\
    \overset{(a)}{\leq}& \|y_{k_3-2} - y^*(x_{k_3-2})\|^2 + \frac{(\mu+L_{g,1})L_y^2\beta_{k_3-1}}{\mu L_{g,1}}\|x_{k_3-2} - x_{k_3-1}\|^2 \nonumber \\
    =& \|y_{k_3-2} - y^*(x_{k_3-2})\|^2 + \frac{(\mu+L_{g,1})L_y^2\beta_{k_3-1}}{\mu L_{g,1}}\frac{\|\Bar{\nabla} f(x_{k_3-2}, y_{k_3-2}, v_{k_3-2})\|^2}{\alpha_{k_3-1}^2\varphi_{k_3-1}^2} \nonumber \\
    \leq& \|y_{k_3-2} - y^*(x_{k_3-2})\|^2 + \frac{(\mu+L_{g,1})L_y^2}{\mu L_{g,1}}\frac{\|\Bar{\nabla} f(x_{k_3-2}, y_{k_3-2}, v_{k_3-2})\|^2}{\alpha_{k_3-1}^2\varphi_{k_3-1}} \nonumber \\
    \leq& \|y_{k_3-2} - y^*(x_{k_3-2})\|^2 + \frac{(\mu+L_{g,1})L_y^2}{\mu L_{g,1}\varphi_0} \frac{\|\Bar{\nabla} f(x_{k_3-2}, y_{k_3-2}, v_{k_3-2})\|^2}{\alpha_{k_3-1}^2} \nonumber \\
    \leq& \|y_{k_2-1} - y^*(x_{k_2-1})\|^2 + \frac{(\mu+L_{g,1})L_y^2}{\mu L_{g,1}\varphi_0}\sum_{k=k_2-1}^{k_3-2}\frac{\|\Bar{\nabla} f(x_k, y_k, v_k)\|^2}{\alpha_{k+1}^2} \nonumber \\
    \overset{(b)}{\leq}& \frac{C_\beta^2}{\mu^2} + \frac{(\mu+L_{g,1})L_y^2}{\mu L_{g,1}\varphi_0}\sum_{k=k_2-1}^{k_3-2}\frac{\|\Bar{\nabla} f(x_k, y_k, v_k)\|^2}{\alpha_{k+1}^2},
\end{align}
}
\endgroup
where (a) uses \Cref{lm:basic}; (b) uses \cref{eq:yk3_1} by replacing $k_3$ by $k_2$ since $\beta_{k_2} \leq C_\beta$ (see \Cref{lm:bar}). 
By combining \cref{eq:yk3_1} and \cref{eq:yk3_2}, we obtain a {\bf general} upper bound of $\|y_{k_3-1} - y^*(x_{k_3-1})\|^2$ as 
{\small
\begin{align}\label{eq:yk3_3}
    \|y_{k_3-1} - y^*(x_{k_3-1})\|^2 \leq \frac{C_\beta^2}{\mu^2} + \frac{(\mu+L_{g,1})L_y^2}{\mu L_{g,1}\varphi_0}\sum_{k=k_2-1}^{k_3-2}\frac{\|\Bar{\nabla} f(x_k, y_k, v_k)\|^2}{\alpha_{k+1}^2}, 
\end{align}}
where we define $\sum_{k=m}^n l_k = 0$ for any $m>n$ and non-negative sequence $\{l_k\}$. 
By plugging \cref{eq:yk3_3} into \cref{eq:boundR1} and using $\|\nabla_v R(x_{{k_3-1}}, y_{{k_3-1}}, v_{{k_3-1}})\|^2 \leq \gamma_{k_3}^2 \leq C_\gamma^2$, we have 
\begingroup
\allowdisplaybreaks
{
\begin{align}
    \sum_{k=k_3}^t &\frac{\|\nabla_v R(x_k, y_k, v_k)\|^2}{\varphi_{k+1}} \nonumber \\
    \leq & \frac{4(\mu+C_{g_{yy}})C_\beta^2}{\mu^4} \bigg(\frac{L_{g,2}C_{f_y}}{\mu}+L_{f,1}\bigg)^2 + \frac{4(\mu+C_{g_{yy}})C_\gamma^2}{\mu^2} \nonumber \\
    &+ \frac{4(\mu+C_{g_{yy}})(\mu+L_{g,1})L_y^2}{\mu^3 L_{g,1}\varphi_0} \bigg(\frac{L_{g,2}C_{f_y}}{\mu}+L_{f,1}\bigg)^2\sum_{k=k_2-1}^{k_3-2}\frac{\|\Bar{\nabla} f(x_k, y_k, v_k)\|^2}{\alpha_{k+1}^2}\nonumber \\
    &+ \frac{4(\mu+C_{g_{yy}})^2L_v^2}{\mu C_{g_{yy}}C_\gamma}\sum_{k=k_3-1}^t \frac{\|\bar{\nabla}f(x_k, y_k, v_k)\|^2}{\alpha_{k+1}^2} \nonumber \\
    &+ \bigg(\frac{4(\mu+C_{g_{yy}})^2}{\mu C_{g_{yy}}}+8\bigg)\bigg(\frac{L_{g,2}C_{f_y}}{\mu}+L_{f,1}\bigg)^2\frac{1}{\mu^2}\sum_{k=k_3-1}^t\frac{\|\nabla_y g(x_k, y_k)\|^2}{\beta_{k+1}}. \nonumber
\end{align}
}
\endgroup
Then, the proof is complete. 
\end{proof}

\noindent
Supported by \Cref{lm:sumG} and \Cref{lm:sumR}, we derive upper bounds of $\beta_t$ and $\varphi_t$. 
\begin{lemma}\label{lm:beta}
Suppose the total iteration rounds of Algorithm \ref{alg:main} is $T$. Under Assumptions \ref{as:sc}, \ref{as:lip}, 
if $k_2$ in \Cref{lm:bar} exists within $T$ iterations, 
we have 
{
\small
$$
\left\{
\begin{aligned}
\beta_{t+1} \leq& C_\beta, &  t < k_2; \\
\beta_{t+1} \leq & \Big(C_\beta+\frac{(\mu+L_{g,1})C_\beta^2}{\mu^2} + \frac{(\mu+L_{g,1})^2L_y^2}{\mu L_{g,1}\varphi_{0}}\Big) + \frac{(\mu+L_{g,1})^2L_y^2}{\mu L_{g,1}C_\beta}\sum_{k=k_2}^t\frac{\|\bar{\nabla}f(x_k, y_k, v_k)\|^2}{\alpha_{k+1}^2},  &   t \geq k_2.
\end{aligned}
\right.
$$
}
\noindent
When such $k_2$ does not exist, $\beta_{t+1} \leq C_\beta$ holds for any $t < T$.
\end{lemma}
\begin{proof}
According to \Cref{lm:bar}, the proof can be split into the following three cases. 

\noindent
{\bf Case 1: $k_2$ does not exist:} In this case, based on \Cref{lm:bar}, we have $\beta_T \leq C_\beta$, and hence $\beta_{t+1} \leq C_\beta$ for any $t < T$ because $\beta_{t}$ is non-decreasing with $t$. 

\noindent
{\bf Case 2: $k_2$ exists and $t<k_2$:} In this case, based on \Cref{lm:bar}, we have $\beta_{t+1} \leq C_\beta$.

\noindent
{\bf Case 3: $k_2$ exists and $t\geq k_2$:} Inspired by \cite{ward2020adagrad} and using telescoping, we have 
\begingroup
\allowdisplaybreaks
\begin{align}\label{eq:boundy4}
    \beta_{t+1} =& \beta_t + \frac{\|\nabla_y g(x_t, y_t)\|^2}{\beta_{t+1} + \beta_t} \nonumber \\
    \leq& \beta_t + \frac{\|\nabla_y g(x_t, y_t)\|^2}{\beta_{t+1}} \nonumber \\
    \leq& \beta_{{k_2}} + \sum_{k={k_2}}^t\frac{\|\nabla_y g(x_k, y_k)\|^2}{\beta_{k+1}} \nonumber \\
    \overset{(a)}{\leq}& \bigg(C_\beta+\frac{(\mu+L_{g,1})C_\beta^2}{\mu^2} + \frac{(\mu+L_{g,1})^2L_y^2}{\mu L_{g,1}\varphi_{0}}\bigg) + \frac{(\mu+L_{g,1})^2L_y^2}{\mu L_{g,1}C_\beta}\sum_{k=k_2}^t\frac{\|\bar{\nabla}f(x_k, y_k, v_k)\|^2}{\alpha_{k+1}^2},
\end{align}
\endgroup
where (a) uses \cref{lm:sumG}. Thus, the proof is complete.
\end{proof}

\begin{lemma}\label{lm:varphi}
Under Assumptions \ref{as:sc}, \ref{as:lip}, suppose the total iteration rounds of Algorithm \ref{alg:main} is $T$.  
{If at least one of $k_2$ and $k_3$} in \Cref{lm:bar} exists, we denote $k_{\text{min}}:=\min\{k_2,k_3\}$. 
Then we have the upper bound of $\varphi_{t}$ as 
$$
\left\{
\begin{aligned}
\varphi_{t} \leq & C_\varphi, \quad & t \leq k_{\text{min}}; \\
\varphi_{t} \leq & a_1\log(t) + b_1, \quad & t > k_{\text{min}},
\end{aligned}
\right.
$$
where $a_1$, $b_1$ are defined as
\begingroup
\allowdisplaybreaks
\begin{align}\label{def:a1b1}
    a_1 := 6a_0, \quad b_1 := 4a_0\log\Big(1 + \frac{C_{g_{xy}}\bar{b}+C_{f_x}+\alpha_0}{C_{g_{xy}}\bar{a}}\Big) + 4a_0\log(C_{g_{xy}}\bar{a}) +4a_0 + 2b_0,
\end{align}
in which we define constants 
{
\small
\begin{align}\label{def:barab_a0b0}
    \bar{a}:=& \frac{\sqrt{2}}{\mu}, \quad 
    \bar{b}:= \frac{\sqrt{2}C_{f_y}}{\mu}, \nonumber \\
    a_0 :=& \bigg(\Big(\frac{4(\mu+C_{g_{yy}})^2}{\mu C_{g_{yy}}}+8\Big)\Big(\frac{L_{g,2}C_{f_y}}{\mu}+L_{f,1}\Big)^2\frac{1}{\mu^2}+1\bigg)\frac{(\mu+L_{g,1})^2L_y^2}{\mu L_{g,1}C_\beta} \nonumber \\
    &+ \frac{4(\mu+C_{g_{yy}})(\mu+L_{g,1})L_y^2}{\mu^3 L_{g,1}\varphi_0} \bigg(\frac{L_{g,2}C_{f_y}}{\mu}+L_{f,1}\bigg)^2 + \frac{4(\mu+C_{g_{yy}})^2L_v^2}{\mu C_{g_{yy}}\gamma_0}, \nonumber \\
    b_0 :=& C_\beta + C_\gamma + \frac{4(\mu+C_{g_{yy}})C_\beta^2}{\mu^4} \bigg(\frac{L_{g,2}C_{f_y}}{\mu}+L_{f,1}\bigg)^2 + \frac{4(\mu+C_{g_{yy}})C_\gamma^2}{\mu^2} \nonumber \\
    & + \Big(\frac{4(\mu+C_{g_{yy}})^2}{\mu C_{g_{yy}}}+8\Big)\Big(\frac{L_{g,2}C_{f_y}}{\mu}+L_{f,1}\Big)^2\frac{1}{\mu^2}\Big(\frac{C_\beta^2}{\beta_{0}}-\beta_{0}\Big) \nonumber \\
    & + \bigg[\Big(\frac{4(\mu+C_{g_{yy}})^2}{\mu C_{g_{yy}}}+8\Big)\Big(\frac{L_{g,2}C_{f_y}}{\mu}+L_{f,1}\Big)^2\frac{1}{\mu^2}+1\bigg]\bigg(\frac{(\mu+L_{g,1})C_\beta^2}{\mu^2} + \frac{(\mu+L_{g,1})^2L_y^2}{\mu L_{g,1}\varphi_{0}}\bigg).
\end{align}
}
\endgroup 

\noindent
When such $k_2$ and $k_3$ do not exist, we have $\varphi_{t} \leq C_\varphi$ for all $t \leq T$. 
\end{lemma}
\begin{proof}
To begin with, 
we first show the following result as the first two lines of \cref{eq:boundy4}: 
since $\beta_t$ and $\gamma_t$ are positive and increasing monotonically with $t$, we can easily have
\begingroup
\allowdisplaybreaks
\begin{align}
     0 \leq& \min\{\beta_{t+1}^2, \gamma_{t+1}^2\} - \min\{\beta_t^2, \gamma_t^2\} \nonumber \\
     =& \big(\beta_{t+1}^2 + \gamma_{t+1}^2 - \max\{\beta_{t+1}^2, \gamma_{t+1}^2\}\big) - \big(\beta_{t}^2 + \gamma_{t}^2 - \max\{\beta_{t}^2, \gamma_{t}^2\}\big) \nonumber \\
     \overset{(a)}{=}& (\beta_{t+1}^2 + \gamma_{t+1}^2) - (\beta_{t}^2 + \gamma_{t}^2) - (\varphi_{t+1}^2 - \varphi_t^2), \nonumber
\end{align}
\endgroup
where (a) uses the definition $\varphi_t:=\max\{\beta_t, \gamma_t\}$. 
Similar to \cref{eq:boundy4}, we have
\begingroup
\allowdisplaybreaks
\begin{align}
    \varphi_{t+1}^2 - \varphi_t^2 
    \leq& (\beta_{t+1}^2 - \beta_{t}^2) + (\gamma_{t+1}^2 - \gamma_{t}^2) = \|\nabla_y g(x_t, y_t)\|^2 + \|\nabla_v R(x_t, y_t, v_t)\|^2, \nonumber
\end{align}
which indicates that 
\begin{align}\label{eq:varphi_1}
    \varphi_{t+1} \leq& \varphi_t + \frac{\|\nabla_y g(x_t, y_t)\|^2}{\varphi_{t+1} + \varphi_t} + \frac{\|\nabla_v R(x_t, y_t, v_t)\|^2}{\varphi_{t+1} + \varphi_t} \nonumber \\
    \leq& \varphi_t + \frac{\|\nabla_y g(x_t, y_t)\|^2}{\beta_{t+1} + \beta_t} + \frac{\|\nabla_v R(x_t, y_t, v_t)\|^2}{\varphi_{t+1}} \nonumber \\
    \leq& \varphi_t + \frac{\|\nabla_y g(x_t, y_t)\|^2}{\beta_{t+1}} + \frac{\|\nabla_v R(x_t, y_t, v_t)\|^2}{\varphi_{t+1}}. 
\end{align}
\endgroup

\noindent
Note that, to simplify the proof, 
we define $\sum_{k=m}^n l_k = 0$ for any $m>n$ and non-negative sequence $\{l_k\}$. 
According to the definitions of $\bm{k_2}$ and $\bm{k_3}$ in \Cref{lm:bar}, the proof can be split into the following four cases. 

\noindent
{\bf Case 1: neither $k_2$ nor $k_3$ exists:}
for any $t \in (0,T)$, we can easily have {\small $\varphi_t = \max\{\beta_t, \gamma_t\} \leq \max\{C_\beta, C_\gamma\} \leq C_\varphi$}.  

\noindent
{\bf Case 2: $k_2$ exists but $k_3$ does not:} 
by using the third line of \cref{eq:boundy4}, 
for any $t \in (0,T)$, we have 
\begingroup
\allowdisplaybreaks
\begin{align}\label{eq:boundphi_case2}
    \varphi_{t+1} \leq \beta_{t+1} + \gamma_{t+1} \leq C_\beta + \sum_{k=k_2}^t\frac{\|\nabla_y g(x_k, y_k)\|^2}{\beta_{k+1}} + C_\gamma,
\end{align}
\endgroup
where we take $\sum_{k=k_2}^t\frac{\|\nabla_y g(x_k, y_k)\|^2}{\beta_{k+1}} = 0$ for any $t < k_2$. 

\noindent
{\bf Case 3: $k_3$ exists but $k_2$ does not:} 
from the second line of \cref{eq:varphi_1}, for any $t \in (0,T)$, we have 
\begingroup
\allowdisplaybreaks
\begin{align}\label{eq:boundphi_case3}
    \varphi_{t+1} \overset{(a)}{\leq}& \varphi_t + \frac{\|\nabla_y g(x_t, y_t)\|^2}{\beta_{t+1} + \beta_t} + \frac{\|\nabla_v R(x_t, y_t, v_t)\|^2}{\varphi_{t+1}} \nonumber \\
    \leq& \varphi_{k_3} + \sum_{k=k_3}^t\frac{\|\nabla_y g(x_k, y_k)\|^2}{\beta_{k+1} + \beta_k} + \sum_{k=k_3}^t\frac{\|\nabla_v R(x_k, y_k, v_k)\|^2}{\varphi_{k+1}} \nonumber \\
    \leq& \beta_{k_3} + \gamma_{k_3} + \sum_{k=k_3}^t\frac{\|\nabla_y g(x_k, y_k)\|^2}{\beta_{k+1} + \beta_k} + \sum_{k=k_3}^t\frac{\|\nabla_v R(x_k, y_k, v_k)\|^2}{\varphi_{k+1}} \nonumber \\
    \overset{(b)}{=}& \beta_{t+1} + \gamma_{k_3} + \sum_{k=k_3}^t\frac{\|\nabla_v R(x_k, y_k, v_k)\|^2}{\varphi_{k+1}} \nonumber \\
    \leq& \beta_{t+1} + C_\gamma + \sum_{k=k_3}^t\frac{\|\nabla_v R(x_k, y_k, v_k)\|^2}{\varphi_{k+1}} \nonumber \\
    \leq& C_\beta + C_\gamma + \sum_{k=k_3}^t\frac{\|\nabla_v R(x_k, y_k, v_k)\|^2}{\varphi_{k+1}},
\end{align}
\endgroup
where (a) uses the second line of \cref{eq:varphi_1}; and we take $\sum_{k=k_3}^t \frac{\|\nabla_v R(x_k, y_k, v_k)\|^2}{\varphi_{k+1}} = 0$ for any $t < k_3$; (b) uses the first line of \cref{eq:boundy4}. 

\noindent
{\bf Case 4: both $k_2$ and $k_3$ exist:} from the third line of \cref{eq:boundphi_case3}, for any $t \in (0,T)$, we have
\begingroup
\allowdisplaybreaks
\begin{align}\label{eq:boundphi_case4}
    \varphi_{t+1} \leq& \beta_{k_3} + \gamma_{k_3} + \sum_{k=k_3}^t\frac{\|\nabla_y g(x_k, y_k)\|^2}{\beta_{k+1}} + \sum_{k=k_3}^t\frac{\|\nabla_v R(x_k, y_k, v_k)\|^2}{\varphi_{k+1}} \nonumber \\
    \overset{(a)}{\leq}& \beta_{k_2} + \sum_{k=k_2}^{k_3-1}\frac{\|\nabla_y g(x_k, y_k)\|^2}{\beta_{k+1}} + C_\gamma + \sum_{k=k_3}^t\frac{\|\nabla_y g(x_k, y_k)\|^2}{\beta_{k+1}} + \sum_{k=k_3}^t\frac{\|\nabla_v R(x_k, y_k, v_k)\|^2}{\varphi_{k+1}} \nonumber \\
    =& C_\beta + C_\gamma + \sum_{k=k_2}^t\frac{\|\nabla_y g(x_k, y_k)\|^2}{\beta_{k+1}} + \sum_{k=k_3}^t\frac{\|\nabla_v R(x_k, y_k, v_k)\|^2}{\varphi_{k+1}} ,
\end{align}
\endgroup
where (a) uses the third line of \cref{eq:boundy4}; 
and we take $\sum_{k=k_2}^{k_3-1}\frac{\|\nabla_y g(x_k, y_k)\|^2}{\beta_{k+1}}=0$ when $k_2 \geq k_3$, and $\sum_{k=k_2}^t\frac{\|\nabla_y g(x_k, y_k)\|^2}{\beta_{k+1}} = 0$ for any $t < k_2$ and $\sum_{k=k_3}^t \frac{\|\nabla_v R(x_k, y_k, v_k)\|^2}{\gamma_{k+1}} = 0$ for any $t < k_3$. It is easy to see that the upper bound of $\varphi_{t+1}$ in \cref{eq:boundphi_case4} is the largest among all cases. Thus, in the remaining proof, we only explore the upper bound of $\varphi_{t}$ in {\bf Case 4}. 

\noindent
{To further explore the bound of $\varphi_t$, we need to use some auxiliary results and bounds. So we split them into three parts as follows.}

\noindent
{\bf Part I: an auxiliary bound of $\sum\frac{\|\bar{\nabla}f(x_k, y_k, v_k)\|^2}{\alpha_{k+1}^2}$.}

\noindent
To further explore {\bf Case 4}, we begin with a common term $\sum_{k=k_0}^t\frac{\|\bar{\nabla}f(x_k, y_k, v_k)\|^2}{\alpha_{k+1}^2}$ for any $k_0 \leq t$. Recall in \Cref{lm:vk}, we have
\begin{align}\nonumber
    \|v_{k}\| \leq \frac{\sqrt{2}}{\mu}{\varphi_{k+1}} + \frac{\sqrt{2}C_{f_y}}{\mu}\sqrt{k} =: \bar{a}{\varphi_{k+1}} + \bar{b}\sqrt{k},
\end{align}
where $\bar{a}$ and $\bar{b}$ refer to \cref{def:barab_a0b0}. 
According to \Cref{lm:log}, since $\alpha_0 \geq 1$, for any integer $t > 0$, we have 
\begingroup
\allowdisplaybreaks
\begin{align}\label{eq:boundf}
    \sum_{k=k_0}^t\frac{\|\bar{\nabla}f(x_k, y_k, v_k)\|^2}{\alpha_{k+1}^2} 
    \leq& \sum_{k=0}^t\frac{\|\bar{\nabla}f(x_k, y_k, v_k)\|^2}{\alpha_{k+1}^2} \nonumber \\
    \leq& \log\bigg(\sum_{k=0}^t \|\bar{\nabla}f(x_k, y_k, v_k)\|^2 + \alpha_0^2\bigg) +1\nonumber \\
    \overset{(a)}{\leq}& \log\bigg(\sum_{k=0}^t \big(C_{g_{xy}}\bar{a}{\varphi_{k+1}} + C_{g_{xy}}\bar{b}\sqrt{k}+C_{f_x}\big)^2 + \alpha_0^2\bigg) +1\nonumber \\
    \leq & \log\bigg(\Big(\sum_{k=0}^t C_{g_{xy}}\bar{a}{\varphi_{k+1}} + C_{g_{xy}}\bar{b}\sqrt{k}+C_{f_x} + \alpha_0\Big)^2\bigg) +1\nonumber \\
    =& 2\log\bigg(\sum_{k=0}^t C_{g_{xy}}\bar{a}{\varphi_{k+1}} + C_{g_{xy}}\bar{b}\sqrt{k}+C_{f_x} + \alpha_0\bigg) +1 \nonumber \\
    \leq & 2\log\bigg((t+1)(C_{g_{xy}}\bar{a}{\varphi_{t+1}} + C_{g_{xy}}\bar{b}\sqrt{t}+C_{f_x} + \alpha_0)\bigg) +1 \nonumber \\
    = & 2\log(t+1) + 2\log\big(C_{g_{xy}}\bar{a}{\varphi_{t+1}} + C_{g_{xy}}\bar{b}\sqrt{t}+C_{f_x} + \alpha_0\big) +1 \nonumber \\
    \leq & 2\log(t+1) + 2\log\big((C_{g_{xy}}\bar{a}{\varphi_{t+1}} + C_{g_{xy}}\bar{b}+C_{f_x} + \alpha_0)\sqrt{t}\big) +1 \nonumber \\
    \leq & 3\log(t+1) + 2\log(C_{g_{xy}}\bar{a}{\varphi_{t+1}} + C_{g_{xy}}\bar{b}+C_{f_x} + \alpha_0) +1,
\end{align}
\endgroup
where (a) follows from \Cref{as:grad} and \Cref{lm:vk}. Therefore, we obtain the upper bound of {\small$\sum_{k=k_0}^t\frac{\|\bar{\nabla}f(x_k, y_k, v_k)\|^2}{\alpha_{k+1}^2}$ for any $k_0 \leq t$} in \cref{eq:boundf}. {\bf Part I} is completed. 

\noindent
{\bf Part II: a more general bound of $\sum\frac{\|\nabla_y g(x_k, y_k)\|^2}{\beta_{k+1}}$.}

\noindent
In \Cref{lm:sumG}, we show the bound of $\sum_{k=k_2}^t\frac{\|\nabla_y g(x_k, y_k)\|^2}{\beta_{k+1}}$ when $k_2$ exists. In {\bf Part II}, we further provide a rough bound of $\sum_{k=\Tilde{k}}^t\frac{\|\nabla_y g(x_k, y_k)\|^2}{\beta_{k+1}}$ for any potential $\Tilde{k} \leq T$. 
Firstly, if $\Tilde{k} \geq k_2$, it is easy to have 
\begin{align}
    \sum_{k=\Tilde{k}}^t \frac{\|\nabla_y g(x_k, y_k)\|^2}{\beta_{k+1}} \leq \sum_{k=k_2}^t \frac{\|\nabla_y g(x_k, y_k)\|^2}{\beta_{k+1}}; \nonumber
\end{align}
secondly, if $\Tilde{k} < k_2$, we have 
\begingroup
\allowdisplaybreaks
\begin{align}
    \sum_{k=\Tilde{k}}^t \frac{\|\nabla_y g(x_k, y_k)\|^2}{\beta_{k+1}} 
    \leq &\sum_{k=\Tilde{k}}^{k_2-1} \frac{\|\nabla_y g(x_k, y_k)\|^2}{\beta_{k+1}} + \sum_{k=k_2}^t \frac{\|\nabla_y g(x_k, y_k)\|^2}{\beta_{k+1}} \nonumber \\
    \leq & \frac{\sum_{k=\Tilde{k}}^{k_2-1}\|\nabla_y g(x_k, y_k)\|^2}{\beta_0} + \sum_{k=k_2}^t \frac{\|\nabla_y g(x_k, y_k)\|^2}{\beta_{k+1}} \nonumber \\
    \leq & \frac{\beta_{k_2}^2 - \beta_{\Tilde{k}}^2}{\beta_0} + \sum_{k=k_2}^t \frac{\|\nabla_y g(x_k, y_k)\|^2}{\beta_{k+1}} \nonumber \\
    \leq & \frac{C_\beta^2 - \beta_{0}^2}{\beta_{0}} + \sum_{k=k_2}^t \frac{\|\nabla_y g(x_k, y_k)\|^2}{\beta_{k+1}} \nonumber \\
    = & \frac{C_\beta^2}{\beta_{0}}-\beta_{0} + \sum_{k=k_2}^t \frac{\|\nabla_y g(x_k, y_k)\|^2}{\beta_{k+1}}. \nonumber
\end{align}
\endgroup
Combining these two situations, since $C_\beta \geq \beta_0$, for any $\Tilde{k} \leq t$, we have 
\begingroup
\allowdisplaybreaks
{
\begin{align}\label{eq:boundgk0}
    \sum_{k=\Tilde{k}}^t \frac{\|\nabla_y g(x_k, y_k)\|^2}{\beta_{k+1}} 
    \leq& \frac{C_\beta^2}{\beta_{0}}-\beta_{0} + \sum_{k=k_2}^t \frac{\|\nabla_y g(x_k, y_k)\|^2}{\beta_{k+1}} \nonumber \\
    \overset{(a)}{\leq}& \frac{C_\beta^2}{\beta_{0}}-\beta_{0} + \frac{(\mu+L_{g,1})C_\beta^2}{\mu^2} + \frac{(\mu+L_{g,1})^2L_y^2}{\mu L_{g,1}\varphi_{0}} \nonumber \\
    & + \frac{(\mu+L_{g,1})^2L_y^2}{\mu L_{g,1}}\sum_{k=k_2}^t\frac{\|\bar{\nabla}f(x_k, y_k, v_k)\|^2}{\alpha_{k+1}^2\varphi_{k+1}},
\end{align}}
where (a) uses \Cref{lm:sumG}. Thus, {\bf Part II}  is completed. 
\endgroup

\noindent
{\bf Part III: the bound of $\varphi_t$ in Case 4.} 

\noindent
Here, we explore the upper bound of $\varphi_t$ in {\bf Case 4}. 
Recalling \cref{eq:boundphi_case4}, we have 
\begin{align}
    \varphi_{t+1} \leq & C_\beta + C_\gamma + \sum_{k=k_2}^t\frac{\|\nabla_y g(x_k, y_k)\|^2}{\beta_{k+1}} + \sum_{k=k_3}^t\frac{\|\nabla_v R(x_k, y_k, v_k)\|^2}{\varphi_{k+1}} = C_\beta + C_\gamma = C_\varphi, \nonumber
\end{align}
for $t \leq k_{\text{min}}:=\min\{k_2,k_3\}$. For $t > k_{\text{min}}$, we have
\begingroup
\allowdisplaybreaks
{
\small
\begin{align}
    \varphi_{t+1} \leq & C_\beta + C_\gamma + \sum_{k=k_2}^t\frac{\|\nabla_y g(x_k, y_k)\|^2}{\beta_{k+1}} + \sum_{k=k_3}^t\frac{\|\nabla_v R(x_k, y_k, v_k)\|^2}{\varphi_{k+1}} \nonumber \\
    \overset{(a)}{\leq}& C_\beta + C_\gamma + \sum_{k=k_2}^t\frac{\|\nabla_y g(x_k, y_k)\|^2}{\beta_{k+1}} \nonumber \\
    & + \frac{4(\mu+C_{g_{yy}})C_\beta^2}{\mu^4} \bigg(\frac{L_{g,2}C_{f_y}}{\mu}+L_{f,1}\bigg)^2 + \frac{4(\mu+C_{g_{yy}})C_\gamma^2}{\mu^2} \nonumber \\
    &+ \frac{4(\mu+C_{g_{yy}})(\mu+L_{g,1})L_y^2}{\mu^3 L_{g,1}\varphi_0} \bigg(\frac{L_{g,2}C_{f_y}}{\mu}+L_{f,1}\bigg)^2\sum_{k=k_2-1}^{k_3-2}\frac{\|\Bar{\nabla} f(x_k, y_k, v_k)\|^2}{\alpha_{k+1}^2} \nonumber \\
    &+ \frac{4(\mu+C_{g_{yy}})^2L_v^2}{\mu C_{g_{yy}}C_\gamma}\sum_{k=k_3-1}^t \frac{\|\bar{\nabla}f(x_k, y_k, v_k)\|^2}{\alpha_{k+1}^2} \nonumber \\
    &+ \bigg(\frac{4(\mu+C_{g_{yy}})^2}{\mu C_{g_{yy}}}+8\bigg)\bigg(\frac{L_{g,2}C_{f_y}}{\mu}+L_{f,1}\bigg)^2\frac{1}{\mu^2}\sum_{k=k_3-1}^t\frac{\|\nabla_y g(x_k, y_k)\|^2}{\beta_{k+1}} \nonumber \\
    \overset{(b)}{\leq}& \bigg[\Big(\frac{4(\mu+C_{g_{yy}})^2}{\mu C_{g_{yy}}}+8\Big)\Big(\frac{L_{g,2}C_{f_y}}{\mu}+L_{f,1}\Big)^2\frac{1}{\mu^2}+1\bigg]\sum_{k=k_2}^t\frac{\|\nabla_y g(x_k, y_k)\|^2}{\beta_{k+1}} \nonumber \\
    & + C_\beta + C_\gamma + \frac{4(\mu+C_{g_{yy}})C_\beta^2}{\mu^4} \bigg(\frac{L_{g,2}C_{f_y}}{\mu}+L_{f,1}\bigg)^2 + \frac{4(\mu+C_{g_{yy}})C_\gamma^2}{\mu^2} \nonumber \\
    & + \Big(\frac{4(\mu+C_{g_{yy}})^2}{\mu C_{g_{yy}}}+8\Big)\Big(\frac{L_{g,2}C_{f_y}}{\mu}+L_{f,1}\Big)^2\frac{1}{\mu^2}\Big(\frac{C_\beta^2}{\beta_{0}}-\beta_{0}\Big) \nonumber \\
    & + \frac{4(\mu+C_{g_{yy}})(\mu+L_{g,1})L_y^2}{\mu^3 L_{g,1}\varphi_0} \bigg(\frac{L_{g,2}C_{f_y}}{\mu}+L_{f,1}\bigg)^2\sum_{k=k_2-1}^{t}\frac{\|\Bar{\nabla} f(x_k, y_k, v_k)\|^2}{\alpha_{k+1}^2} \nonumber \\
    & + \frac{4(\mu+C_{g_{yy}})^2L_v^2}{\mu C_{g_{yy}}C_\gamma}\sum_{k=k_3-1}^t \frac{\|\bar{\nabla}f(x_k, y_k, v_k)\|^2}{\alpha_{k+1}^2} \nonumber \\
    \overset{(c)}{\leq}& \bigg(\Big(\frac{4(\mu+C_{g_{yy}})^2}{\mu C_{g_{yy}}}+8\Big)\Big(\frac{L_{g,2}C_{f_y}}{\mu}+L_{f,1}\Big)^2\frac{1}{\mu^2}+1\bigg)\frac{(\mu+L_{g,1})^2L_y^2}{\mu L_{g,1}}\sum_{k=k_2}^t\frac{\|\bar{\nabla}f(x_k, y_k, v_k)\|^2}{\alpha_{k+1}^2\varphi_{k+1}}  \nonumber \\
    & + \frac{4(\mu+C_{g_{yy}})(\mu+L_{g,1})L_y^2}{\mu^3 L_{g,1}\varphi_0} \bigg(\frac{L_{g,2}C_{f_y}}{\mu}+L_{f,1}\bigg)^2\sum_{k=k_2-1}^{t}\frac{\|\Bar{\nabla} f(x_k, y_k, v_k)\|^2}{\alpha_{k+1}^2} \nonumber \\
    & + \frac{4(\mu+C_{g_{yy}})^2L_v^2}{\mu C_{g_{yy}}C_\gamma}\sum_{k=k_3-1}^t \frac{\|\bar{\nabla}f(x_k, y_k, v_k)\|^2}{\alpha_{k+1}^2} \nonumber \\
    & + C_\beta + C_\gamma + \frac{4(\mu+C_{g_{yy}})C_\beta^2}{\mu^4} \bigg(\frac{L_{g,2}C_{f_y}}{\mu}+L_{f,1}\bigg)^2 + \frac{4(\mu+C_{g_{yy}})C_\gamma^2}{\mu^2} \nonumber \\
    & + \Big(\frac{4(\mu+C_{g_{yy}})^2}{\mu C_{g_{yy}}}+8\Big)\Big(\frac{L_{g,2}C_{f_y}}{\mu}+L_{f,1}\Big)^2\frac{1}{\mu^2}\Big(\frac{C_\beta^2}{\beta_{0}}-\beta_{0}\Big) \nonumber \\
    & + \bigg[\Big(\frac{4(\mu+C_{g_{yy}})^2}{\mu C_{g_{yy}}}+8\Big)\Big(\frac{L_{g,2}C_{f_y}}{\mu}+L_{f,1}\Big)^2\frac{1}{\mu^2}+1\bigg]\bigg(\frac{(\mu+L_{g,1})C_\beta^2}{\mu^2} + \frac{(\mu+L_{g,1})^2L_y^2}{\mu L_{g,1}\varphi_{0}}\bigg) \nonumber \\
    \leq & \bigg[\bigg(\Big(\frac{4(\mu+C_{g_{yy}})^2}{\mu C_{g_{yy}}}+8\Big)\Big(\frac{L_{g,2}C_{f_y}}{\mu}+L_{f,1}\Big)^2\frac{1}{\mu^2}+1\bigg)\frac{(\mu+L_{g,1})^2L_y^2}{\mu L_{g,1}C_\beta}  \nonumber \\
    & \quad + \frac{4(\mu+C_{g_{yy}})(\mu+L_{g,1})L_y^2}{\mu^3 L_{g,1}\varphi_0} \bigg(\frac{L_{g,2}C_{f_y}}{\mu}+L_{f,1}\bigg)^2\bigg]\sum_{k=k_2-1}^{t}\frac{\|\Bar{\nabla} f(x_k, y_k, v_k)\|^2}{\alpha_{k+1}^2} \nonumber \\
    & + \frac{4(\mu+C_{g_{yy}})^2L_v^2}{\mu C_{g_{yy}}\gamma_0}\sum_{k=k_3-1}^t \frac{\|\bar{\nabla}f(x_k, y_k, v_k)\|^2}{\alpha_{k+1}^2} \nonumber \\
    & + C_\beta + C_\gamma + \frac{4(\mu+C_{g_{yy}})C_\beta^2}{\mu^4} \bigg(\frac{L_{g,2}C_{f_y}}{\mu}+L_{f,1}\bigg)^2 + \frac{4(\mu+C_{g_{yy}})C_\gamma^2}{\mu^2} \nonumber \\
    & + \Big(\frac{4(\mu+C_{g_{yy}})^2}{\mu C_{g_{yy}}}+8\Big)\Big(\frac{L_{g,2}C_{f_y}}{\mu}+L_{f,1}\Big)^2\frac{1}{\mu^2}\Big(\frac{C_\beta^2}{\beta_{0}}-\beta_{0}\Big) \nonumber \\
    & + \bigg[\Big(\frac{4(\mu+C_{g_{yy}})^2}{\mu C_{g_{yy}}}+8\Big)\Big(\frac{L_{g,2}C_{f_y}}{\mu}+L_{f,1}\Big)^2\frac{1}{\mu^2}+1\bigg]\bigg(\frac{(\mu+L_{g,1})C_\beta^2}{\mu^2} + \frac{(\mu+L_{g,1})^2L_y^2}{\mu L_{g,1}\varphi_{0}}\bigg) \nonumber \\
    \overset{(d)}{=:}& a_0 \sum_{k=\min\{k_2-1, k_3-1\}}^t \frac{\|\bar{\nabla}f(x_k, y_k, v_k)\|^2}{\alpha_{k+1}^2} + b_0 \nonumber \\
    \leq & a_0 \sum_{k=\min\{k_2, k_3\}}^t \frac{\|\bar{\nabla}f(x_k, y_k, v_k)\|^2}{\alpha_{k+1}^2} + a_0 + b_0 \nonumber \\
    \overset{(e)}{\leq}& a_0 \bigg[3\log(t+1) + 2\log\bigg({\varphi_{t+1}} + \frac{C_{g_{xy}}\bar{b}+C_{f_x}+\alpha_0}{C_{g_{xy}}\bar{a}}\bigg) + 2\log(C_{g_{xy}}\bar{a}) +1\bigg] + a_0 + b_0,
\end{align}
}
\endgroup
where (a) uses \Cref{lm:sumR}; 
(b) uses the first line in \cref{eq:boundgk0} by replacing $\Tilde{k}$ with $k_3-1$; 
(c) results from \cref{eq:boundy5}; 
(d) refers to \cref{def:barab_a0b0}; 
(e) uses \cref{eq:boundf}. 
Since $\min\{k_2, k_3\} \leq T$, 
we have $\varphi_{t+1} \geq \min\{C_\beta, C_\gamma\} \geq \max\{64a_0^2, 1\}$, which indicate that 

\noindent
(i) if $8a_0 \leq 1$, we have
\begingroup
\allowdisplaybreaks
\begin{align}
    4a_0\log(\varphi_{t+1}) \leq \frac{\log(\varphi_{t+1})}{2} \leq \frac{\varphi_{t+1}}{2} \leq \varphi_{t+1}; \nonumber
\end{align}
\endgroup
(ii) if $8a_0 > 1$, we have
\begingroup
\allowdisplaybreaks
\begin{align}
    \varphi_{t+1} - 4a_0\log(\varphi_{t+1}) = \varphi_{t+1} - 8a_0\log(\sqrt{\varphi_{t+1}}) \geq 8a_0\big(\sqrt{\varphi_{t+1}} - \log(\sqrt{\varphi_{t+1}})\big) \geq 0. \nonumber
\end{align}
\endgroup
Combining (i) and (ii), we have $4a_0\log(\varphi_{t+1}) \leq \varphi_{t+1}$. Then we obtain
{
\small
\begin{align}
    \varphi_{t+1} \leq& a_0 \bigg[3\log(t+1) + 2\log\bigg({\varphi_{t+1}} + \frac{C_{g_{xy}}\bar{b}+C_{f_x}+\alpha_0}{C_{g_{xy}}\bar{a}}\bigg) + 2\log(C_{g_{xy}}\bar{a}) +1\bigg] + a_0 + b_0 \nonumber \\
    \leq& a_0 \bigg[3\log(t+1) + 2\log({\varphi_{t+1}}) + 2\log\bigg(1 + \frac{C_{g_{xy}}\bar{b}+C_{f_x}+\alpha_0}{C_{g_{xy}}\bar{a}}\bigg) + 2\log(C_{g_{xy}}\bar{a}) +1\bigg] + a_0 + b_0 \nonumber \\
    \leq& \frac{1}{2}\varphi_{t+1} + a_0 \bigg[3\log(t+1) + 2\log\bigg(1 + \frac{C_{g_{xy}}\bar{b}+C_{f_x}+\alpha_0}{C_{g_{xy}}\bar{a}}\bigg) + 2\log(C_{g_{xy}}\bar{a}) +1\bigg] + a_0 + b_0, \nonumber
\end{align}
}
which indicates that 
\begin{align}\label{eq:boundphi}
    \varphi_{t+1} \leq&  6a_0\log(t+1) + 4a_0\log\Big(1 + \frac{C_{g_{xy}}\bar{b}+C_{f_x}+\alpha_0}{C_{g_{xy}}\bar{a}}\Big) + 4a_0\log(C_{g_{xy}}\bar{a}) +4a_0 + 2b_0 \nonumber \\
    \overset{(a)}{=:}& a_1 \log(t+1) + b_1, 
\end{align}
where (a) refers to \cref{def:a1b1}.
Thus, {\bf Part III} is completed and the proof of this lemma is completed.
\end{proof}

\begin{lemma}\label{lm:somebounds}
    Under Assumptions \ref{as:sc}, \ref{as:lip}, for any integer $k_0 \in [0,t)$, we have the upper bounds in terms of logarithmic functions as
\begingroup
\allowdisplaybreaks
\begin{align}
    \sum_{k=k_0}^t& \frac{\|\bar{\nabla}f(x_k, y_k, v_k)\|^2}{\alpha_{k+1}^2} \leq 5\log(t+1) + c_2, \nonumber \\
    \sum_{k=k_0}^t& \frac{\|\nabla_y g(x_k, y_k)\|^2}{\beta_{k+1}} \leq a_2 \log(t+1) + b_2, \nonumber \\
    \sum_{k=k_0}^t& \frac{\|\nabla_v R(x_k, y_k, v_k)\|^2}{\varphi_{k+1}} \leq a_3 \log(t+1) + b_3, \nonumber
\end{align}
\endgroup
where referring to \cref{def:a1b1}, \cref{def:barab_a0b0}, 
$c_2$, $a_2$, $b_2$, $a_3$, $b_3$ are defined as
\begingroup
\allowdisplaybreaks
{
\small
\begin{align}\label{def:c1_a2b2_a3b3}
    c_2 :=& 2\log\big(C_{g_{xy}}\bar{a}a_1 + C_{g_{xy}}\bar{a}b_1 + C_{g_{xy}}\bar{b}+C_{f_x} + \alpha_0\big) +1, \nonumber \\
    a_2:=& \frac{5(\mu+L_{g,1})^2L_y^2}{\mu L_{g,1}C_\beta}, \quad b_2 := \frac{(\mu+L_{g,1})^2L_y^2}{\mu L_{g,1}C_\beta}c_2 + \Big(\frac{C_\beta^2}{\beta_{0}}-\beta_{0}+\frac{(\mu+L_{g,1})C_\beta^2}{\mu^2} + \frac{(\mu+L_{g,1})^2L_y^2}{\mu L_{g,1}\varphi_{0}}\Big), \nonumber \\
    a_3 :=& \frac{20(\mu+C_{g_{yy}})(\mu+L_{g,1})L_y^2}{\mu^3 L_{g,1}\varphi_0} \bigg(\frac{L_{g,2}C_{f_y}}{\mu}+L_{f,1}\bigg)^2 + \frac{20(\mu+C_{g_{yy}})^2L_v^2}{\mu C_{g_{yy}}C_\gamma} \nonumber \\
    & + \big(\frac{4(\mu+C_{g_{yy}})^2}{\mu C_{g_{yy}}}+8\big)\big(\frac{L_{g,2}C_{f_y}}{\mu}+L_{f,1}\big)^2\frac{a_2}{\mu^2}, \nonumber \\
    b_3 :=& \frac{C_\gamma^2}{\gamma_0} - \gamma_0 + \frac{4(\mu+C_{g_{yy}})C_\beta^2}{\mu^4} \big(\frac{L_{g,2}C_{f_y}}{\mu}+L_{f,1}\big)^2 + \frac{4(\mu+C_{g_{yy}})C_\gamma^2}{\mu^2} \nonumber \\
    &+ \bigg(\frac{4(\mu+C_{g_{yy}})(\mu+L_{g,1})L_y^2}{\mu^3 L_{g,1}\varphi_0} \Big(\frac{L_{g,2}C_{f_y}}{\mu}+L_{f,1}\Big)^2 + \frac{4(\mu+C_{g_{yy}})^2L_v^2}{\mu C_{g_{yy}}C_\gamma}\bigg)c_2 \nonumber \\
    & + \big(\frac{4(\mu+C_{g_{yy}})^2}{\mu C_{g_{yy}}}+8\big)\big(\frac{L_{g,2}C_{f_y}}{\mu}+L_{f,1}\big)^2\frac{b_2}{\mu^2}. 
\end{align}}
\endgroup
\end{lemma}
\begin{proof}
Based on the logarithmic-function form bound in \Cref{lm:varphi}, we can further have the logarithmic-function form bounds of the components in \Cref{lm:objective} as the following 3 parts. 

\noindent
{\bf Part I: the bound of $\sum \frac{\|\bar{\nabla}f(x_k, y_k, v_k)\|^2}{\alpha_{k+1}^2}$ in terms of logarithmic function.} 

\noindent
Firstly, we bound $\sum_{k=k_0}^t \frac{\|\bar{\nabla}f(x_k, y_k, v_k)\|^2}{\alpha_{k+1}^2}$ for arbitrary $k_0 < t$. 
Back to \cref{eq:boundf}, by plugging in \cref{eq:boundphi}, we have
\begingroup
\allowdisplaybreaks
\begin{align}\label{eq:boundf2}
    \sum_{k=k_0}^t& \frac{\|\bar{\nabla}f(x_k, y_k, v_k)\|^2}{\alpha_{k+1}^2} \nonumber \\
    \leq& 3\log(t+1) + 2\log(C_{g_{xy}}\bar{a}\varphi_{t+1} + C_{g_{xy}}\bar{b}+C_{f_x} + \alpha_0) +1 \nonumber \\
    \overset{(a)}{\leq}& 3\log(t+1) + 2\log\big(C_{g_{xy}}\bar{a}a_1\log(t+1) + C_{g_{xy}}\bar{a}b_1 + C_{g_{xy}}\bar{b}+C_{f_x} + \alpha_0\big) +1 \nonumber \\
    \leq& 3\log(t+1) + 2\log\big(C_{g_{xy}}\bar{a}a_1(t+1) + C_{g_{xy}}\bar{a}b_1 + C_{g_{xy}}\bar{b}+C_{f_x} + \alpha_0\big) +1 \nonumber \\
    \leq& 3\log(t+1) + 2\log\big((C_{g_{xy}}\bar{a}a_1 + C_{g_{xy}}\bar{a}b_1 + C_{g_{xy}}\bar{b}+C_{f_x} + \alpha_0)(t+1)\big) +1 \nonumber \\
    \leq& 5\log(t+1) + 2\log\big(C_{g_{xy}}\bar{a}a_1 + C_{g_{xy}}\bar{a}b_1 + C_{g_{xy}}\bar{b}+C_{f_x} + \alpha_0\big) +1 \nonumber \\
    \overset{(b)}{=:}& 5\log(t+1) + c_2, 
\end{align}
\endgroup
where (a) results from \cref{eq:boundphi}; (b) refers to \cref{def:c1_a2b2_a3b3}.

\noindent
{\bf Part II: the bound of $\sum \frac{\|\nabla_y g(x_k, y_k)\|^2}{\beta_{k+1}}$ in terms of logarithmic function.} 

\noindent
Secondly, we bound $\sum_{k=k_0}^t \frac{\|\nabla_y g(x_k, y_k)\|^2}{\beta_{k+1}}$.
We split this part into two cases using \Cref{lm:bar}. 

\noindent
{\bf Case 1:} If $\beta_{t+1} \leq C_\beta$, we have
\begin{align}
    \sum_{k=k_0}^t \frac{\|\nabla_y g(x_k, y_k)\|^2}{\beta_{k+1}} \leq \frac{\sum_{k=k_0}^t \|\nabla_y g(x_k, y_k)\|^2}{\beta_0} \leq \frac{\beta_{t+1}^2 - \beta_{k_0}^2}{\beta_0} \leq \frac{C_\beta^2 - \beta_0^2}{\beta_0} = \frac{C_\beta^2}{\beta_0} - \beta_0 \leq b_2. \nonumber
\end{align}

\noindent
{\bf Case 2:} If $\beta_{t+1} > C_\beta$, we have $k_2 \leq t$, where $k_2$ refers to \Cref{lm:bar}. Then we can use \cref{eq:boundgk0}, which indicates  
\begingroup
\allowdisplaybreaks
{
\small
\begin{align}\label{eq:boundgfinal}
    \sum_{k=k_0}^t& \frac{\|\nabla_y g(x_k, y_k)\|^2}{\beta_{k+1}} \nonumber \\
    \leq& \bigg(\frac{C_\beta^2}{\beta_{0}}-\beta_{0}+\frac{(\mu+L_{g,1})C_\beta^2}{\mu^2} + \frac{(\mu+L_{g,1})^2L_y^2}{\mu L_{g,1}\varphi_{0}}\bigg) + \frac{(\mu+L_{g,1})^2L_y^2}{\mu L_{g,1}C_\beta}\sum_{k=k_2}^t\frac{\|\bar{\nabla}f(x_k, y_k, v_k)\|^2}{\alpha_{k+1}^2} \nonumber \\
    \leq & \frac{5(\mu+L_{g,1})^2L_y^2}{\mu L_{g,1}C_\beta} \log(t+1) + \frac{(\mu+L_{g,1})^2L_y^2}{\mu L_{g,1}C_\beta}c_2 + \bigg(\frac{C_\beta^2}{\beta_{0}}-\beta_{0}+\frac{(\mu+L_{g,1})C_\beta^2}{\mu^2} + \frac{(\mu+L_{g,1})^2L_y^2}{\mu L_{g,1}\varphi_{0}}\bigg) \nonumber \\
    \overset{(a)}{=:}& a_2 \log(t+1) + b_2, 
\end{align}}
\endgroup
where the second inequality uses \eqref{eq:boundf2}, and (a) refers to \cref{def:c1_a2b2_a3b3}. 
Since the upper bound of {\bf Case 2} is larger, we take \cref{eq:boundgfinal} as our final result. 

\noindent
{\bf Part III: the bound of $\sum \frac{\|\nabla_v R(x_k, y_k, v_k)\|^2}{\varphi_{k+1}}$ in terms of logarithmic function.} 

\noindent
Last, we bound $\sum_{k=k_0}^t \frac{\|\nabla_v R(x_k, y_k, v_k)\|^2}{\varphi_{k+1}}$.
We split this part into two cases using \Cref{lm:bar}. 

\noindent
{\bf Case 1:} If $\gamma_{t+1} \leq C_\gamma$, we have
\begin{align}
    \sum_{k=k_0}^t \frac{\|\nabla_v R(x_k, y_k, v_k)\|^2}{\varphi_{k+1}} \leq \frac{\sum_{k=k_0}^t\|\nabla_v R(x_k, y_k, v_k)\|^2}{\varphi_0} \leq \frac{C_\gamma^2 - \gamma_{0}^2}{\gamma_0} \leq \frac{C_\gamma^2}{\gamma_0} - \gamma_0 \leq b_3. \nonumber
\end{align}

\noindent
{\bf Case 2:} If $\gamma_{t+1} > C_\gamma$, we have $k_3 \leq t$, where $k_3$ refers to \Cref{lm:bar}. 
\begingroup
\allowdisplaybreaks
\begin{align}\label{eq:boundRfinal}
    \sum_{k=k_0}^t& \frac{\|\nabla_v R(x_k, y_k, v_k)\|^2}{\varphi_{k+1}} \nonumber \\
    \overset{(a)}{\leq} &\sum_{k=k_0}^{k_3-1} \frac{\|\nabla_v R(x_k, y_k, v_k)\|^2}{\varphi_{k+1}} + \sum_{k=k_3}^t \frac{\|\nabla_v R(x_k, y_k, v_k)\|^2}{\varphi_{k+1}} \nonumber \\
    \overset{(b)}{\leq} & \frac{C_\gamma^2}{\gamma_0} - \gamma_0 + \frac{4(\mu+C_{g_{yy}})C_\beta^2}{\mu^4} \bigg(\frac{L_{g,2}C_{f_y}}{\mu}+L_{f,1}\bigg)^2 + \frac{4(\mu+C_{g_{yy}})C_\gamma^2}{\mu^2} \nonumber \\
    &+ \frac{4(\mu+C_{g_{yy}})(\mu+L_{g,1})L_y^2}{\mu^3 L_{g,1}\varphi_0} \bigg(\frac{L_{g,2}C_{f_y}}{\mu}+L_{f,1}\bigg)^2\sum_{k=k_2-1}^{k_3-2}\frac{\|\Bar{\nabla} f(x_k, y_k, v_k)\|^2}{\alpha_{k+1}^2} \nonumber \\
    &+ \frac{4(\mu+C_{g_{yy}})^2L_v^2}{\mu C_{g_{yy}}C_\gamma}\sum_{k=k_3-1}^t \frac{\|\bar{\nabla}f(x_k, y_k, v_k)\|^2}{\alpha_{k+1}^2} \nonumber \\
    &+ \bigg(\frac{4(\mu+C_{g_{yy}})^2}{\mu C_{g_{yy}}}+8\bigg)\bigg(\frac{L_{g,2}C_{f_y}}{\mu}+L_{f,1}\bigg)^2\frac{1}{\mu^2}\sum_{k=k_3-1}^t\frac{\|\nabla_y g(x_k, y_k)\|^2}{\beta_{k+1}} \nonumber \\
    \overset{(c)}{\leq} & \frac{C_\gamma^2}{\gamma_0} - \gamma_0 + \frac{4(\mu+C_{g_{yy}})C_\beta^2}{\mu^4} \bigg(\frac{L_{g,2}C_{f_y}}{\mu}+L_{f,1}\bigg)^2 + \frac{4(\mu+C_{g_{yy}})C_\gamma^2}{\mu^2} \nonumber \\
    &+ \bigg(\frac{4(\mu+C_{g_{yy}})(\mu+L_{g,1})L_y^2}{\mu^3 L_{g,1}\varphi_0} \Big(\frac{L_{g,2}C_{f_y}}{\mu}+L_{f,1}\Big)^2 + \frac{4(\mu+C_{g_{yy}})^2L_v^2}{\mu C_{g_{yy}}C_\gamma}\bigg)\big(5\log(t+1)+c_2\big) \nonumber \\
    &+ \bigg(\frac{4(\mu+C_{g_{yy}})^2}{\mu C_{g_{yy}}}+8\bigg)\bigg(\frac{L_{g,2}C_{f_y}}{\mu}+L_{f,1}\bigg)^2\frac{1}{\mu^2}\big(a_2 \log(t+1) + b_2\big) \nonumber \\
    \overset{(d)}{=:} & a_3 \log(t+1) + b_3,
\end{align}
\endgroup
where (a) allows $\sum_{k=k_0}^{k_3-1} \frac{\|\nabla_v R(x_k, y_k, v_k)\|^2}{\varphi_{k+1}} = 0$ when $k_0\geq k_3$; (b) uses $C_\gamma \geq \gamma_0$ and \Cref{lm:sumR}; (c) follows from \cref{eq:boundf2} and \cref{eq:boundgfinal}; 
(d) refers to \cref{def:c1_a2b2_a3b3}. Since the upper bound of {\bf Case 2} is larger, we take \cref{eq:boundRfinal} as our final result. 

\noindent
Thus, the proof is complete. 
\end{proof}
\noindent
Next, we show the upper bound of $\alpha_t$. 
\begin{lemma}[The upper bound of $\alpha_{t}$]\label{lm:alpha}
Under Assumptions \ref{as:sc}, \ref{as:lip}, \ref{as:inf}, suppose the number of total iteration rounds in \Cref{alg:main} is $T$. 
If there exists $k_1 \leq T$ described in \Cref{lm:bar}, we have 
$$
\left\{
\begin{aligned}
\alpha_{t} \leq & C_\alpha, \quad & t \leq k_1; \\
\alpha_{t} \leq & C_\alpha + \Big(a_4 \log(t) + b_4 + {4}\big(\Phi(x_0) - \inf_x \Phi(x)\big)\Big)\varphi_{t},\quad & t \geq k_1, 
\end{aligned}
\right.
$$
where $a_4$, $b_4$ are defined as
\begingroup
\allowdisplaybreaks
\begin{align}\label{def:a4b4}
    a_4 :=& \frac{2\bar{L}^2a_2}{\mu^2C_\alpha}\Big[1 + \frac{2}{\mu^2}\big(\frac{L_{g,2}C_{f_y}}{\mu}+L_{f,1}\big)^2\Big] + \frac{4\bar{L}^2a_3}{\mu^2C_\alpha} \nonumber \\
    b_4 :=& \frac{2\bar{L}^2b_2}{\mu^2C_\alpha}\Big[1 + \frac{2}{\mu^2}\big(\frac{L_{g,2}C_{f_y}}{\mu}+L_{f,1}\big)^2\Big] + \frac{4\bar{L}^2b_3}{\mu^2C_\alpha} + \frac{2L_{\Phi}}{\varphi_{0}^2}\frac{C_\alpha^2}{\alpha_0^2},
\end{align}
\endgroup
and the upper bound of $\varphi_t := \max\{\beta_t, \gamma_t\}$ refers to \Cref{lm:varphi}. 
When such $k_1$ does not exist, we have $\alpha_t \leq C_\alpha$ for any $t \leq T$. 
\end{lemma}

\begin{proof}
According to \Cref{lm:bar}, the proof can be split into the following three cases. 

\noindent
{\bf Case 1:} if $\alpha_T \leq C_\alpha$, for any $t < T$, we have the upper bound of $\alpha_{t+1}$ as $\alpha_{t+1} \leq C_\alpha$. 

\noindent
{\bf Case 2:} if $\alpha_T > C_\alpha$, there exists $k_1 \leq T$ described in \Cref{lm:bar}. 
Then we have the upper bound of $\alpha_{t+1}$ as $\alpha_{t+1} \leq C_\alpha$ for any $t < k_1$. 

\noindent
{\bf Case 3:} in the remaining proof, we only consider and explore the case $k_1 \leq t \leq T$ when $\alpha_T > C_\alpha$.

\noindent
From \Cref{lm:objective}, for $k \geq k_1$, we have
\begingroup
\allowdisplaybreaks
\begin{align}
    \Phi(x_{k+1}) 
    \leq& \Phi(x_k) - \frac{1}{2\alpha_{k+1}\varphi_{k+1}}\|\nabla \Phi(x_k)\|^2 - \frac{1}{4\alpha_{k+1}\varphi_{k+1}}\|\bar{\nabla}f(x_k, y_k, v_k)\|^2 \nonumber \\
    & + \frac{\bar{L}^2}{2\mu^2}\bigg[1 + \frac{2}{\mu^2}\Big(\frac{L_{g,2}C_{f_y}}{\mu}+L_{f,1}\Big)^2\bigg]\frac{\big\|\nabla_y g(x_k,y_k)\big\|^2}{\alpha_{k+1}\varphi_{k+1}} + \frac{\bar{L}^2}{\mu^2}\frac{\big\|\nabla_v R(x_k,y_k,v_k)\|^2}{\alpha_{k+1}\varphi_{k+1}}, \nonumber
\end{align}
\endgroup
which indicates that 
\begin{align}
    &\frac{\|\bar{\nabla}f(x_k, y_k, v_k)\|^2}{\alpha_{k+1}\varphi_{k+1}} \leq 4\big(\Phi(x_k) - \Phi(x_{k+1})\big) \nonumber \\
    &\quad \quad \quad \quad \quad + \frac{2\bar{L}^2}{\mu^2}\bigg[1 + \frac{2}{\mu^2}\Big(\frac{L_{g,2}C_{f_y}}{\mu}+L_{f,1}\Big)^2\bigg]\frac{\big\|\nabla_y g(x_k,y_k)\big\|^2}{\alpha_{k+1}\varphi_{k+1}} + \frac{4\bar{L}^2}{\mu^2}\frac{\big\|\nabla_v R(x_k,y_k,v_k)\|^2}{\alpha_{k+1}\varphi_{k+1}}. \nonumber
\end{align}
By taking summation, we have 
\begingroup
\allowdisplaybreaks
\begin{align}\label{eq:boundffinal1}
    \sum_{k=k_1}^t & \frac{\|\bar{\nabla}f(x_k, y_k, v_k)\|^2}{\alpha_{k+1}\varphi_{k+1}} \nonumber \\
    \leq & 4\big(\Phi(x_{k_1}) - \inf_x \Phi(x)\big) + \frac{2\bar{L}^2}{\mu^2C_\alpha}\bigg[1 + \frac{2}{\mu^2}\Big(\frac{L_{g,2}C_{f_y}}{\mu}+L_{f,1}\Big)^2\bigg]\sum_{k=k_1}^t\frac{\big\|\nabla_y g(x_k,y_k)\big\|^2}{\varphi_{k+1}} \nonumber \\
    & + \frac{4\bar{L}^2}{\mu^2C_\alpha}\sum_{k=k_1}^t\frac{\big\|\nabla_v R(x_k,y_k,v_k)\|^2}{\varphi_{k+1}} .
\end{align}
\endgroup
For $\Phi(x_{k_1})$, by telescoping \cref{eq:objective_title1} in \Cref{lm:objective}, we get
\begingroup
\allowdisplaybreaks
\begin{align}\label{eq:boundPhik}
    \Phi(x_{k_1}) \leq& \Phi(x_{0}) + \frac{L_{\Phi}}{2}\sum_{k=0}^{k_1-1}\frac{\|\bar{\nabla}f(x_t, y_t, v_t)\|^2}{\alpha_{t+1}^2\varphi_{t+1}^2} \nonumber \\
    & + \frac{\bar{L}^2}{2\mu^2}\bigg[1 + \frac{2}{\mu^2}\Big(\frac{L_{g,2}C_{f_y}}{\mu}+L_{f,1}\Big)^2\bigg]\sum_{k=0}^{k_1-1}\frac{\big\|\nabla_y g(x_t,y_t)\big\|^2}{\alpha_{t+1}\varphi_{t+1}} \nonumber \\
    & + \frac{\bar{L}^2}{\mu^2}\sum_{k=0}^{k_1-1}\frac{\big\|\nabla_v R(x_t,y_t,v_t)\|^2}{\alpha_{t+1}\varphi_{t+1}}. 
\end{align}
\endgroup
By plugging \cref{eq:boundPhik} into \cref{eq:boundffinal1}, we have
\begingroup
\allowdisplaybreaks
\begin{align}
    \sum_{k=k_1}^t & \frac{\|\bar{\nabla}f(x_k, y_k, v_k)\|^2}{\alpha_{k+1}\varphi_{k+1}} \nonumber \\
    \leq & 4\big(\Phi(x_{0}) - \inf_x \Phi(x)\big) + \frac{2\bar{L}^2}{\mu^2C_\alpha}\bigg[1 + \frac{2}{\mu^2}\Big(\frac{L_{g,2}C_{f_y}}{\mu}+L_{f,1}\Big)^2\bigg]\sum_{k=0}^t\frac{\big\|\nabla_y g(x_k,y_k)\big\|^2}{\varphi_{k+1}} \nonumber \\
    & + \frac{4\bar{L}^2}{\mu^2C_\alpha}\sum_{k=0}^t\frac{\big\|\nabla_v R(x_k,y_k,v_k)\|^2}{\varphi_{k+1}} + \frac{2L_{\Phi}}{\varphi_{0}^2}\frac{C_\alpha^2}{\alpha_0^2}  \nonumber \\
    \leq & 4\big(\Phi(x_{0}) - \inf_x \Phi(x)\big) + \frac{2\bar{L}^2}{\mu^2C_\alpha}\bigg[1 + \frac{2}{\mu^2}\Big(\frac{L_{g,2}C_{f_y}}{\mu}+L_{f,1}\Big)^2\bigg]\sum_{k=0}^t\frac{\big\|\nabla_y g(x_k,y_k)\big\|^2}{\beta_{k+1}} \nonumber \\
    & + \frac{4\bar{L}^2}{\mu^2C_\alpha}\sum_{k=0}^t\frac{\big\|\nabla_v R(x_k,y_k,v_k)\|^2}{\varphi_{k+1}} + \frac{2L_{\Phi}}{\varphi_{0}^2}\frac{C_\alpha^2}{\alpha_0^2}  \nonumber \\
    \overset{(a)}{\leq}& 4\big(\Phi(x_{0}) - \inf_x \Phi(x)\big) + \frac{2\bar{L}^2}{\mu^2C_\alpha}\bigg[1 + \frac{2}{\mu^2}\Big(\frac{L_{g,2}C_{f_y}}{\mu}+L_{f,1}\Big)^2\bigg]\big(a_2\log(t+1) + b_2\big) \nonumber \\
    & + \frac{4\bar{L}^2}{\mu^2C_\alpha}\big(a_3\log(t+1) + b_3\big) + \frac{2L_{\Phi}}{\varphi_{0}^2}\frac{C_\alpha^2}{\alpha_0^2}  \nonumber \\
    \overset{(b)}{=:}& a_4 \log(t+1) + b_4 + 4\big(\Phi(x_{0}) - \inf_x \Phi(x)\big),
\end{align}
\endgroup
where 
(a) plugs in \cref{eq:boundgfinal} and \cref{eq:boundRfinal}; (b) refers to \cref{def:a4b4}.  
This immediately implies
\begingroup
\allowdisplaybreaks
\begin{align}
    \sum_{k=k_1}^t \frac{\|\bar{\nabla}f(x_k, y_k, v_k)\|^2}{\alpha_{k+1}} \leq \Big(a_4 \log(t+1) + b_4 + 4\big(\Phi(x_0) - \inf_x \Phi(x)\big)\Big)\varphi_{t+1}. 
\end{align}
\endgroup
Similarly, we can have the upper bound of $\alpha_{t+1}$ as
\begingroup
\allowdisplaybreaks
\begin{align}\label{eq:boundalphafinal}
    \alpha_{t+1} \leq& \alpha_{k_1} + \sum_{k=k_1}^t \frac{\|\bar{\nabla}f(x_k, y_k, v_k)\|^2}{\alpha_{k+1}} \nonumber \\
    \leq& C_\alpha + \Big(a_4 \log(t+1) + b_4 + 4\big(\Phi(x_0) - \inf_x \Phi(x)\big)\Big)\varphi_{t+1}.
\end{align}
\endgroup
Then the upper bound of $\alpha_{t+1}$ is proved. 
\end{proof}

\subsection{Proof of Theorem \ref{thm:main}}

Here we still assume the total iteration rounds of Algorithm \ref{alg:main} is $T$. 
According to \Cref{lm:bar}, the proof can be split into the following two cases. 

\noindent
{\bf Case 1:} If $\alpha_T \leq C_\alpha$, then by \Cref{lm:objective} and \Cref{lm:alpha}, we have 
\begingroup
\allowdisplaybreaks
\begin{align}
    \frac{\|\nabla \Phi(x_t)\|^2}{\alpha_{t+1}\varphi_{t+1}} \leq& 2\big(\Phi(x_t) - \Phi(x_{t+1})\big) + \frac{L_{\Phi}}{\alpha_{t+1}^2\varphi_{t+1}^2}\|\bar{\nabla}f(x_t, y_t, v_t)\|^2\nonumber \\
    & + \frac{\bar{L}^2}{\mu^2}\bigg[1 + \frac{2}{\mu^2}\Big(\frac{L_{g,2}C_{f_y}}{\mu}+L_{f,1}\Big)^2\bigg]\frac{\big\|\nabla_y g(x_t,y_t)\big\|^2}{\alpha_{t+1}\varphi_{t+1}} + \frac{2\bar{L}^2}{\mu^2}\frac{\big\|\nabla_v R(x_t,y_t,v_t)\|^2}{\alpha_{t+1}\varphi_{t+1}}, \nonumber
\end{align}
\endgroup

\noindent
By taking the average, we have 
\begingroup
\allowdisplaybreaks
\begin{align}
    \frac{1}{T}\sum_{t=0}^{T-1} \frac{\|\nabla \Phi(x_t)\|^2}{\alpha_{t+1}\varphi_{t+1}} 
    \leq & \frac{2}{T}\big(\Phi(x_0) - \Phi(x_{T})\big) + \frac{L_{\Phi}}{\alpha_0^2\varphi_0^2}\frac{1}{T}\sum_{t=0}^{T-1}\|\bar{\nabla}f(x_t, y_t, v_t)\|^2\nonumber \\
    & + \frac{\bar{L}^2}{\mu^2}\bigg[1 + \frac{2}{\mu^2}\Big(\frac{L_{g,2}C_{f_y}}{\mu}+L_{f,1}\Big)^2\bigg]\frac{1}{T}\sum_{t=0}^{T-1}\frac{\big\|\nabla_y g(x_t,y_t)\big\|^2}{\alpha_{t+1}\varphi_{t+1}} \nonumber \\
    & + \frac{2\bar{L}^2}{\mu^2}\frac{1}{T}\sum_{t=0}^{T-1}\frac{\big\|\nabla_v R(x_t,y_t,v_t)\|^2}{\alpha_{t+1}\varphi_{t+1}} \nonumber \\
    \leq & \frac{2}{T}\big(\Phi(x_0) - \Phi(x_{T})\big) + \frac{L_{\Phi}C_\alpha^2}{T\alpha_0^2\varphi_0^2}\nonumber \\
    & + \frac{\bar{L}^2}{\mu^2\alpha_0T}\bigg[1 + \frac{2}{\mu^2}\Big(\frac{L_{g,2}C_{f_y}}{\mu}+L_{f,1}\Big)^2\bigg]\sum_{t=0}^{T-1}\frac{\big\|\nabla_y g(x_t,y_t)\big\|^2}{\beta_{t+1}} \nonumber \\
    & + \frac{2\bar{L}^2}{\mu^2\alpha_0T}\sum_{t=0}^{T-1}\frac{\big\|\nabla_v R(x_t,y_t,v_t)\|^2}{\varphi_{t+1}} \nonumber \\
    \overset{(a)}{\leq} & \frac{2}{T}\big(\Phi(x_0) - \inf_x \Phi(x)\big) + \frac{L_{\Phi}C_\alpha^2}{T\alpha_0^2\varphi_0^2} \nonumber \\
    & + \frac{\bar{L}^2}{\mu^2\alpha_0T}\bigg[1 + \frac{2}{\mu^2}\Big(\frac{L_{g,2}C_{f_y}}{\mu}+L_{f,1}\Big)^2\bigg]\big(a_2\log(T) + b_2\big) \nonumber \\
    & + \frac{2\bar{L}^2}{\mu^2\alpha_0T}\big(a_3\log(T) + b_3\big) \nonumber \\
    =& \frac{1}{2T}\Big(a_4 \log(T) + b_4 + 4\big(\Phi(x_0) - \inf_x \Phi(x)\big)\Big), 
\end{align}
\endgroup
where (a) uses \Cref{lm:somebounds} with $k_0 = 0$.

\noindent
{\bf Case 2: }
If $\alpha_T > C_\alpha$, by \Cref{lm:bar}, 
there exists $ k_1 \leq T_0$ such that $\alpha_{k_1} \leq C_\alpha$, $\alpha_{k_1+1} > C_\alpha$. 

\noindent
Then for $t < k_1$ when $\alpha_T > C_\alpha$, from \Cref{lm:objective}, we have
\begin{align}
    \frac{\|\nabla \Phi(x_t)\|^2}{\alpha_{t+1}\varphi_{t+1}} \leq& 2\big(\Phi(x_t) - \Phi(x_{t+1})\big) + \frac{L_{\Phi}}{\alpha_{t+1}^2\varphi_{t+1}^2}\|\bar{\nabla}f(x_t, y_t, v_t)\|^2\nonumber \\
    & + \frac{\bar{L}^2}{\mu^2}\bigg[1 + \frac{2}{\mu^2}\Big(\frac{L_{g,2}C_{f_y}}{\mu}+L_{f,1}\Big)^2\bigg]\frac{\big\|\nabla_y g(x_t,y_t)\big\|^2}{\alpha_{t+1}\varphi_{t+1}} + \frac{2\bar{L}^2}{\mu^2}\frac{\big\|\nabla_v R(x_t,y_t,v_t)\|^2}{\alpha_{t+1}\varphi_{t+1}}.\nonumber
\end{align}
For $t \geq k_1$ when $\alpha_T > C_\alpha$, from \Cref{lm:objective}, we have
\begin{align}
    \frac{\|\nabla \Phi(x_t)\|^2}{\alpha_{t+1}\varphi_{t+1}} \leq& 2\big(\Phi(x_t) - \Phi(x_{t+1})\big)\nonumber \\
    & + \frac{\bar{L}^2}{\mu^2}\bigg[1 + \frac{2}{\mu^2}\Big(\frac{L_{g,2}C_{f_y}}{\mu}+L_{f,1}\Big)^2\bigg]\frac{\big\|\nabla_y g(x_t,y_t)\big\|^2}{\alpha_{t+1}\varphi_{t+1}} + \frac{2\bar{L}^2}{\mu^2}\frac{\big\|\nabla_v R(x_t,y_t,v_t)\|^2}{\alpha_{t+1}\varphi_{t+1}}. \nonumber
\end{align}
By taking the average, we can merge $t < k_1$ and $t \geq k_1$ as 
\begingroup
\allowdisplaybreaks
\begin{align}\label{eq:boundPhi1}
    \frac{1}{T}\sum_{t=0}^{T-1} \frac{\|\nabla \Phi(x_t)\|^2}{\alpha_{t+1}\varphi_{t+1}} 
    =& \frac{1}{T}\sum_{t=0}^{k_1-1} \frac{\|\nabla \Phi(x_t)\|^2}{\alpha_{t+1}\varphi_{t+1}} + \frac{1}{T}\sum_{t=k_1}^{T-1} \frac{\|\nabla \Phi(x_t)\|^2}{\alpha_{t+1}\varphi_{t+1}} \nonumber \\
    \leq & \frac{2}{T}\big(\Phi(x_0) - \Phi(x_{k_1})\big) + \frac{L_{\Phi}}{\alpha_0^2\varphi_0^2}\frac{1}{T}\sum_{t=0}^{k_1-1}\|\bar{\nabla}f(x_t, y_t, v_t)\|^2\nonumber \\
    & + \frac{\bar{L}^2}{\mu^2}\bigg[1 + \frac{2}{\mu^2}\Big(\frac{L_{g,2}C_{f_y}}{\mu}+L_{f,1}\Big)^2\bigg]\frac{1}{T}\sum_{t=0}^{k_1-1}\frac{\big\|\nabla_y g(x_t,y_t)\big\|^2}{\alpha_{t+1}\varphi_{t+1}} \nonumber \\
    & + \frac{2\bar{L}^2}{\mu^2}\frac{1}{T}\sum_{t=0}^{k_1-1}\frac{\big\|\nabla_v R(x_t,y_t,v_t)\|^2}{\alpha_{t+1}\varphi_{t+1}} \nonumber \\
    & + \frac{2}{T}\big(\Phi(x_{k_1}) - \Phi(x_{T})\big) \nonumber \\
    & + \frac{\bar{L}^2}{\mu^2}\bigg[1 + \frac{2}{\mu^2}\Big(\frac{L_{g,2}C_{f_y}}{\mu}+L_{f,1}\Big)^2\bigg]\frac{1}{T}\sum_{t={k_1}}^{T-1}\frac{\big\|\nabla_y g(x_t,y_t)\big\|^2}{\alpha_{t+1}\varphi_{t+1}} \nonumber \\
    & + \frac{2\bar{L}^2}{\mu^2}\frac{1}{T}\sum_{t={k_1}}^{T-1}\frac{\big\|\nabla_v R(x_t,y_t,v_t)\|^2}{\alpha_{t+1}\varphi_{t+1}} \nonumber \\
    \leq & \frac{2}{T}\big(\Phi(x_0) - \inf_x \Phi(x)\big) + \frac{L_{\Phi}}{\alpha_0^2\varphi_0^2}\frac{1}{T}\sum_{t=0}^{k_1-1}\|\bar{\nabla}f(x_t, y_t, v_t)\|^2\nonumber \\
    & + \frac{\bar{L}^2}{\mu^2}\bigg[1 + \frac{2}{\mu^2}\Big(\frac{L_{g,2}C_{f_y}}{\mu}+L_{f,1}\Big)^2\bigg]\frac{1}{T}\sum_{t=0}^{T-1}\frac{\big\|\nabla_y g(x_t,y_t)\big\|^2}{\alpha_{t+1}\varphi_{t+1}} \nonumber \\
    & + \frac{2\bar{L}^2}{\mu^2}\frac{1}{T}\sum_{t=0}^{T-1}\frac{\big\|\nabla_v R(x_t,y_t,v_t)\|^2}{\alpha_{t+1}\varphi_{t+1}} \nonumber \\
    \leq & \frac{2}{T}\big(\Phi(x_0) - \inf_x \Phi(x)\big) + \frac{L_{\Phi}}{\alpha_0^2\varphi_0^2}\frac{1}{T}\sum_{t=0}^{k_1-1}\|\bar{\nabla}f(x_t, y_t, v_t)\|^2\nonumber \\
    & + \frac{\bar{L}^2}{\mu^2\alpha_0T}\bigg[1 + \frac{2}{\mu^2}\Big(\frac{L_{g,2}C_{f_y}}{\mu}+L_{f,1}\Big)^2\bigg]\sum_{t=0}^{T-1}\frac{\big\|\nabla_y g(x_t,y_t)\big\|^2}{\varphi_{t+1}} \nonumber \\
    & + \frac{2\bar{L}^2}{\mu^2\alpha_0T}\sum_{t=0}^{T-1}\frac{\big\|\nabla_v R(x_t,y_t,v_t)\|^2}{\varphi_{t+1}} \nonumber \\
    \overset{(a)}{\leq} & \frac{2}{T}\big(\Phi(x_0) - \inf_x \Phi(x)\big) + \frac{L_{\Phi}C_\alpha^2}{T\alpha_0^2\varphi_0^2} \nonumber \\
    & + \frac{\bar{L}^2}{\mu^2\alpha_0T}\bigg[1 + \frac{2}{\mu^2}\Big(\frac{L_{g,2}C_{f_y}}{\mu}+L_{f,1}\Big)^2\bigg]\big(a_2\log(T) + b_2\big) \nonumber \\
    & + \frac{2\bar{L}^2}{\mu^2\alpha_0T}\big(a_3\log(T) + b_3\big) \nonumber \\
    =& \frac{1}{2T}\Big(a_4 \log(T) + b_4 + 4\big(\Phi(x_0) - \inf_x \Phi(x)\big)\Big), 
\end{align}
\endgroup
where (a) uses \Cref{lm:somebounds} by plugging in $k_0 = 0$.

\noindent
Note that {\bf Case 1} and {\bf Case 2} indicate the same result.
Thus, we have
\begingroup
\allowdisplaybreaks
\begin{align}
    \frac{1}{T}\sum_{t=0}^{T-1} \|\nabla \Phi(x_t)\|^2 
    \leq& \frac{1}{2T}\Big(a_4 \log(T) + b_4 + 4\big(\Phi(x_0) - \inf_x \Phi(x)\big)\Big)\alpha_T\varphi_T \nonumber \\
    \overset{(a)}{\leq}& \frac{1}{2T}\Big[\Big(a_4 \log(T) + b_4 + 4\big(\Phi(x_0) - \inf_x \Phi(x)\big)\Big)^2\varphi_T^2 \nonumber \\
    & \quad \quad + C_\alpha\Big(a_4 \log(T) + b_4 + 4\big(\Phi(x_0) - \inf_x \Phi(x)\big)\Big)\varphi_T\Big] \nonumber \\
    \overset{(b)}{\leq}& \frac{1}{2T}\Big[\Big(a_4 \log(T) + b_4 + 4\big(\Phi(x_0) - \inf_x \Phi(x)\big)\Big)^2\big(a_1\log(T)+b_1\big)^2 \nonumber \\
    & \quad \quad + C_\alpha\Big(a_4 \log(T) + b_4 + 4\big(\Phi(x_0) - \inf_x \Phi(x)\big)\Big)\big(a_1\log(T)+b_1\big)\Big] \nonumber \\
    = & \mathcal{O}\bigg(\frac{\log^4(T)}{T}\bigg). \nonumber
\end{align}
\endgroup
where (a) follows from \Cref{lm:alpha}; 
(b) results from \Cref{lm:varphi}. 
Thus, the proof is finished.

\subsection{Complexity Analysis of \Cref{alg:main} (Proof of \Cref{cor:main})}
Recall in \Cref{thm:main}, we know that there exist a constant $M$ such that
\begin{align}
    \frac{1}{T}\sum_{t=0}^{T-1}\|\nabla \Phi(x_t)\|^2 \leq \frac{M\log^4(T)}{T}. \nonumber
\end{align}
When we set the iteration number $T=\frac{MN}{\epsilon}\log^4 (\frac{M}{\epsilon})$ and assume the constant $N = 12^4$, we have 
\begingroup
\allowdisplaybreaks
{\small
\begin{align}
    \frac{M\log^4(T)}{T} =& \frac{M\log^4(\frac{MN}{\epsilon}\log^4 (\frac{M}{\epsilon}))}{\frac{MN}{\epsilon}\log^4 (\frac{M}{\epsilon})} \nonumber \\
    \leq & \frac{[\log(N) + \log(\frac{M}{\epsilon}) + 4\log(\log(\frac{M}{\epsilon}))]^4}{N\log^4(\frac{M}{\epsilon})}\cdot\epsilon \nonumber \\
    \overset{(a)}{\leq} & \bigg(\frac{\log(N) + 2\log(\frac{M}{\epsilon})}{N^{\frac{1}{4}}\log(\frac{M}{\epsilon})}\bigg)^4\cdot\epsilon \overset{(b)}{\leq} \epsilon, \nonumber
\end{align}}
\endgroup
where (a) follows from the inequality $\log(\log(\frac{M}{\epsilon})) \leq \frac{1}{4}\log(\frac{M}{\epsilon})$ for sufficiently small $\epsilon$; (b) holds because $\log(N) + 2\log(\frac{M}{\epsilon}) \leq N^{\frac{1}{4}}\log(\frac{M}{\epsilon})$ for $N = 12^4$ and $\epsilon$ is sufficiently small. 
Thus, to achieve $\epsilon$-accurate stationary point, we require $T =\frac{MN}{\epsilon}\log^4 (\frac{M}{\epsilon}) = \mathcal{O}\big(\frac{1}{\epsilon}\log^4 (\frac{1}{\epsilon})\big)$, and the gradient complexity is given by ${\rm Gc}(\epsilon) = \Omega(T) = \mathcal{O}\big(\frac{1}{\epsilon}\log^4 (\frac{1}{\epsilon})\big)$.

\end{document}